\documentclass[oneside,11pt]{article}

\usepackage{tikz}
\usetikzlibrary{3d,calc}
\usepackage{mathrsfs,placeins,soulutf8}
\usepackage[linesnumbered,ruled]{algorithm2e} 
\usepackage{tikz}
\usepackage{float}
\usepackage{subcaption} 

\usepackage[a4paper, total={6in, 8in}]{geometry}

\usepackage{mathtools}
\usetikzlibrary{shapes.geometric, arrows}
\usepackage{algorithmic}
\usepackage{bm}
\usepackage{booktabs}
\usepackage[english]{babel}
\usepackage[utf8]{inputenc}
\usepackage{multirow}
\usepackage{adjustbox}
\usepackage{subcaption,afterpage} 
\usepackage{mathtools}
\usepackage{amsfonts}
\usepackage[colorinlistoftodos]{todonotes}
\usepackage{enumerate}
\usepackage{hyperref}
\RequirePackage{epsfig}
\RequirePackage{amssymb}
\usepackage{amsthm}
\RequirePackage{natbib}
\RequirePackage{graphicx}

\if@usehyper
\RequirePackage[colorlinks=false,allbordercolors={1 1 1}]{hyperref}
\fi

\if@abbrvbib
\else
\fi

\bibpunct{(}{)}{;}{a}{,}{,}
 
\newcommand{\excise}[1]{}

\usepackage{amsmath}
\usepackage{cleveref}
\usepackage{dsfont}

\usepackage{setspace}

\newcommand\QQ{\mathbb{Q}}
\newcommand\RR{\mathbb{R}}
\newcommand\PP{\mathbb{P}}

\newcommand{\GG}{\mathrm{G}}
\DeclareMathOperator*{\argmin}{argmin}

\theoremstyle{plain}
\newtheorem{theorem}{Theorem}[section]

\newtheorem{corollary}[theorem]{Corollary}
\theoremstyle{definition}
\newtheorem{definition}[theorem]{Definition}

\theoremstyle{remark}

\newtheorem{assumption}{Assumption}

\makeatletter
\newtheorem{rep@theorem}{\rep@title}
\newcommand{\newreptheorem}[2]{%
\newenvironment{rep#1}[1]{%
 \def\rep@title{#2 \ref{##1}}%
 \begin{rep@theorem}}%
 {\end{rep@theorem}}}
\makeatother

\newreptheorem{theorem}{Theorem}
\newreptheorem{lemma}{Lemma}
\newreptheorem{definition}{Definition}
\crefname{definition}{definition}{definitions}
\Crefname{definition}{Definition}{Definitions}

\usepackage{changepage}
\newif\ifhighlightchanges

\usepackage{lastpage}

\begin{document}

\title{Can Generative Artificial Intelligence Survive Data Contamination? Theoretical Guarantees under Contaminated Recursive Training
}
\author{Kevin Wang, Hongqian Niu, Didong Li\thanks{didongli@unc.edu}\\
Department of Biostatistics, University of North Carolina at Chapel Hill}
\date{}

\maketitle

\begin{abstract}

    As artificial intelligence (AI)-generated content proliferates, models are increasingly trained on their own outputs, risking progressive degradation or collapse. In this article, we provide the first positive, rigorous theoretical results, to the best of our knowledge, showing that under model-agnostic mild conditions, the model converges to the true data-generating distribution. The convergence rate is the minimum of the model's intrinsic rate and the fraction of real data at each training iteration, revealing a phase transition between data-limited and model-limited regimes. We further show that, for biased real data, correcting the bias prevents the persistence and amplification of early bias over training iteration. Extensive experiments across simulations, real images and texts validate our theoretical framework, establishing quantitative conditions for long-term AI stability in contaminated environments.
    \end{abstract}
        

\section{Introduction}

    Recent analyses suggest that Artificial Intelligence (AI) generated text, images, and code constitute an increasingly large share of online content. Journalistic investigations have documented widespread use of AI-generated text across platforms such as Wikipedia and hobbyist communities \citep{Froio2025AIHouseplant}, as well as coordinated campaigns on social media \citep{BBC2025spammersAI}. In parallel, empirical research has documented the growing presence of AI-generated content in scientific and academic writing \citep{brooks2024rise, liang2024monitoring}.

    Consequently, modern generative models are increasingly trained on data streams that contain both human-generated and model-generated content. Such recursive training is already occurring, both unintentionally through web-scale data collection and intentionally in applications that incorporate synthetic data \citep{bowles2018gan,tobin2017domain,billot2023synthseg,xie2018differentially,papernot2016semi,jordon2018pate,lopes2017data,yin2020dreaming,lee2013pseudo,xie2020selfnoisystudent,arazo2020pseudo}. Because synthetic content is often difficult to reliably identify, simply filtering it from training data is unlikely to be practical \citep{nist2024reducing}. Synthetic content is therefore becoming an intrinsic part of the digital information ecosystem, influencing both the information encountered by users and the data pipelines that shape future models. The critical question then is: \textit{will recursive training on mixtures of real and synthetic data generated by the model's predecessors lead to progressive degradation, or can learning still improve over time?} This question is central to understanding the statistical behavior of learning algorithms trained on recursively contaminated data.

        \begin{figure}[ht!]
        \centering
        \includegraphics[width=0.75\linewidth]{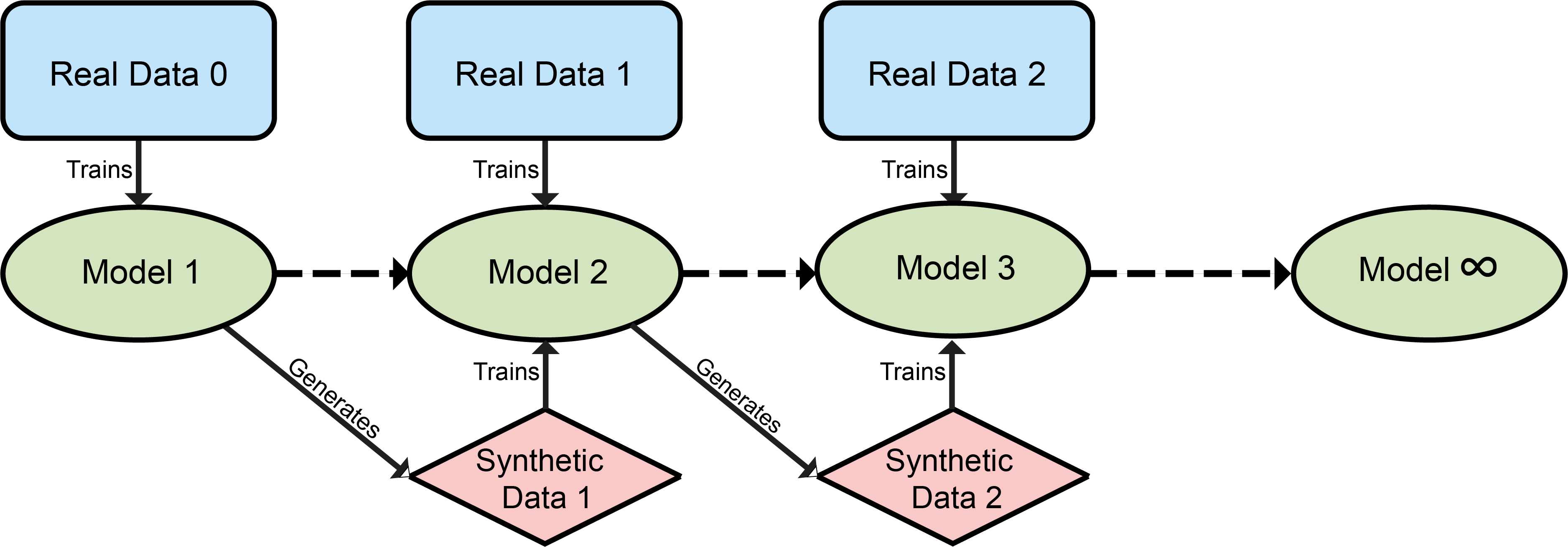}
            \caption{
            The Contaminated Recursive Training (CRT) framework is illustrated via a diagram. An initial dataset of real data (Real Data 0) trains the first model. After deployment, the model produces synthetic content (Synthetic Data 1) while new real data continue to appear online (Real Data 1). Subsequent models are trained on the accumulated mixture of real and synthetic data.}
        \label{fig:figure1}
    \end{figure}

    Recent work has begun to examine the behavior of models trained on synthetic data. Theoretical studies have characterized model collapse (where performance deteriorates and generated outputs become increasingly distorted or less diverse) when training relies entirely on synthetic data \citep{shumailov2023curse,shumailov2024ai,suresh2024rate}. Empirical evidence suggests that incorporating real data at each iteration can mitigate degradation \citep{zhang2026regurgitative}, and theoretical analyses support this stabilizing effect in parametric settings~\citep{gerstgrasser2024model, dohmatob2024model}. However, general theoretical conditions with realistic assumptions under which recursive training converges to the underlying data distribution remain unknown. Without such guarantees, the long-term behavior of models trained on mixtures of real and synthetic data remains fundamentally uncertain.

    We formalize this setting as contaminated recursive training (CRT). At each iteration, fresh samples from the target distribution and synthetic samples from the current generator are added to an accumulated training set. A learning procedure is then applied to the full accumulated dataset to produce the next generator. This formulation captures the case in which real data continue to arrive over time, but the training corpus is increasingly contaminated by synthetic data that may not be reliably separated. This framework is illustrated in \Cref{fig:figure1}.

    In addition to degradation caused by synthetic contamination, recursive training may also propagate systematic errors arising from bias in real data. Generative models trained on non-representative datasets are known to produce systematically distorted outputs \citep{mehrabi2021survey,zhou2024bias}, and such biases have been extensively documented in widely used training corpora \citep{lucy2021gender,sheng2019woman}. This raises the concern that when synthetic content is reused for training, biases introduced early in the process may persist or be amplified across successive generations, creating long-term feedback effects. Although numerous methods have been developed to mitigate sampling bias and distribution shift \citep{he2009learning,cortes2014domain,chen2023comprehensive}, the long-term behavior of bias in the context of recursive training has not been formally characterized. A central question is \textit{whether early distortions can be corrected over time, or whether they instead become permanently embedded in future models, gradually shaping what people see, read, and rely on online.}

     This paper addresses these questions by providing, to the best of our knowledge, the first  general theoretical analysis of recursive training that establishes convergence rates to the true data-generating distribution under data contamination. We show that generative models trained on mixtures of real and synthetic data converge to the target distribution without requiring parametric assumptions on the data-generating distribution, relying only on general convergence properties of the underlying learning procedure. Moreover, the convergence rate is determined by the minimum of the convergence rate of the base AI model and the fraction of real data available at each iteration, revealing a phase transition between data-limited and model-limited regimes. We further show that sufficiently accurate bias-correction prevents the persistence and amplification of initial sampling bias. We validate these theoretical findings through extensive experiments across simulations, real images and texts, using generative adversarial networks (GANs), diffusion models, and large language models (LLMs). These experiments confirm the convergence of the recursively trained models as well as the predicted phase transitions, demonstrating that the theoretical results hold in practical, complex settings. Section 2 reviews existing work; Section 3 and 4 contain main results for CRT and BCRT respectively; Section 5-7 presents experimental results. All proofs and additional experimental details are provided in the Appendix.
    

\section{Previous Work}

    Prior work on recursive training can be broadly divided into three settings: synthetic-only recursion, recursion with continued access to real data, and contaminated training with both real and synthetic data. Most theoretical and empirical work has focused on the synthetic-only setting, where various forms of model collapse have been established. Existing positive convergence results rely on growing the original real dataset and the parametric model setting. In contrast, the contaminated setting, in which new real and synthetic data are introduced simultaneously over time, has received comparatively little theoretical attention despite its relevance to modern generative AI systems.

    Theoretically in the synthetic-only setting, \cite{shumailov2024ai} analyze iterative training under both discrete and Gaussian models trained via maximum likelihood, proving that synthetic-only recursive training inevitably leads to model collapse. Their theoretical results are supported by empirical evidence demonstrating severe degradation. Building on this, \cite{suresh2024rate} provide explicit rates of collapse in similar settings, quantifying how quickly the learned distribution diverges from the target.
    
    Related work by \cite{shumailov2023curse} also considers synthetic-only training and establishes collapse in this regime. Notably, they introduce a “partial-refresh” setting in which a small fraction of the original real data is reintroduced at each iteration. While this slows degradation, it does not prevent eventual collapse. Furthermore, \cite{dohmatob2024model, dohmatob2025strong} derive explicit rates of performance degradation in recursive training loops that rely solely on synthetic data within parametric model classes. Complementing this line of work, \cite{dohmatob2024tale} characterize model collapse through changes in tail-diversity scaling laws in certain language models, while \cite{feng2025beyond} show that sufficiently effective verification of synthetic data can mitigate collapse. Related work by \cite{dey2024universality} studies an asymptotic universality phenomenon in recursive training under accumulated-data settings.
    
    Beyond synthetic-only regimes, several works study recursive training as a form of data contamination, where real data is progressively mixed with model-generated samples. In a two-generation setting, \cite{hataya2023will} show empirically that training a second-generation model on a mixture of real and synthetic data leads to degradation relative to the first-generation model. Extending this to multiple iterations, \cite{martinez2023combining} demonstrate that recursively adding synthetic data while keeping the original dataset fixed results in cumulative performance decline. \cite{bertrand2023stability} show that when a sufficiently large proportion of the original data is retained, collapse can be avoided entirely, though these analyses assume that no new data is introduced over time. Moreover, \cite{briesch2023large} study a training loop for LLMs where generated synthetic data as well as new real data are repeatedly incorporated into subsequent training sets. Their experiments suggest that semantic coherence is largely preserved while output diversity declines over iterations, though they do not theoretically establish convergence to the true data distribution.
    
    Additional theoretical support for these findings is provided by \cite{gerstgrasser2024model}, which analyzes linear models trained via least squares and shows that when maintaining access to the full original dataset, collapse is not inevitable. \cite{barzilai2026models} analyze iterative maximum-likelihood estimation for parametric models in an accumulating-data setting and establish conditions under which consistency is preserved, with stability depending on continued growth of the initial real-data sample.

    Taken together, existing theoretical analyses either focus on synthetic-only recursion or establish stability in restricted accumulating-data settings, typically under parametric assumptions and continued access to a growing corpus of real data. By contrast, our setting allows fresh real and synthetic samples to arrive simultaneously, permits accumulation of all historical data, and accommodates general target distributions and generators. The next section introduces this framework formally and develops corresponding convergence guarantees.


\section{Recursive Training under Data Contamination} \label{sec:CRT}
    
    In this section we begin by rigorously defining the generative models and the 
    recursive training under data contamination scheme, and present theorems for the theoretical convergence under such frameworks.  
    
    \begin{definition}[Generative model]\label{def:generator}
    Let $\PP_0$ be the target distribution. Given a model class $\mathcal{G}$ and a distance metric $d$ between distributions, the goal is to learn a generative model $\GG\in\mathcal{G}$ such that
    \[
        \argmin_{\GG\in\mathcal{G}} d(\GG,\PP_0).
    \]
    \end{definition}
    
    Typically, one will not have access to $\PP_0$ directly, but rather through $n$ i.i.d. samples from $\PP_0.$ In this case, while the goal of minimizing the distance remains the same, the algorithm for learning the generator will often not minimize this quantity directly as its objective but some empirical counterpart.


    \begin{definition}[Convergence rate]\label{def:conv_rate}
        Let $\widehat{\PP}_n$ be an estimated distribution from a generative model $\mathcal{G}$ trained from $n$ i.i.d. samples drawn from $\PP_0.$ We say that $\widehat{\PP}_n$ converges at a (polynomial) rate $p>0$  if $d(\widehat{\PP}_n,\PP_0) \lesssim n^{-p}$, i.e., there exists a constant $C>0$ such that $d(\widehat{\PP}_n,\PP_0) \leq Cn^{-p}$. 
    \end{definition}

    In this paper, when referring to setting where the model is trained with no contamination, we call the convergence rate the baseline rate. For generality, we leave the method of finding the $\widehat{\PP}_n$ unspecified, as our theory does not depend on the specific model but only on its convergence rate $p$. Moreover, if we change the notion to convergence in probability, all theorems hold in probability as well. Next, we introduce our setting of recursive training under data contamination.

    
    \begin{definition}[Contaminated Recursive Training (CRT)]\label{def:CRT}
    Let $X_0$ be an initial dataset with 
    $m_1$ i.i.d. samples drawn from $\PP_0$. Let the initial estimator $\widehat{\PP}_0$ be trained on $X_0$. For each step $t\geq 1$, perform the following operations:
    
    \begin{enumerate}
        \item \textbf{Recursive generation.}
        From the previous estimator $\widehat{\PP}_{t-1}$, draw an i.i.d.\ synthetic
        sample $Y_{t}$ of size $m_2$. Also draw a new i.i.d. real dataset $X_{t}$ of size $m_1$ from $\PP_0$. Let $\alpha\coloneqq\frac{m_1}{m_1+m_2}\in(0,1)$ be the real-data fraction.
    
        \item \textbf{Contaminated data accumulation.}
        Accumulate all real datasets $X_0,\ldots,X_{t}$ and all previously generated synthetic datasets $Y_1,\ldots,Y_{t}$ into one dataset. 
    
        \item \textbf{Recursive update.}
        Train the learner on the accumulated data, producing the new generator $\widehat{\PP}_t$.
    \end{enumerate}
    
    $\{\widehat{\PP}_t\}_{t\ge0}$ is called the
    \emph{contaminated recursively trained} (CRT) sequence of generators.
    \end{definition}
    
    \noindent\textbf{Comparison to prior recursive training settings.} Our setup differs from the synthetic-only recursive schemes studied in the model-collapse literature~\citep{shumailov2024ai,shumailov2023curse,suresh2024rate, dohmatob2025strong}, where each new model is trained primarily on synthetic samples and real data from earlier rounds is discarded. Such synthetic-only recursion drives the training distribution progressively away from the target distribution and leads to collapse. Our setting also differs from replacement-style contamination models~\citep{hataya2023will}, where each iteration trains on a mixture of real and synthetic data but without accumulating all past real samples. In contrast, both real and synthetic data are continuously uploaded to the internet. Additionally, internet archives retain essentially all data that has been uploaded, both real and synthetic. Thus, it is imperative to study the setting where each iteration's generator is trained on a new sample of real data, a new sample of synthetic data from the most recent generator, and the running data from all previous iterations. 
    
    We then state the following two standard assumptions for our theoretical results.

     \begin{assumption}[Polynomial rate for baseline generative models]\label{ass:poly}
 Let  $\mathcal{Q}$ be a convex class of distributions. Then 
 there exists a constant $0<M<\infty$ such that for any $\PP_0 \in \mathcal{Q}$ the baseline generative model based on a sample of size $n$ satisfies $d(\widehat{\PP}_n, \PP_0) \leq M n^{-p}$. 

    \end{assumption}

    \begin{assumption}[Convex distance metric]\label{ass:convex}
    The distance metric $d(\cdot,\cdot)$ between distributions is convex, i.e., for any distributions $P, Q_1,Q_2$, and any scalars $\lambda_1,\lambda_2>0$ with $\lambda_1 + \lambda_2 = 1$,
         \[    d\!\left(P,\,\lambda_1 Q_1 +\lambda_2 Q_2\right)    \leq \lambda_1\, d(P, Q_1) +  \lambda_2 d(P,Q_2).    \]
    \end{assumption}


    We emphasize that both assumptions are relatively weak in the literature and are satisfied under standard conditions commonly used in the theory of generative modeling. For example, \Cref{ass:poly} requires a uniform polynomial convergence rate over a distribution class $\mathcal{Q}$ that is closed under convex combinations. Such uniform guarantees are classical for estimators such as kernel density estimators (KDE,~\citealp{devroye1985nonparametric}) and Dirichlet process mixture models (DPMM,~\citealp{ghosal2017fundamentals}). Moreover, modern deep generative models, including Variational AutoEncoders~(VAE,~\citealp{cherief2022pac,mbacke2023statistical}), Generative Adversarial Networks~(GAN,~\citealp{uppal2019nonparametric,puchkin2024rates}), Diffusion Models~\citep{de2022convergence,oko2023diffusion,chen2023score,tang2024conditional}, and certain classes of Large Language Models~\citep{lotfi2024unlocking,rawat2024little}, are also known to satisfy such uniform convergence guarantees when the generator architecture is subject to standard regularity constraints. Finally, Assumption~\ref{ass:convex} is satisfied by most widely used statistical distances, including total variation, Kolmogorov--Smirnov, Wasserstein-$p$, energy distance, maximum mean discrepancy (MMD), and other integral probability metrics.


        
    
    
    With these two assumptions, we present our first main theorem about convergence rate of CRT.
    
    \begin{theorem}[Convergence rate under CRT]
    \label{thm:recursive_convergence}
    Suppose \Cref{ass:poly} and \ref{ass:convex} hold. Let $\{\widehat{\PP}_t\}_{t\ge0}$ be the sequence of CRT trained generators, then
    \[
    d(\widehat{\PP}_t, \PP_0)
    \lesssim
    \begin{cases}
    t^{-\alpha}, & p>\alpha,\\
    t^{-\alpha}\log t, & p=\alpha,\\
    t^{-p}, & p<\alpha,
    \end{cases}
    \]
    Equivalently, up to logarithms,
    \[
    d(\widehat \PP_t, \PP_0)
    \lesssim
    t^{-\min\{p,\,\alpha\}}.
    \]
    Additionally, if \Cref{ass:poly} is weakened to convergence in probability, the same rates hold in probability. 
    \end{theorem}

    This theorem shows that should the real-data fraction $\alpha$ be greater than the baseline rate $p$, then we may achieve the same convergence rate as the baseline. However, if the real-data fraction $\alpha$ is smaller than the baseline rate $p$, the convergence rate is slowed down. In the case where they are equal ($p=\alpha$), we see a phase transition.

\section{Biased Recursive Training under Data Contamination}\label{sec:BCRT}

    In a closely related problem, we consider the case where the sampling distribution of the real data may be biased to begin with. Because generative models trained on biased datasets exhibit biased generation, this poses a significant risk in real-world settings. Furthermore, recursively contaminated training processes may further propagate errors present in the initial model stemming from this biased sampling for the training data, which subsequently poses two important questions: what happens when early model iterations are trained from biased data, and can such initially biased models be corrected in future rounds of training? We find that such questions can be studied by naturally extending our framework to the case we now refer to as biased contaminated recursive training (BCRT). In essence, one might imagine that while the original data may come from a biased source, over time these biases may be recognized and slowly corrected in the subsequent real samples. Below we study sufficient conditions on the sampling distribution in the recursive training paradigm to yield convergence as in the previous section under these circumstances.
    
    \begin{definition}[Biased contaminated recursive training (BCRT)]\label{def:CRTbias}
    Let $(\PP_t^{\textnormal{bias}})_{t\ge0}$ be a sequence of (possibly biased) distributions, and $\PP_0$ be our target distribution. Let $X_0$ be an initial sample drawn from $\PP^\textnormal{bias}_0$ and the initial estimator be $\widehat{\PP}_0$. For each step $t\geq1$, perform the following operations:
    
    \begin{enumerate}
        \item \textbf{Recursive generation.}
        From the previous estimator $\widehat{\PP}_{t-1}$, draw an i.i.d.\ synthetic
        sample $Y_{t}$ of size $m_2$. Also draw an i.i.d. real dataset $X_{t}$ of size $m_1$ from $\PP^\textnormal{bias}_{t}.$ Let $\alpha\coloneqq\frac{m_1}{m_1+m_2}\in(0,1)$ be the real-data fraction.
    
        \item \textbf{Biased contaminated data accumulation.}
        Accumulate all real datasets $X_0,\ldots,X_{t}$ and all previously
        generated synthetic datasets $Y_1,\ldots,Y_{t}$ into one dataset.
    
        \item \textbf{Recursive update.} Train the learner on the accumulated hybrid data, producing the new generator $\widehat{\PP}_t.$
    \end{enumerate}
    
    The sequence $\{\widehat{\PP}_t\}_{t\ge0}$ is called the
    \emph{biased contaminated recursively trained} (BCRT) sequence of generators.
    \end{definition}

    As a direct consequence of \Cref{thm:recursive_convergence}, we can characterize what happens when the real samples are drawn from a fixed biased distribution and no effort is made to improve the sampling procedure or correct the bias. In this case, the recursive contamination process simply converges to the biased distribution itself. That is, the procedure will indeed converge, but to the biased distribution rather than the true one.

    \begin{corollary}[Convergence to a biased distribution]\label{cor:bias}
    If the real datasets $X_t$ are drawn i.i.d.\ from a fixed biased distribution $\PP^{\textnormal{bias}}_t=\PP_0^\textnormal{bias}\neq \PP_0$ and \Cref{ass:poly} and \ref{ass:convex} are satisfied, then the recursive procedure converges to $\PP_0^\textnormal{bias}$ rather than $\PP_0$, with the same rates as in Theorem~\ref{thm:recursive_convergence}.
    \end{corollary}
    

    
    In many realistic settings, practitioners apply bias-correction techniques or improved sampling strategies (see, e.g., \citep{he2009learning, cortes2014domain, chen2023comprehensive} so that the real-data distribution evolves over time and gradually approaches the desired target distribution. When such corrections are applied across iterations, it is not obvious whether the recursive contamination process still converges, and if so, at what rate and to which limit. The following result addresses this setting by analyzing a biased recursive contamination process in which the real-data distributions move toward the true distribution over time. To account for this, we add an assumption on the sequence of biased distributions.

    \begin{assumption}[Bias decay]\label{ass:biasdecay}
    The biased distributions $\PP^{\textnormal{bias}}_t \in \mathcal{Q}$ and the bias decays at a polynomial rate $q$, i.e.,
         \[    d\!\left(\PP_t^{\textnormal{bias}} , \PP_0\right)    \lesssim t^{-q} .    \]
    \end{assumption}


    This assumption captures the case where the practitioner performs any of, or any combination of: (1) improving their sampling procedures, (2) implementing bias-correction methods, or (3) working with subsequent distributions that are naturally unbiased relative to the target distribution. We note that the plausibility of this assumption in the first two cases is supported by the biased sampling and covariate shift literature, such as \citep{he2009learning, cortes2014domain, chen2023comprehensive}. We also emphasize that, in these, the rate is not necessarily determined by the model itself, but rather by the quality of the external adjustments made by the practitioner. The third case covers settings where the assumption is naturally satisfied, such as fine-tuning on a subpopulation of corpora. Additionally, the assumption that $\PP^{\textnormal{bias}}_t \in \mathcal{Q}$ simply ensures that \Cref{ass:poly} applies to each of the new distributions under study. Our next result shows that \Cref{thm:recursive_convergence} can be extended to the BCRT paradigm.

    \begin{theorem}[Convergence rate under BCRT]
    \label{thm:recursive_convergence_bias}
    Assume that Assumptions \ref{ass:poly}, \ref{ass:convex}, and \ref{ass:biasdecay} hold. Let $\{\widehat{\PP}_t\}_{t\geq 0}$ be the sequence of BCRT trained generators, then

    \[
    d(\widehat{\PP}_t,\PP_{0})
    \lesssim
    \begin{cases}
    t^{-\alpha},         & \min(p,q) > \alpha,\\
    t^{-\alpha}\log t,   & \min(p,q) = \alpha,\\
    t^{-\min(p,q)},       & \min(p,q) < \alpha.
    \end{cases}
    \]
    Equivalently, up to logarithms,
    \[
    d(\widehat{\PP}_t, \PP_{0})
    \lesssim
    t^{-\min\{p,\,q,\,\alpha\}}.
    \]
    Like before, if either \Cref{ass:poly} or \ref{ass:biasdecay} are weakened to convergence in probability, the above rates hold in probability.
    \end{theorem}

    This theorem shows that the biased setting mirrors the unbiased case, with an additional limitation on the rate determined by how quickly the bias in the real data distribution decays. 
    

\section{Simulations}\label{sec:sims}\label{sec:sim-CRT}

    We first conduct simulation studies to validate our theoretical findings. We examined both the CRT \Cref{sec:CRT} and BCRT \Cref{sec:BCRT} settings under controlled conditions in which the target distribution (\Cref{fig:figure2}) and thus the baseline rate are known. This design allowed a direct comparison between the convergence rates in \Cref{thm:recursive_convergence} and \Cref{thm:recursive_convergence_bias}, and those observed empirically. We consider three generative models, Empirical Cumulative Distribution Function (ECDF), Kernel Density Estimator (KDE), and Wasserstein Generative Adversarial Network (WGAN). Convergence was evaluated using two discrepancy metrics: Wasserstein--1 ($W_1$) and maximum mean discrepancy (MMD). For BCRT, we introduce bias into the real-data distribution and allow the bias to decrease over iterations. This design enabled independent control of all three quantities determining the convergence rate: the baseline rate, the real-data fraction, and the bias-correction rate. For further details about the theoretical rates and the simulations, see Appendix \Cref{apdx:CRT} and \Cref{apdx:sims}.

    We first consider the CRT setting from \Cref{sec:CRT} using KDEs with varying bandwidths. The density used is a mixture of two univariate Gaussians with known parameters,
    \[\PP_0=w_1\mathcal{N}(\mu_1,\sigma_1^2)+(1-w_1)\mathcal{N}(\mu_2,\sigma_2^2),\]
    with a density function shown in \Cref{fig:sim1-true-density}.

        \begin{figure}[!h]
            \centering            \includegraphics[width=0.5\textwidth]{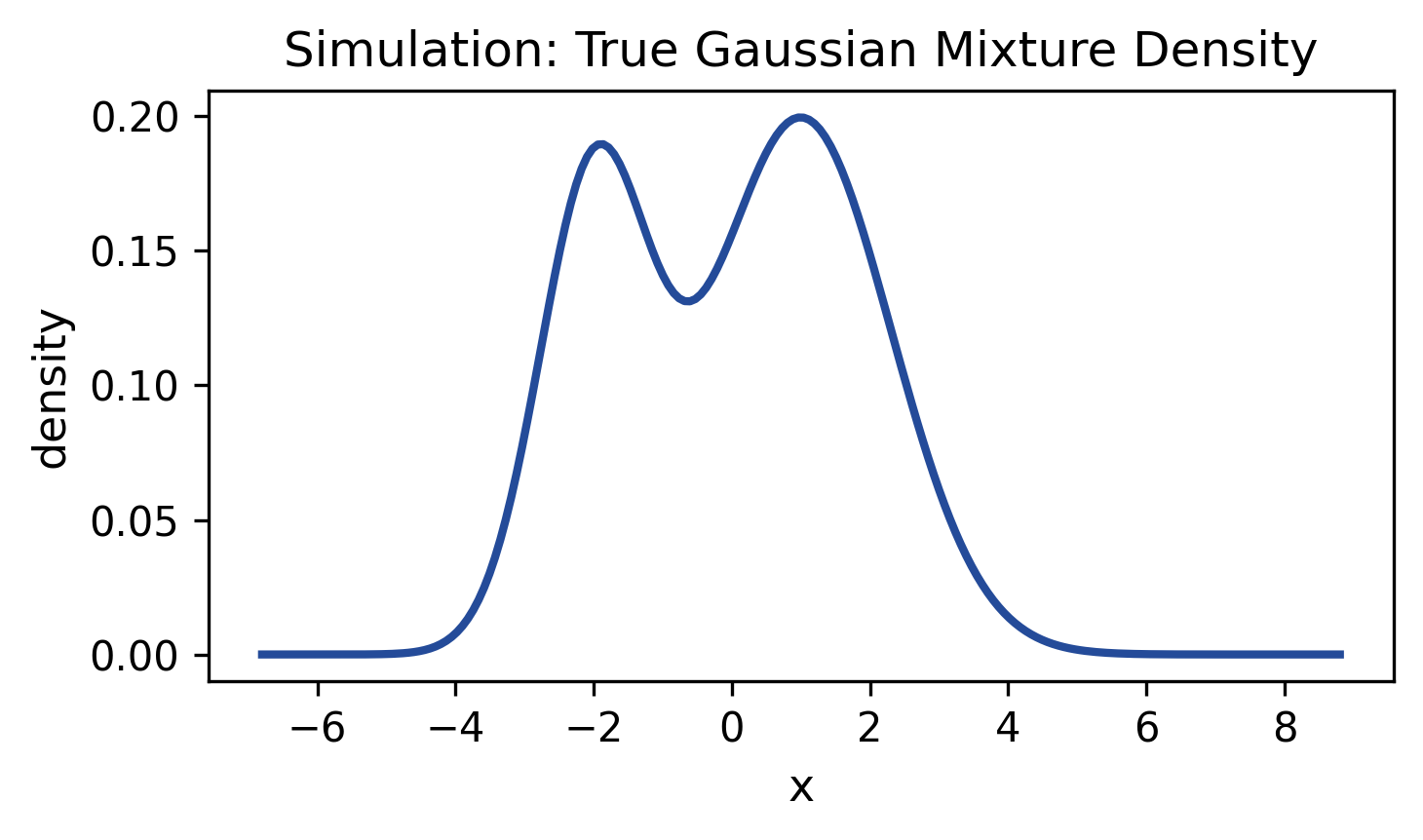}
            \caption{Density function of a two-Gaussian mixture  with parameters \(w_1=0.35\), \((\mu_1,\sigma_1)=(-2.0,0.8)\), and \((\mu_2,\sigma_2)=(1.0,1.3)\) used in Simulation of Section 5.1.}
            \label{fig:sim1-true-density}
        \end{figure}

    The density for the BCRT setting from $\Cref{sec:BCRT}$ is chosen to be the same but with a small biasing perturbation, whose rate of decay we may control:

    \[\PP_t^{\text{bias}}=(1-\text{bias}_t)\PP_0 + \text{bias}_t \PP_\text{bias}\mathcal{N}(\mu_3,\sigma_3^2),\]

    where $\PP_0$ is defined above and $\text{bias}_t = 0.2 (t+5)^{-q}$.
    
    For both the CRT and BCRT settings, we present three generators: the empirical cumulative distribution function~(ECDF), the standard Gaussian KDE with a bandwidth decaying on a deterministic schedule, and the WGAN. At each iteration $t$, we draw a new true sample of 50 real data from $\PP_0$ and a synthetic sample of $\frac{1-\alpha}{\alpha}50$ from the previous model, and apply the CRT setting as in Definition \ref{def:CRT}. Then, either an ECDF, KDE, or WGAN is used to form the next model in the sequence and the $W_1$ and $MMD$ losses from the ground truth are estimated by generating a sample and comparing the quantiles of the generated data to the quantiles of the true distribution on a uniform grid of $200$ points. A simple linear fit to the log-log transformed data is then used to estimate the observed order of convergence. Each experiment is performed for 2000 iterations. Each experiment is repeated for 50 replicates and the mean is reported for each distance metric. The results for the CRT setting are summarized in \Cref{fig:figure2}A and the results for the BCRT setting are summarized in \Cref{fig:figure2}B, where for each $\alpha=\{0.1,0.2,\dots ,1\}$ and metric, the rate predicted by our theory and the rate observed in the simulation are compared.


    Despite practical factors such as optimization dynamics and finite-sample effects, the empirical results are broadly consistent with the theoretical results. In particular, the observed convergence rates were limited by $\alpha$ in data-sparse regimes and by $q$ when bias improved slowly, reproducing the predicted phase-transition behavior between model-limited, data-limited, and bias-limited regimes. Additional details may be found in Appendix \Cref{apdx:CRT} and \ref{apdx:BCRT}.

        \begin{figure}[ht!]
        \centering
        \includegraphics[width=0.95\linewidth]{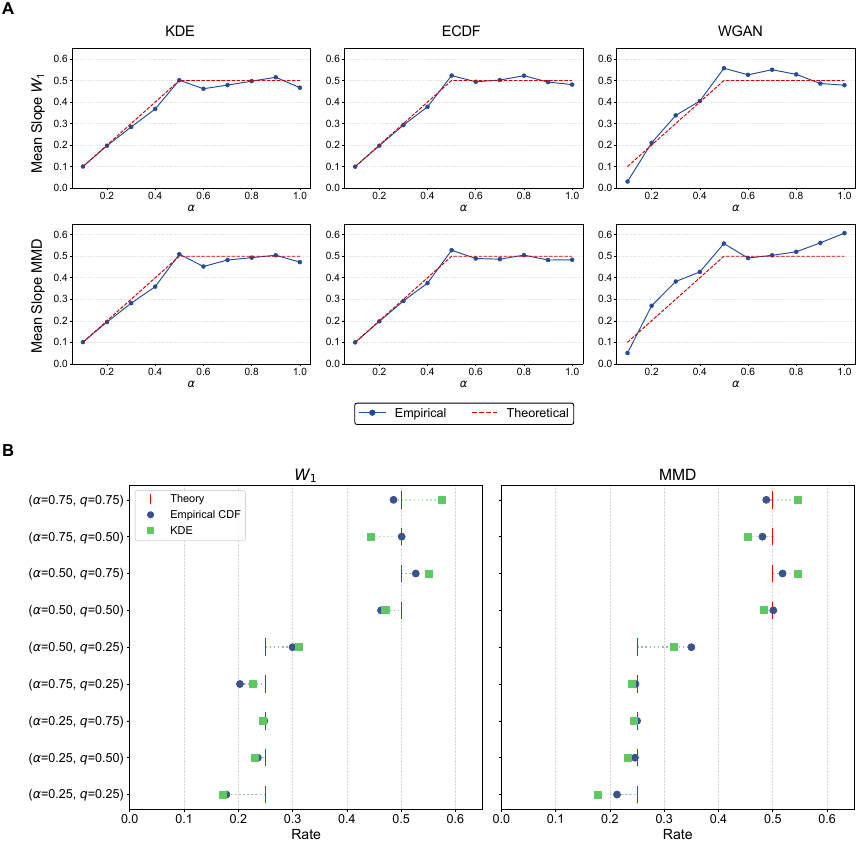}
        \caption{Simulation results.
            (A) Rates of KDE (left), ECDF (middle), and WGAN (right) measured using $W_1$ (top) and MMD (bottom) under CRT with varying $\alpha$, with empirical rates in blue and theoretical rates in red, where the target density is a 1-dimensional mixture of two Gaussians.  
            (B) Rates for ECDF and KDE under BCRT with varying $\alpha$ and $q$ measured by $W_1$ and MMD. For both CRT and BCRT, the empirical rates match the theoretical rates.}
        \label{fig:figure2}
    \end{figure}

\section{Image Experiments}\label{sec:mnist}

    Thus far, we have demonstrated the empirical rates of convergence under CRT and BCRT in simulated settings where the true distribution $\PP_0$ is known and where exact empirical formulations of $W_1$ and MMD are available in one dimension.

    We now consider a more realistic setting by training a state-of-the-art diffusion model on MNIST images. Although Diffusion Models are known to converge under $W_1$ and MMD loss and therefore fit within the CRT framework of \Cref{sec:CRT}, observing the precise theoretical rates in practice is challenging due to complex neural network optimization dynamics, numerous architectural and training hyperparameters, and the lack of a computationally tractable exact $W_1$ metric in high-dimensional spaces. Moreover, unlike the synthetic experiments, the true data distribution $\PP_0$ is not explicitly known.

     
     Consequently, we evaluate model quality using the Fr\'echet Inception Distance (FID), a standard metric for assessing the similarity between generated and real images. While FID is not itself a convergence metric, sustained improvements in FID indicate that generated samples are becoming progressively closer to the real-data distribution. The resulting FID trajectories are shown in \Cref{fig:figure3}B, while representative generated samples are displayed in \Cref{fig:figure3}C.
     

     Hence, our objective in this experiment is not to verify exact convergence rates, but rather to investigate whether the qualitative convergence behavior predicted by \Cref{thm:recursive_convergence} persists in a realistic high-dimensional setting. In particular, we examine whether Diffusion Models trained under CRT continue to improve even when the real-data fraction at each iteration is substantially less than $0.5$, corresponding to a setting with significant recursive contamination.
     

    At each iteration, a training set of 300 images is constructed by sampling a fraction $\alpha$ of images from the MNIST dataset without replacement and a fraction $1-\alpha$ of synthetic images generated by the model trained in the previous iteration. The diffusion model is then trained for 150 epochs on the resulting dataset. This procedure is repeated for 300 iterations, with the model parameters carried forward between iterations rather than being reinitialized. Consequently, data accumulates through the optimizer, and each real training sample is used an equal number of times throughout the recursive training process, consistent with the assumptions of \Cref{thm:recursive_convergence}. Due to the high dimensionality of the image space, $W_1$ distances are not evaluated. Additional experimental details may be found in \Cref{apdx:mnist}.
    


    \begin{figure}[htbp!]
        \centering
        \includegraphics[width=0.80\linewidth]{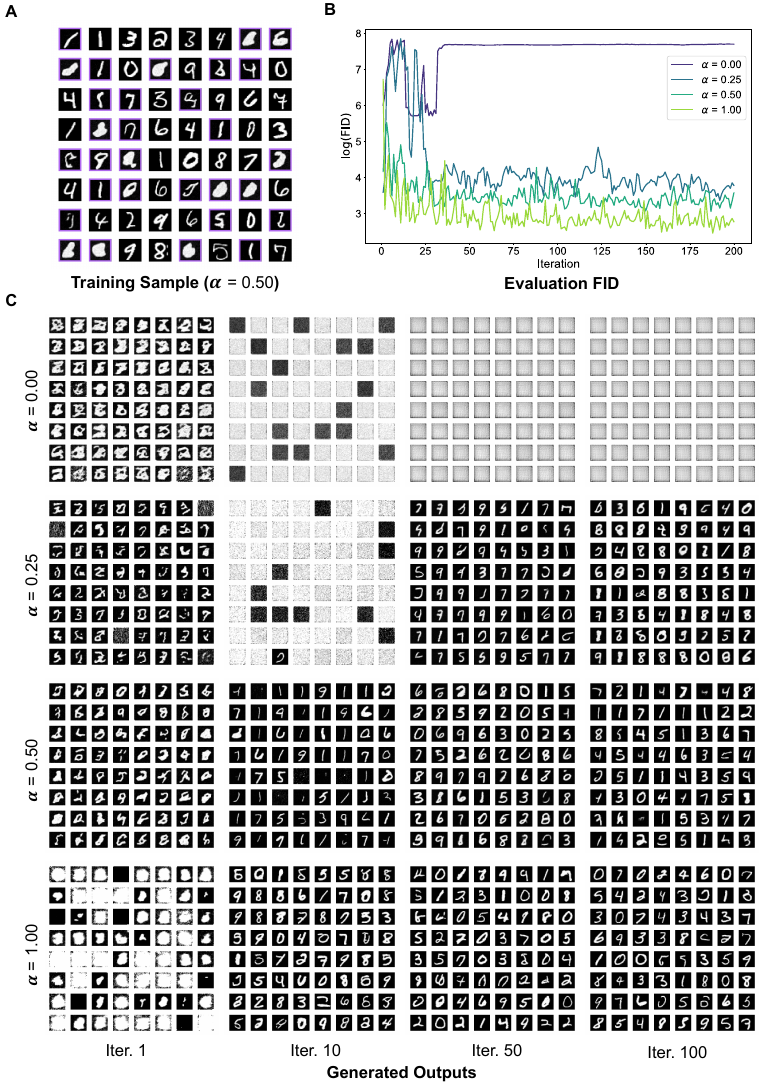}
            \caption{
            Image experiments under CRT using Diffusion Models. 
            (A) Sample input data with $\alpha=0.5$, where synthetically generated samples are highlighted.
            (B) FID at each training iteration for different values of $\alpha$. FID decreases steadily across iterations only when $\alpha>0$.
            (C) Samples generated by successive models in the CRT sequence for different values of $\alpha$. When $\alpha>0$, image quality improves across iterations, whereas training with $\alpha=0$ leads to collapse, consistent with theoretical results.}
        \label{fig:figure3}
        \end{figure}
    
    Overall, the results demonstrate that even when $\alpha < 1$, the diffusion model trained under CRT continues to improve according to FID and produces increasingly realistic samples. These findings provide empirical evidence that the convergence behavior predicted by \Cref{thm:recursive_convergence} extends beyond the idealized settings considered earlier and remains observable in practical image-generation tasks.

\section{Large Language Model Experiments}
\label{sec:llm}

We next investigate whether the qualitative behavior predicted by CRT persists in Large Language Models. Starting from a pretrained GPT-125M model, we perform recursive fine-tuning on WikiText-103 following the CRT procedure.


At iteration $t$, a total of $m_{\mathrm{total}}=1024$ token blocks of length 1024 are selected. A fraction $\alpha \in \{0,0.1,1\}$ is drawn from the real dataset without replacement, while the remaining fraction $1-\alpha$ consists of synthetic blocks generated by the model obtained from the previous iteration. The model is then fine-tuned for exactly one epoch on the resulting dataset. This procedure is repeated for $T=50$ iterations. As in the image experiments of \Cref{sec:mnist}, the model parameters are carried forward between iterations rather than reinitialized, so data accumulates through the optimizer. Consequently, each real data block is utilized an equal number of times throughout training, consistent with the assumptions of \Cref{thm:recursive_convergence}. The results are summarized in Figure~\ref{fig:figure4}.

\begin{figure}[H]
        \centering
        \includegraphics[width=1.0\linewidth]{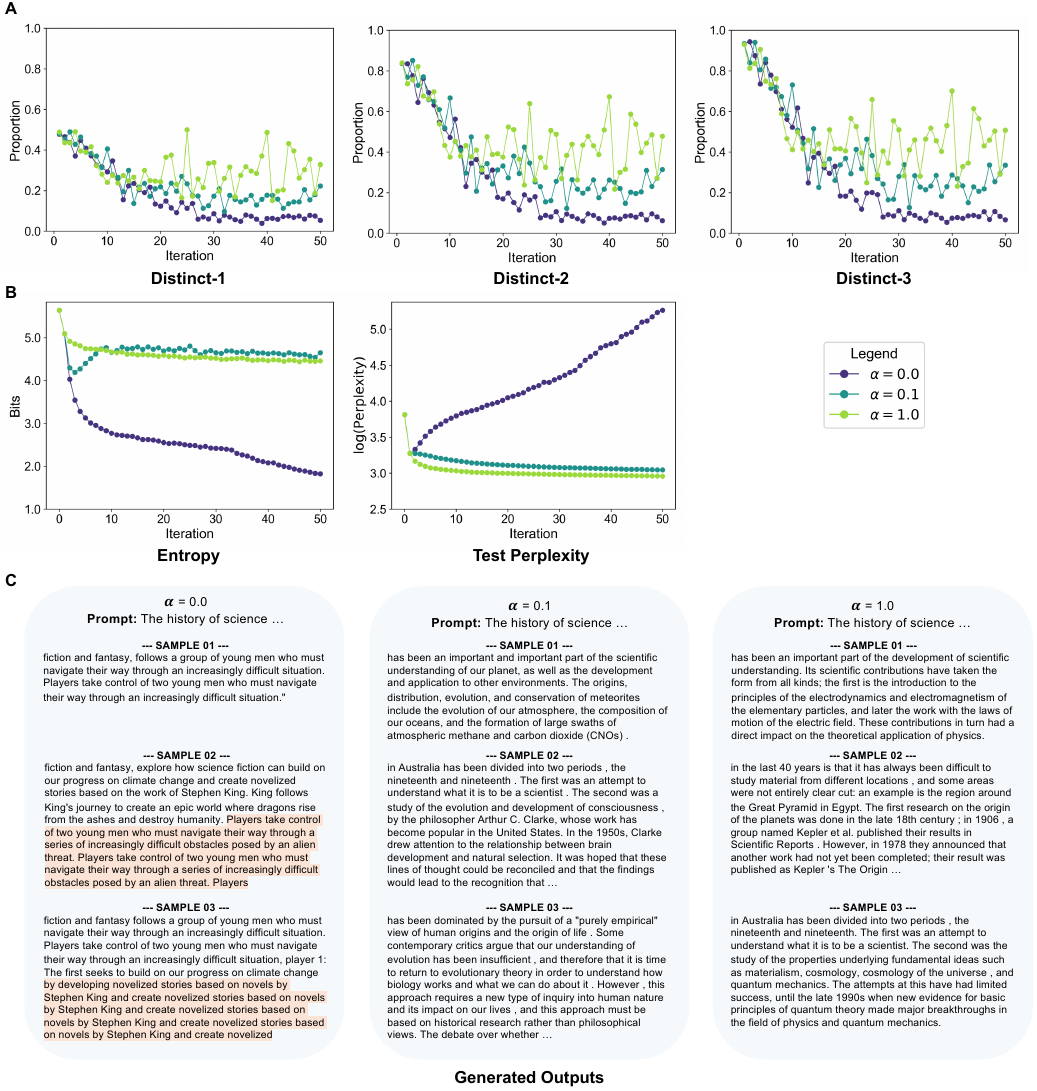}
        
        \caption{Text experiments under CRT using LLMs. (A)  Output diversity of the fine-tuned LLMs over iterations with varying $\alpha$'s, measured using number of distinct $k$-grams to assess potential model collapse for $k=1,2,3$. Diversity collapses rapidly when $\alpha = 0$, but is maintained when $\alpha = 0.1,1$ cases. 
                (B) Output entropy (left) and log prediction perplexity (middle) on a held-out WikiText test set measured at each iteration, with varying $\alpha$'s. Even with $\alpha=0.1,$ the test perplexity continues to decrease.
                (C) Variety of generated outputs of different generations of the model with different $\alpha,$ displaying model collapsed outputs for $\alpha=0$ (highlighted in orange) and stabilized outputs for $\alpha>0.$ 
                }
        \label{fig:figure4}
    \end{figure}

To assess potential model collapse more directly, Figure~\ref{fig:figure4} reports several complementary metrics. Figure~\ref{fig:figure4}A shows output diversity measured by the proportion of distinct $k$-grams. Figure~\ref{fig:figure4}B reports output entropy and test perplexity on a held-out WikiText test set. Because the underlying data distribution is unavailable, direct distributional distances cannot be computed. We therefore evaluate model quality using perplexity on a held-out WikiText test set. Although perplexity is only an indirect measure of distributional alignment, decreasing test perplexity indicates improved predictive performance and closer agreement with the test distribution. When $\alpha=0$, diversity collapses nearly to zero, output entropy decreases substantially, and test perplexity steadily worsens, consistent with model collapse under purely synthetic recursion. In contrast, for both $\alpha=0.1$ and $\alpha=1$, diversity remains substantially preserved, entropy remains stable at levels comparable to the fully real-data setting, and test perplexity continues to improve across iterations. Together, these results suggest that even a modest influx of real data is sufficient to prevent the degenerative behavior observed in the synthetic-only setting and maintain stable recursive training.


We also present the generated outputs in \Cref{fig:figure4}C after 50 iterations of CRT. When $\alpha = 0$, the outputs exhibit severe degeneration, both in terms of textual quality and relevance to the prompt. In contrast, when $\alpha = 1$, the generated outputs remain coherent and closely aligned with the prompt. Notably, even when $\alpha = 0.1$, incorporating only $10\%$ real data at each iteration is sufficient to prevent collapse and maintain reasonable output quality.

These results indicate that recursive training can remain stable and continue improving at scale. They suggest that the dynamics predicted by the CRT framework may extend to large-scale language models operating in increasingly synthetic data environments.

\section{Discussion and Future Work}
    In this article, we studied the long-term behavior of generative AI recursively trained on mixtures of real and synthetic data. Both our theoretical and empirical analysis shows that recursive training can survive data contamination as long as real data is included in each training iteration. Moreover, the asymptotic behavior is governed by three factors: the intrinsic rate of the AI model, real-data fraction in each iteration, and the rate at which bias in the data is corrected, with phase transitions, leading to model-limited, data-limited, and bias-limited regimes. 

    
    Our results suggest that the spread of AI-generated content across the internet does not necessarily cause AI models to deteriorate over time. What matters most is whether models continue to learn from new information about the real-world. In practice, this means that human-created material, such as journalism, photography, scientific articles, blog posts, and everyday online discussion, are central for shaping future AI models. As long as this stream of original content remains sufficient in both quality and quantity, AI models can continue to improve even in an environment where synthetic content is common.


    The results also provide perspective on concerns about bias in AI models. When real datasets contain systematic distortions, for example, when certain groups, perspectives, or experiences are under or over represented, models trained on those data can reproduce those patterns in their outputs. Under recursive training, it is feared that such distortions will reinforce themselves across successive generations of models, both perpetuating those biases onto internet users as well as solidifying their own presence in the next generation of AI. However, our analysis also shows that this process is not irreversible. Improvements in data collection and bias-correction can reduce these distortions over time, provided that the quality of incoming data improves sufficiently quickly. In practical terms, this suggests that efforts to broaden dataset representation, such as including more diverse sources, viewpoints, and communities, can still meaningfully improve future AI models even if earlier models were trained on imperfect data. 
    
    While our work establishes the foundations for recursively trained generative AI, several emerging aspects warrant future investigation. First, many important objectives used in modern generative modeling, including likelihood, cross-entropy, and Kullback-Leibler based methods, do not directly satisfy the assumptions of our theory. Future work may proceed in two complementary directions: extending the CRT framework to accommodate broader classes of discrepancies, or establishing conditions under which training procedures based on these objectives yield convergence in a discrepancy compatible with the current theory. Whether such connections exist, and under what assumptions, remains an open question.
    

    Second, the real-world path from model generation to future training data is considerably more complex than the CRT framework. Model outputs are not transferred directly into the internet, but instead pass through multiple stages of selection, editing, publication, and amplification by human users and automated systems. Modern training pipelines further apply procedures such as dataset curation, annotation, supervised fine tuning, reinforcement learning from human feedback, and learned reward models. Consequently, future models are shaped not only by the content that enters the training corpus but also by the mechanisms used to select and learn from that content. Extending recursive-training theory to account for these processes remains an important direction for future work, as they may either amplify or mitigate the effects of recursive contamination.

    
    Third, the distinction between real and synthetic data is becoming increasingly blurred. Many forms of online content, including research articles, software, blog posts, and digital media, are now produced through iterative interactions between humans and AI systems. In practice, content is often drafted, edited, refined, or curated through multiple rounds of human-AI collaboration before it is ultimately published. Such data cannot be naturally categorized as either purely human-generated or purely synthetic. As these forms of hybrid intelligence become increasingly prevalent, it will be important to develop recursive-training frameworks that explicitly account for data generated through human-AI interaction rather than assuming a clean separation between real and synthetic sources.


    Finally, we note that a substantial literature leverages the intentional use of synthetic data for training. In data-scarce settings, training on synthetic data has enabled models to achieve strong performance on real data, both after subsequent retraining and, in some cases, even without further adaptation. In some settings, models are trained entirely on data generated from computer simulation and then deployed in the physical world with little or no additional real-world training, such as in reinforcement learning with low tolerance for negative outcomes, such as controlling vehicles \cite{sadeghi2016cad2rl}. An important direction for future work is to extend our analysis to settings where synthetic data is deliberately introduced due to limited real data availability, and where one can control where, when, and how much synthetic data is incorporated.

    

\bibliography{ref}

\appendix{}

\section*{{\LARGE \bf Appendix}}
\addcontentsline{toc}{section}{Appendix}

This appendix contains the supplementary material for the paper ``Can Generative Artificial Intelligence Survive Data Contamination? Theoretical Guarantees under Contaminated Recursive Training." In \Cref{apdx:code} we provide a link to all code necessary to produce the simulations and experiments of the paper.
In \Cref{apdx:CRT}, we provide a more detailed and technical discussion of the \Cref{sec:CRT}, including the assumption, theorem, and proof. In \Cref{apdx:BCRT}, we provide the corresponding technical details for \Cref{sec:BCRT}. In \Cref{apdx:sims}, we provide theoretical justification for the baseline rates assumed in our \Cref{sec:sims}, as well as additional details for the simulations. In \Cref{apdx:mnist}, we provide additional experimental details for the image experiments in \Cref{sec:mnist}. Finally in \Cref{apdx:llm}, we provide additional experimental details for the LLM experiments in \Cref{sec:llm}.

\section{Code Availability}\label{apdx:code}

    All code to run simulations, experiments, and to generate all figures are available at: \url{https://github.com/hong-niu/generative-recursive-training}.

\section{CRT: Theoretical Details and Proof}\label{apdx:CRT}

This section provides the formal proof of \textbf{\Cref{thm:recursive_convergence} (Convergence under CRT)}. We follow the notation introduced in the main text.

\begin{proof}
Define
\[
\QQ_t
= \frac{1}{M_t}
\left[t m_1 \PP_0 + \sum_{j=1}^{t-1} m_2 \widehat{\PP}_j\right],
\qquad
M_t = t m_1 + (t-1)m_2.
\]
Using Assumptions \ref{ass:poly} and \ref{ass:convex}, we first use the triangle inequality to write
\[
d(\widehat{\PP}_t,\PP_0)
\le d(\widehat{\PP}_t, \QQ_t) + d(\QQ_t,\PP_0).
\]

To treat $d(\widehat{\PP}_t, \QQ_t)$, note that $\widehat{\PP}_t$ is learned from the accumulated data whose distribution is $\QQ_t.$ Using Assumption \ref{ass:poly}, which states that the baseline generative model has a uniform polynomial rate $p,$  we treat $\QQ_t$ as the distribution to be learned by $\widehat{\PP}_t$ at iteration $t,$ yielding $d(\widehat{\PP}_t, \QQ_t) \lesssim M_t^{-p}.$ We write

\[
d(\widehat{\PP}_t,\PP_0)
\le d(\widehat{\PP}_t, \QQ_t) + d(\QQ_t,\PP_0)
\lesssim M_t^{-p} + d(\QQ_t,\PP_0).
\]
From Assumption \ref{ass:convex}, by convexity of $d$ in its second argument,
\begin{align*}
d(\QQ_t,\PP_0)
&= d\!\left(\PP_0, \frac{t m_1}{M_t}\PP_0
  + \frac{m_2}{M_t}\sum_{j=1}^{t-1}\widehat{\PP}_j\right) \\
&\le \frac{t m_1}{M_t} d(\PP_0,\PP_0)
   + \frac{m_2}{M_t}\sum_{j=1}^{t-1} d(\PP_0,\widehat{\PP}_j) \\
&= \frac{m_2}{M_t}\sum_{j=1}^{t-1} d(\PP_0,\widehat{\PP}_j).
\end{align*}
Thus
\begin{equation}
\label{eq:main-upper}
d(\widehat{\PP}_t,\PP_0)
\lesssim
M_t^{-p} + \frac{m_2}{M_t}
\sum_{j=1}^{t-1}d(\PP_0,\widehat{\PP}_j).
\end{equation}
Let $S_{t-1} := \sum_{j=1}^{t-1} d(\PP_0,\widehat{\PP}_j)$.
Applying \eqref{eq:main-upper} at step $t-1$ gives
\[
d(\widehat{\PP}_{t-1},\PP_0)
\lesssim
M_{t-1}^{-p}
+ \frac{m_2}{M_{t-1}} S_{t-2},
\]
so
\begin{align*}
S_{t-1}
&= d(\widehat{\PP}_{t-1},\PP_0) + S_{t-2} \\
&\lesssim M_{t-1}^{-p}
 + \left(1+\frac{m_2}{M_{t-1}}\right) S_{t-2}.
\end{align*}
Iterating this recursion yields
\[
S_{t-1}
\lesssim
M_{t-1}^{-p}
+ \sum_{j=1}^{t-2} M_j^{-p}
\prod_{k=j}^{t-2}\left(1+\frac{m_2}{M_k}\right).
\]
Substituting back into \eqref{eq:main-upper} gives
\begin{align*}
d(\widehat{\PP}_t,\PP_0)
&\lesssim
M_t^{-p}
+ \frac{m_2}{M_t}M_{t-1}^{-p}
+ \frac{m_2}{M_t}\sum_{j=1}^{t-2} M_j^{-p}
\prod_{k=j}^{t-2}\left(1+\frac{m_2}{M_k}\right).
\end{align*}
Since
\[
M_t = t m_1 + (t-1)m_2 = (m_1+m_2)t - m_2,
\]
we have $M_t \asymp (m_1+m_2)t$, hence there exist constants
$c_1,c_2>0$ such that
\[
c_1 (m_1+m_2)t \le M_t \le c_2 (m_1+m_2)t \quad\text{for all }t\geq 1.
\]
Similarly,
\[
M_t^{-p} \lesssim (m_1+m_2)^{-p} t^{-p},
\qquad
\frac{m_2}{M_t} \lesssim \frac{1-\alpha}{t},
\]
with $\alpha = m_1/(m_1+m_2)$, and similarly $M_j^{-p} \lesssim
(m_1+m_2)^{-p} j^{-p}$.
Thus
\begin{align}
d(\widehat{\PP}_t,\PP_0)
&\lesssim
(m_1+m_2)^{-p} t^{-p}
+ m_2 (m_1+m_2)^{-p-1}(t-1)^{-p-1} \nonumber\\
&\quad
+ (1-\alpha) t^{-1}(m_1+m_2)^{-p}\sum_{j=1}^{t-2} j^{-p}
\prod_{k=j}^{t-2}\left(1+\frac{1-\alpha}{k}\right)\\
&\lesssim t^{-p} + t^{-p-1} + t^{-1}\sum_{j=1}^{t-2} j^{-p}
\prod_{k=j}^{t-2}\left(1+\frac{1-\alpha}{k}\right).\label{eqn:big_ineqn}
\end{align}
Up to multiplicative constants, we analyze the summation in \Cref{eqn:big_ineqn}:

\begin{align*}
t^{-1}\sum_{j=1}^{t-2}  j^{-p}
\prod_{k=j}^{t-2}\left(1+\frac{1-\alpha}{k}\right) 
&=  \sum_{j=1}^{t-2}  j^{-p}t^{-1}
\left[\prod_{k=j}^{t-2}\left(\frac{k+1-\alpha}{k}\right)
\right] \\
&=  \sum_{j=1}^{t-2}  j^{-p}t^{-1}
\frac{t-1}{j}\left[\prod_{k=j}^{t-2}\left(\frac{k+1-\alpha}{k}\right)
\left(\frac{k}{k+1}\right)\right] \\
&= \sum_{j=1}^{t-2}  j^{-p}
\left(\frac{t-1}{t}\right)
j^{-1}\left[\prod_{k=j}^{t-2}\left(\frac{k+1-\alpha}{k+1}\right)\right],
\end{align*}
where we use the identity  $ \prod_{k=j}^{t-2}\frac{k}{k+1}=\frac{j}{t-1}$.
Using the Gamma function, we have the exact identity
\[
\prod_{k=j}^{t-2}\left(\frac{k+1-\alpha}{k+1}\right)
= \frac{\Gamma(t+1-\alpha-1)}{\Gamma(j+1-\alpha)}
  \frac{\Gamma(j+1)}{\Gamma(t)}=\frac{\Gamma(t-\alpha)}{\Gamma(j+1-\alpha)}
  \frac{\Gamma(j+1)}{\Gamma(t)}.
\]
We recall the following inequalities for the Gamma function.
For $x>0$ and $0<s<1$, Gautschi's inequality \citep{gautschi1959some} states that
\[
x^{1-s}
< \frac{\Gamma(x+1)}{\Gamma(x+s)}
< (x+1)^{1-s}.
\]
Combined with standard Stirling-type bounds for $\Gamma$ \citep{rudin1976principles}, this implies that for any
 $\alpha\in(0,1)$ there exist constants $c_1,c_2,c_3,c_4>0$ such that
for all integers $t,j\ge1$,
\begin{equation}
\label{eq:gamma-ratios}
c_1 t^{-\alpha} \le \frac{\Gamma(t-\alpha)}{\Gamma(t)} \le c_2 t^{-\alpha},
\qquad
c_3 j^{\alpha} \le \frac{\Gamma(j+1)}{\Gamma(j+1-\alpha)} \le c_4 j^{\alpha}.
\end{equation}
Using \eqref{eq:gamma-ratios}, the contribution of the last term (without constants) in \Cref{eqn:big_ineqn} is
bounded by

\[
\left(\frac{t-1}{t}\right)
\frac{\Gamma(t-\alpha)}{\Gamma(t)}
\sum_{j=1}^{t-2}
    j^{-p-1}
    \frac{\Gamma(j+1)}{\Gamma(j+1-\alpha)} \lesssim
t^{-\alpha}
\sum_{j=1}^{t-2} j^{-p-1+\alpha}.\]
Putting all terms together, we have shown that

\begin{equation}
\label{eq:pre-final}
d(\widehat{\PP}_t, \PP_0)
\lesssim
t^{-p}
+ (t-1)^{-p-1}
+ 
t^{-\alpha}
\sum_{j=1}^{t-2} j^{-p-1+\alpha}.
\end{equation}
Standard bounds for p-series yield

\[
\sum_{j=1}^{t-2} j^{-1-(p-\alpha)}
\lesssim
\begin{cases}
1, & p > \alpha, \\
\log t, & p = \alpha, \\
t^{\alpha-p}, & p<\alpha.
\end{cases}
\]
We treat the three regimes separately, absorbing the first two
terms of \eqref{eq:pre-final} into the dominant term by adjusting the
constant if needed.

\smallskip\noindent\textbf{Case 1: $p>\alpha$.}
Then $\sum_{j=1}^{t-2} j^{-p-(1-\alpha)} \lesssim 1$, so the third term in
\eqref{eq:pre-final} is

\[\lesssim t^{-\alpha}.\]
Moreover, since $p>\alpha$ we have $t^{-p} \le t^{-\alpha}$ for all $t\ge1$, and
\[
(t-1)^{-p-1} \lesssim t^{-p-1} \le t^{-p} \le t^{-\alpha}.
\]
Hence, all three terms in \eqref{eq:pre-final} are asymptotically bounded by
$t^{-\alpha}$, and therefore
\[
d(\widehat{\PP}_t,\PP_0)
\lesssim
t^{-\alpha}.
\]

\smallskip\noindent\textbf{Case 2: $p=\alpha$.}
Then $\sum_{j=1}^{t-2} j^{-p-(1-\alpha)} = \sum_{j=1}^{t-2} j^{-1} \lesssim \log t$,
and $t^{-p} = t^{-\alpha}$.  Thus the third term in \eqref{eq:pre-final} is

\[
\lesssim t^{-\alpha}\log t.
\]
The first term in \eqref{eq:pre-final} is $t^{-p}$, which is
dominated by $t^{-p}\log t$, and the second term
$(t-1)^{-p-1}$ is of strictly smaller order.  Therefore

\[
d(\widehat{\PP}_t,\PP_0)
\lesssim
t^{-\alpha}\log t.
\]

\smallskip\noindent\textbf{Case 3: $p<\alpha$.}
Then
\[
\sum_{j=1}^{t-2} j^{-1-(p-\alpha)}
\lesssim t^{1-1-(p-\alpha)} = t^{\alpha-p}.
\]
Hence the third term in \eqref{eq:pre-final} is

\[
t^{-\alpha} t^{\alpha-p}
= t^{-p}.
\]
The first term $t^{-p}$ and the second term
$(t-1)^{-p-1}$ are dominated by the third term, so all three terms are
$O\!\left(t^{-p}\right)$ up to multiplicative constants.
Therefore
          
\[
d(\widehat{\PP}_t,\PP_0)
\lesssim
t^{-p}.
\]
Combining the three cases, we have proved that

\[
d(\widehat{\PP}_t, \PP_0)
\lesssim
\begin{cases}
t^{-\alpha}, & p>\alpha,\\
t^{-\alpha}\log t, & p=\alpha,\\
t^{-p}, & p<\alpha.
\end{cases}
\]
If \Cref{ass:poly} holds in probability, then all asymptotic inequalities above also hold in probability, which completes the proof.

\end{proof}

\section{BCRT: Theoretical Details and Proof}\label{apdx:BCRT}
        
    This section provides a formal proof of \textbf{\Cref{thm:recursive_convergence_bias} (Convergence under BCRT)}. We again follow the notation and framework introduced in the main text. The following elementary lemma will be useful in the proof of \Cref{thm:recursive_convergence_bias}.

    \paragraph{Lemma (Ces\`aro rate for drifting distributions).}
    \label{lem:cesaro_drift}
    Let $d(\cdot,\cdot)$ be a metric on distributions that satisfies Assumption~\ref{ass:convex}. Let $\{\PP_t\}_{t\ge 1}$ be a sequence of distributions and
    let $\PP_{0}$ be a target distribution such that for some
    $q>0$,
    \[
    d(\PP_t,\PP_{0}) \lesssim t^{-q}.
    \]
    Define the Ces\`aro averages
    \[
    \overline{\PP}_t := \frac{1}{t}\sum_{j=1}^{t} \PP_j.
    \]
    Then
    \[
    d(\overline{\PP}_t,\PP_{0})
    \lesssim
    t^{-\min\{q,1\}},
    \]
    up to a $\log t$ factor in the boundary case $q=1$.

    \begin{proof}
    By convexity of $d$ in its second argument,
    \[
    d(\overline{\PP}_t,\PP_{0})
    = d\!\left(\PP_{0},\frac{1}{t}\sum_{j=1}^t \PP_j\right)
    \le \frac{1}{t}\sum_{j=1}^t d(\PP_{0},\PP_j)
    \lesssim
    \frac{1}{t}\sum_{j=1}^t j^{-q}.
    \]
    The standard $p$-series asymptotics yield
    \[
    \frac{1}{t}\sum_{j=1}^t j^{-q}
    \asymp
    \begin{cases}
    t^{-q}, & 0<q<1,\\
    t^{-1}\log t, & q=1,\\
    t^{-1}, & q>1,
    \end{cases}
    \]
    which is $t^{-\min\{q,1\}}$ up to a logarithmic factor for $q=1$.
    \end{proof}

    We now prove \Cref{thm:recursive_convergence_bias}.

    
    \begin{proof}[Proof of \Cref{thm:recursive_convergence_bias}]
    At step $t$, define the real data average
    \[
    \overline{\PP}^{\textnormal{bias}}_t := \frac{1}{t}\sum_{j=1}^t \PP^{\textnormal{bias}}_j.
    \]
    Then
    \[
    \QQ_t =
    \frac{t m_1}{M_t}\,\overline{\PP}^{\textnormal{bias}}_t
    + \frac{m_2}{M_t}\sum_{j=1}^{t-1}\widehat \PP_j.
    \]
    Using the triangle inequality,
    \[
    d(\widehat \PP_t, \PP_{0})
    \le d(\widehat \PP_t, \QQ_t) + d(\QQ_t, \PP_{0}).
    \]
    \noindent\textbf{Step 1: Learning error and bias term.}
    By \Cref{ass:poly},
    \[
    d(\widehat \PP_t, \QQ_t) \lesssim M_t^{-p}.
    \]
    For the bias term, convexity of $d$ in its second argument gives
    \begin{align*}
    d(\QQ_t, \PP_{0})
    &=
    d\!\left(
        \PP_{0},
        \frac{t m_1}{M_t}\,\overline{\PP}^{\textnormal{bias}}_t
      + \frac{m_2}{M_t}\sum_{j=1}^{t-1}\widehat \PP_j
    \right)
    \\
    &\le
    \frac{t m_1}{M_t}\, d(\PP_{0}, \overline{\PP}^{\textnormal{bias}}_t)
    +
    \frac{m_2}{M_t} \sum_{j=1}^{t-1} d(\PP_{0}, \widehat \PP_j).
    \end{align*}
    By the preceding lemma (\Cref{lem:cesaro_drift}),
    \[
    d(\overline{\PP}_t^{\textnormal{bias}}, \PP_{0}) \lesssim t^{-\min\{q,1\}}
    \quad
    (\text{up to }\log t \text{ if } q=1).
    \]
    Since $M_t \asymp t(m_1+m_2)$,
    \[
    \frac{t m_1}{M_t} d(\overline{\PP}_t^{\textnormal{bias}}, \PP_{0})
    \lesssim t^{-\min\{q,1\}}.
    \]
    Define
    \[
    d_t := d(\widehat \PP_t, \PP_{0}),
    \qquad
    S_{t-1} := \sum_{j=1}^{t-1} d_j.
    \]
    We obtain the basic recursion
    
    \begin{equation}
    \label{eq:basic_recursion_bias_new}
    d_t
    \lesssim
    M_t^{-p}
    + t^{-\min\{q,1\}}
    + \frac{m_2}{M_t}\, S_{t-1}.
    \end{equation}

    \noindent\textbf{Step 2: Recursion for the partial sums.}
    Repeating the argument at step $t-1$ gives
    \[
    d_{t-1}
    \lesssim
    M_{t-1}^{-p}
    + (t-1)^{-\min\{q,1\}}
    + \frac{m_2}{M_{t-1}} S_{t-2}.
    \]
    Hence
    \[
    S_{t-1}
    = d_{t-1} + S_{t-2}
    \lesssim
    M_{t-1}^{-p} + (t-1)^{-\min\{q,1\}}
    + \left(1 + \frac{m_2}{M_{t-1}}\right) S_{t-2}.
    \]
    Iterating yields
    \begin{align}
    S_{t-1}
    &\lesssim
    \sum_{j=1}^{t-1}
    \bigl( M_j^{-p} + j^{-\min\{q,1\}} \bigr)
    \left[
        \prod_{k=j}^{t-2}
            \left(1 + \frac{m_2}{M_k}\right)
    \right].
    \label{eq:Sn_bias_expansion_new}
    \end{align}
    
    \noindent\textbf{Step 3: Substituting back.}
    Substituting \eqref{eq:Sn_bias_expansion_new} into
    \eqref{eq:basic_recursion_bias_new} gives
    \begin{align*}
    d_t&\lesssim
    M_t^{-p}
    + t^{-\min\{q,1\}}
    + \frac{m_2}{M_t}
    \sum_{j=1}^{t-2}
    \bigl( M_j^{-p} + j^{-\min\{q,1\}} \bigr)
    \left[
        \prod_{k=j}^{t-2}
            \left(1 + \frac{m_2}{M_k}\right)
    \right] \\
    &\lesssim
    M_t^{-p}
    + t^{-\min\{q,1\}}
    + (1-\alpha)
    \sum_{j=1}^{t-2}
    \bigl( M_j^{-p} + j^{-\min\{q,1\}} \bigr)
    \left[
        \frac{1}{M_t}\prod_{k=j}^{t-2}
            \left(1 + \frac{m_2}{M_k}\right)
    \right] .
    \label{eq:dn_bias_preMn_new}
    \end{align*}
    Using $M_t \asymp t(m_1+m_2)$ and
    \[
    \alpha = \frac{m_1}{m_1+m_2},
    \qquad
    \frac{m_2}{M_t} \asymp \frac{1-\alpha}{t},
    \]
    the remainder of the proof follows the same argument as in the unbiased case:
    \[
    \frac{1}{t}
    \prod_{k=j}^{t-2}\!\left(1 + \frac{1-\alpha}{k}\right)
    \asymp
    t^{-\alpha} j^{-(1-\alpha)}.
    \]
    In this form, the three effects are additive:

    \begin{align*}
    d_t 
    &\;\lesssim\;
    M_t^{-p}
    + t^{-\min\{q,1\}}
    + (1-\alpha)\, t^{-\alpha}
    \sum_{j=1}^{t-2}
    \bigl(M_j^{-p} + j^{-\min\{q,1\}}\bigr)
    j^{-(1-\alpha)}
    \\
    &\;\lesssim\;
    t^{-p}
    + t^{-\min\{q,1\}}
    + t^{-\alpha}
    \sum_{j=1}^{t-2}
    \left(
    j^{-1-(p-\alpha)}
    +
    j^{-1-(\min\{q,1\}-\alpha)}
    \right)
    \\
    &\;\lesssim\;
    t^{-p}
    + t^{-\min\{q,1\}}
    + 
    \begin{cases}
    t^{-\alpha}, & p > \alpha, \\
    t^{-\alpha}\log t, & p = \alpha, \\
    t^{\alpha-p}, & p < \alpha,
    \end{cases}
    +
    \begin{cases}
    t^{-\alpha}, & \min\{q,1\} > \alpha, \\
    t^{-\alpha}\log t, & \min\{q,1\} = \alpha, \\
    t^{-\min\{q,1\}}, & \min\{q,1\} < \alpha,
    \end{cases}\\
    &\lesssim t^{-p}
    + t^{-\min\{q,1\}}+t^{-\min\{p,\alpha\}}+t^{-\min\{q,1\}}\\
    &\lesssim t^{-\min\{p,q,\alpha\}},
    \end{align*}
    up to a log factor. All asymptotic inequalities may be replaced by their equivalents in probability.

    \end{proof}

\section{Additional Simulation Details}\label{apdx:sims}

This section provides a detailed description of the pipeline for all simulations, including the justification for the baseline rates of convergence, underlying data-generating process, numerical grids, estimators, metrics, and the recursive procedure used to combine real and synthetic data, as well as additional figures and tables.

\subsection{Baseline Rates Used in the Experiments}

The theoretical predictions of CRT and BCRT depend on the baseline
convergence rate \(d(\hat \PP_n,\PP_0)=O_p(n^{-p})\). For all estimators
considered in our experiments we take \(p=1/2\), corresponding to the
standard root-\(n\) rate. For empirical distribution estimation this
follows from classical results in \(W_1\) \citep{fournier2015rate} and
for MMD under bounded kernels \citep{gretton2012kernel}. For the KDE
experiments, choosing \(h_n\asymp n^{-1/2}\) preserves the same rate
by a straightforward smoothing argument.

\subsection{Additional CRT Simulation details: ECDF and KDE}

In the CRT simulations, the target distribution $\PP_0$ is a 1-dimensional mixture of two Gaussian components. This smooth distribution serves as the ground truth throughout both experiments. A numerical grid covering the effective support of $\PP_0$ is constructed once and used for all deterministic evaluations of distributional discrepancies. The true density and CDF of $\PP_0$ are available in closed form on this grid.

Both simulations follow the CRT procedure in the section \textbf{Contaminated Recursive Training}. At iteration $t$, a batch of $m_1$ real samples drawn from $\PP_0$ is added to the accumulated dataset, along with $m_2 = ((1-\alpha)/\alpha)m_1$ synthetic samples drawn from the previous generator $\widehat{\PP}_{t-1}$.

In the CRT simulations, we use two estimators: the ECDF as the estimator $\widehat{\PP}_t$, and a KDE whose bandwidth is chosen to correspond to a smoothness parameter $p$; the KDE thus serves as a concrete model with known uncontaminated convergence rate. In both settings, the estimator at iteration $t$ defines a density and distribution function $(\widehat f_t, \widehat F_t)$ evaluated on the numerical grid.

The primary metric is the $W_1$ distance,
\[
W_1(\widehat{\PP}_t, \PP_0)
  \approx \int \bigl| F_0(x) - \widehat F_t(x) \bigr|\,dx,
\]
computed using the trapezoidal rule on the fixed grid. MMD is evaluated using the same plug-in approach applied to the grid-based density estimates.

To estimate convergence rates, we record the sequence of losses
$\{d(\widehat{\PP}_t,\PP_0)\}$ over iterations and fit a power law of the form
\[
\log d(\widehat{\PP}_t,\PP_0)
     = a + b \log M_t.
\]
After discarding an initial burn-in period, the fitted slope $b$ yields the empirical rate. When $\alpha = p$, CRT theory predicts a logarithmic phase transition; in this case, losses are pre-normalized by $\log t$ before regression. The resulting values of $b$ are reported as the observed convergence rates under recursive contamination. All experimental parameters for model training and sampling are summarized in \Cref{tab:sim12-params}.

    \begin{table}[h]
        \centering
        \small
        \begin{tabular}{lcl}
            \toprule
            \textbf{Parameter} & \textbf{Value} & \textbf{Description} \\
            \midrule
            $m_1$ & $50$ & Real samples per iteration \\
            $\alpha$ & $\{0.1,\dots,0.9,1.0\}$ & Real-data fraction \\
            $T$ & $2000$ & Total CRT iterations \\
            $n_{\mathrm{reps}}$ & $100$ & Number of repetitions \\
            $m_{\mathrm{grid}}$ & $200$ & Grid size for deterministic evaluation \\
            $[x_{\min},x_{\max}]$ & Mixture-based & Grid interval for density/CDF evaluation \\
            $w_1$ & $0.35$ & Mixture weight \\
            $\mu_1,\sigma_1$ & $-2.0,\; 0.8$ & First Gaussian component \\
            $\mu_2,\sigma_2$ & $1.0,\; 1.3$ & Second Gaussian component \\
            $h_0$ & $0.5$ & Base KDE bandwidth \\
            \bottomrule
            \end{tabular}
        \caption{CRT Simulation parameters for KDE and ECDF estimators.}
        \label{tab:sim12-params}
    \end{table}

    Finally, we show the estimated densities in each case, for every $\alpha$ in \Cref{fig:sim-crt-ecdf-progress} and \Cref{fig:sim-crt-kde-progress}. Additionally, examples of distributional loss curves and fitted slopes are provided for each alpha in \Cref{fig:sim-crt-ecdf-loss-W1} and \Cref{fig:sim-crt-kde-loss-W1}.
    
\begin{figure}[ht!]
    \centering
    \begin{minipage}{0.18\textwidth}
        \centering
        \includegraphics[width=\linewidth]{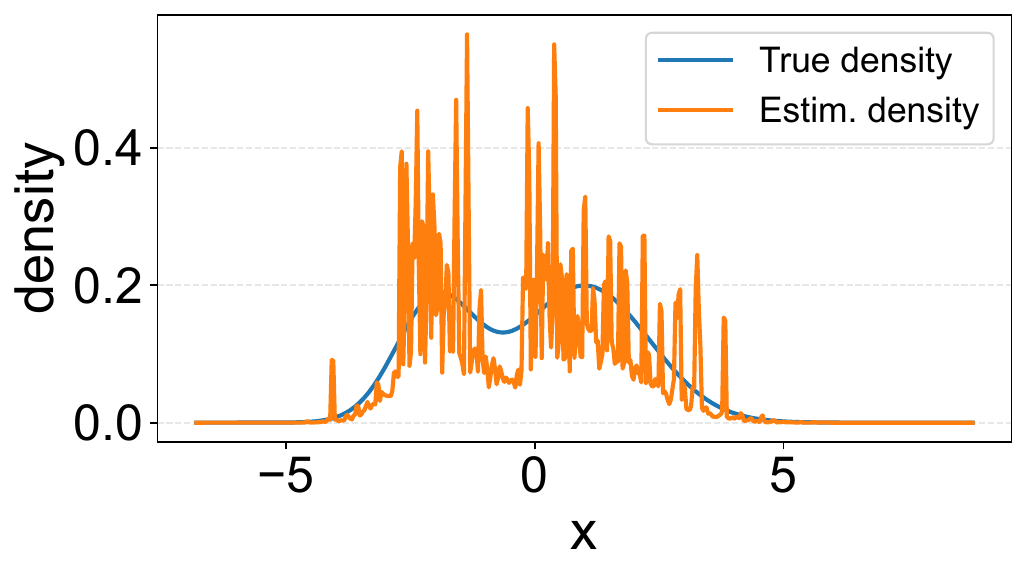}
        \caption*{$\alpha=0.1$}
    \end{minipage}
    \hfill
    \begin{minipage}{0.18\textwidth}
        \centering
        \includegraphics[width=\linewidth]{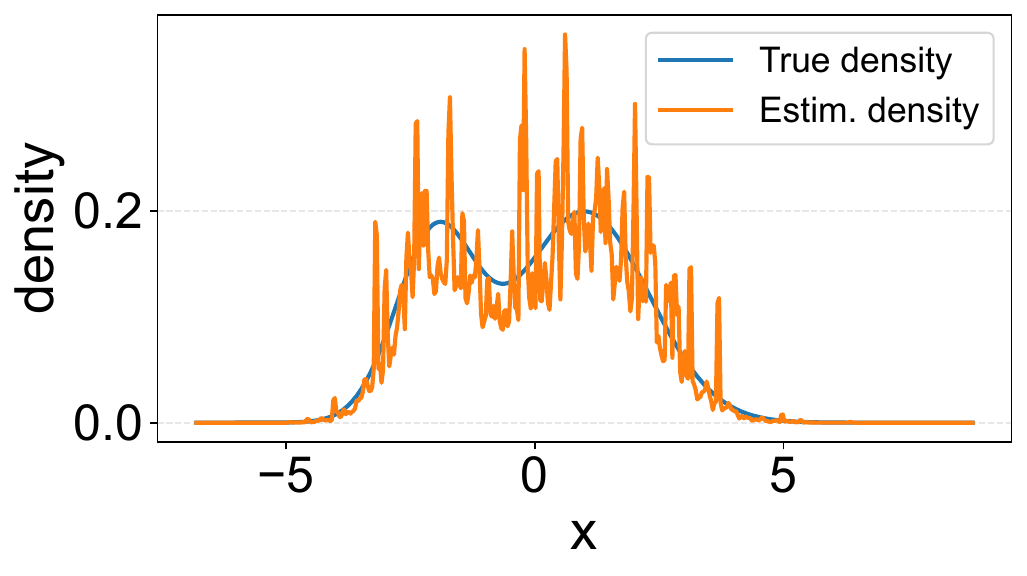}
        \caption*{$\alpha=0.2$}
    \end{minipage}
    \hfill
    \begin{minipage}{0.18\textwidth}
        \centering
        \includegraphics[width=\linewidth]{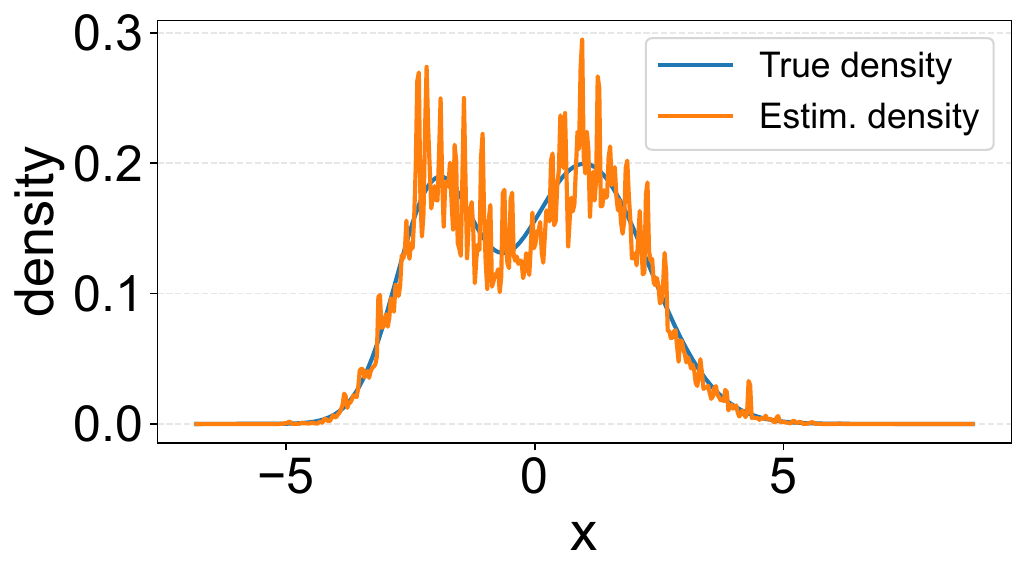}
        \caption*{$\alpha=0.3$}
    \end{minipage}
    \hfill
    \begin{minipage}{0.18\textwidth}
        \centering
        \includegraphics[width=\linewidth]{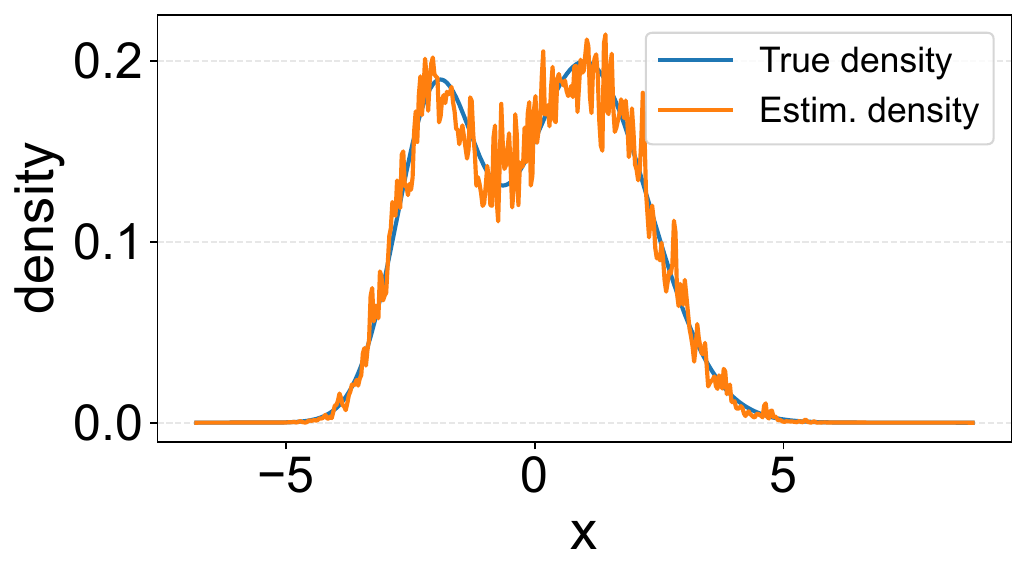}
        \caption*{$\alpha=0.4$}
    \end{minipage}
    \hfill
    \begin{minipage}{0.18\textwidth}
        \centering
        \includegraphics[width=\linewidth]{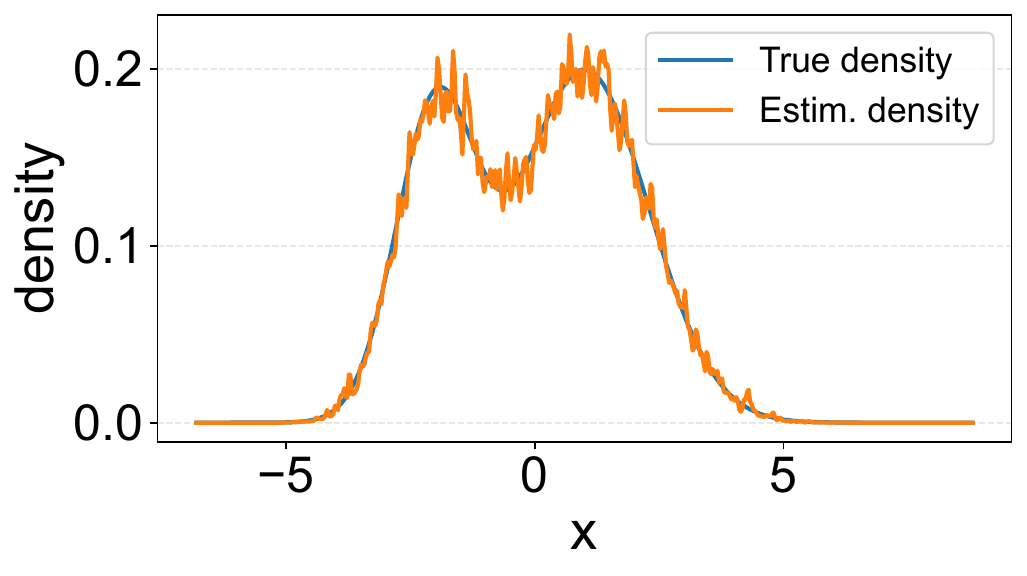}
        \caption*{$\alpha=0.5$}
    \end{minipage}

    \vspace{1em}

    \begin{minipage}{0.18\textwidth}
        \centering
        \includegraphics[width=\linewidth]{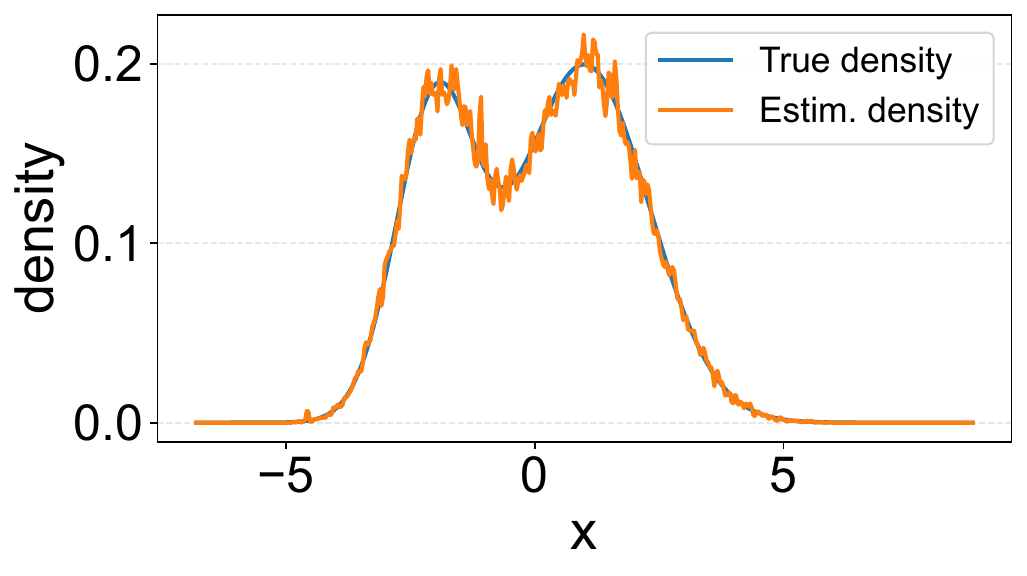}
        \caption*{$\alpha=0.6$}
    \end{minipage}
    \hfill
    \begin{minipage}{0.18\textwidth}
        \centering
        \includegraphics[width=\linewidth]{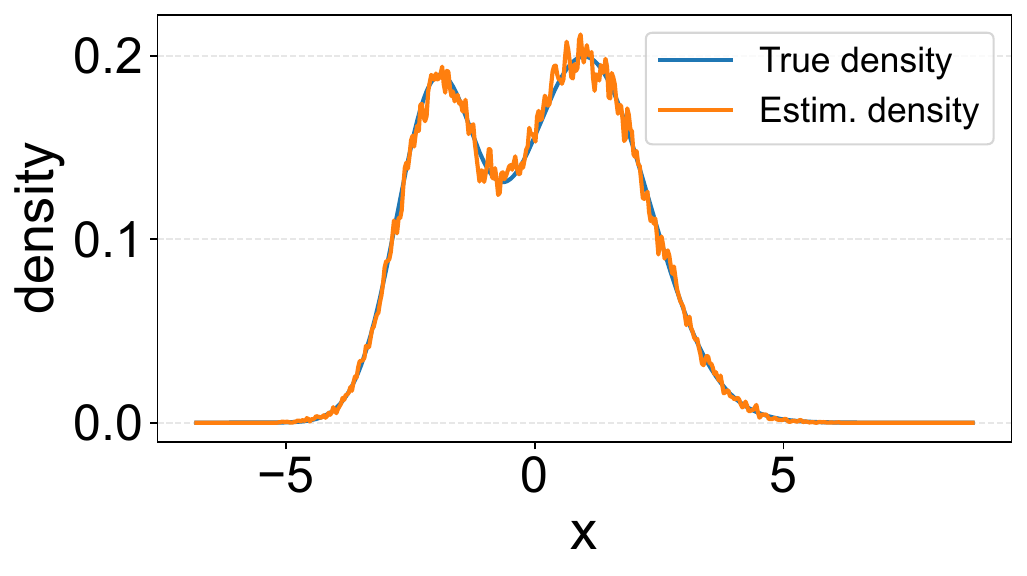}
        \caption*{$\alpha=0.7$}
    \end{minipage}
    \hfill
    \begin{minipage}{0.18\textwidth}
        \centering
        \includegraphics[width=\linewidth]{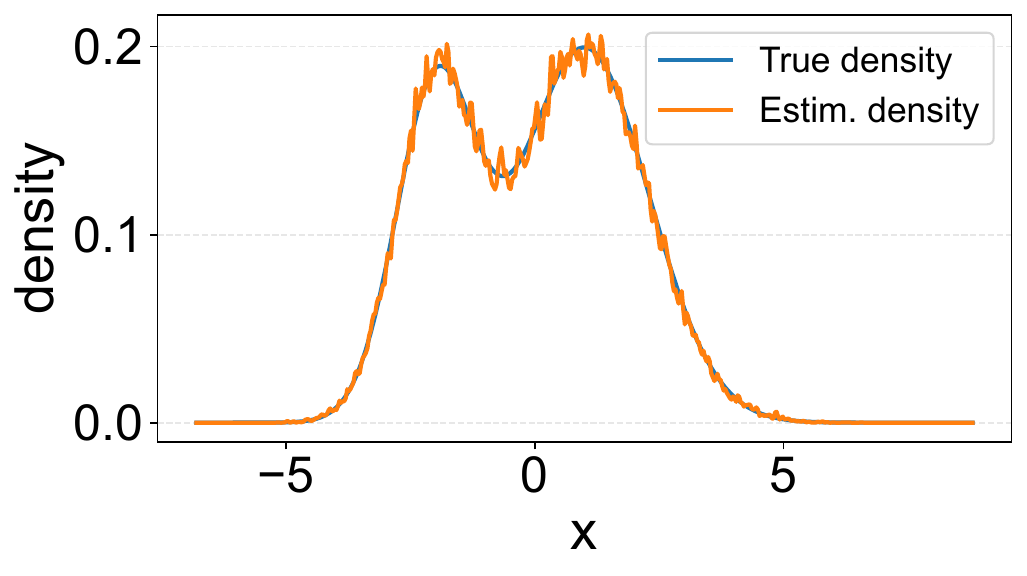}
        \caption*{$\alpha=0.8$}
    \end{minipage}
    \hfill
    \begin{minipage}{0.18\textwidth}
        \centering
        \includegraphics[width=\linewidth]{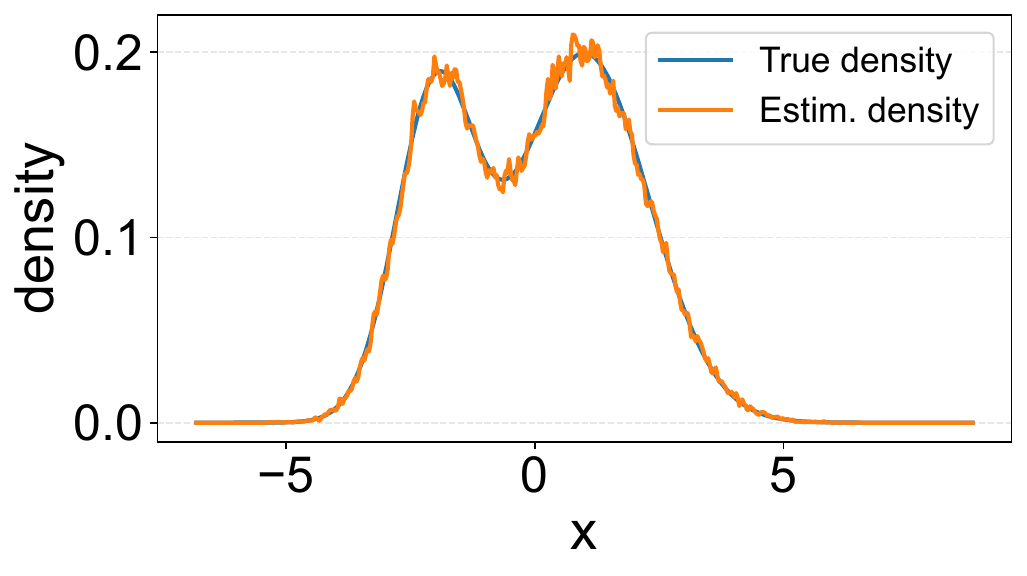}
        \caption*{$\alpha=0.9$}
    \end{minipage}
    \hfill
    \begin{minipage}{0.18\textwidth}
        \centering
        \includegraphics[width=\linewidth]{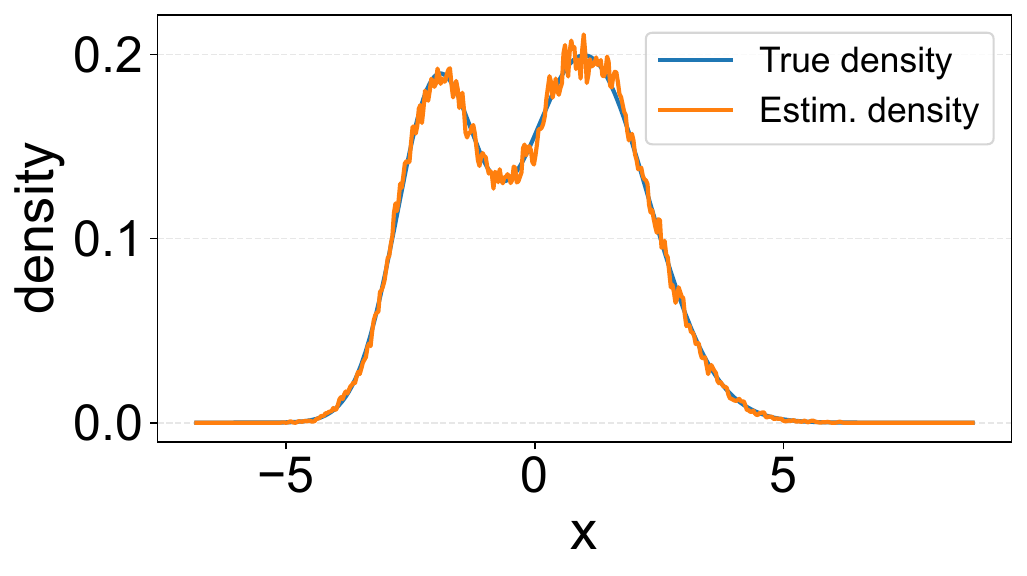}
        \caption*{$\alpha=1.0$}
    \end{minipage}

    \caption{CRT Simulation (ECDF Estimator). Final output distributions $\widehat{\PP}_T$ for varying values of $\alpha$ (real-data fraction) where all models are run for the same number of CRT iterations. This fixed compute budget for all models was chosen to minimize the effect of numerical plateaus in the loss when computing optimal transport distances after models have fully converged. Under this fixed compute budget, convergence improves as $\alpha$ increases; however, for all $\alpha$ we observe the predicted rate of convergence.}
    \label{fig:sim-crt-ecdf-progress}

\end{figure}

    \begin{figure}[ht!]
        \centering
    
        \begin{minipage}{0.18\textwidth}
            \centering
            \includegraphics[width=\linewidth]{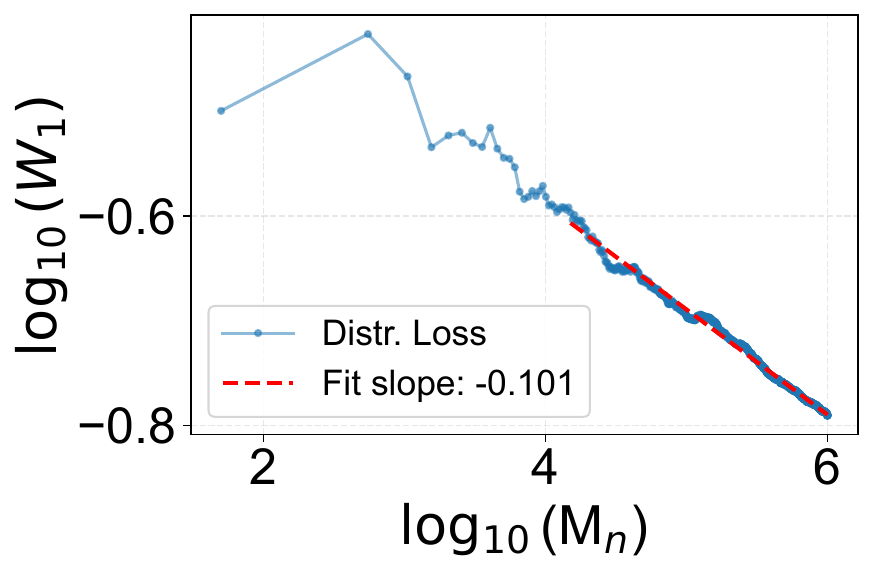}
        \end{minipage}
        \hfill
        \begin{minipage}{0.18\textwidth}
            \centering
            \includegraphics[width=\linewidth]{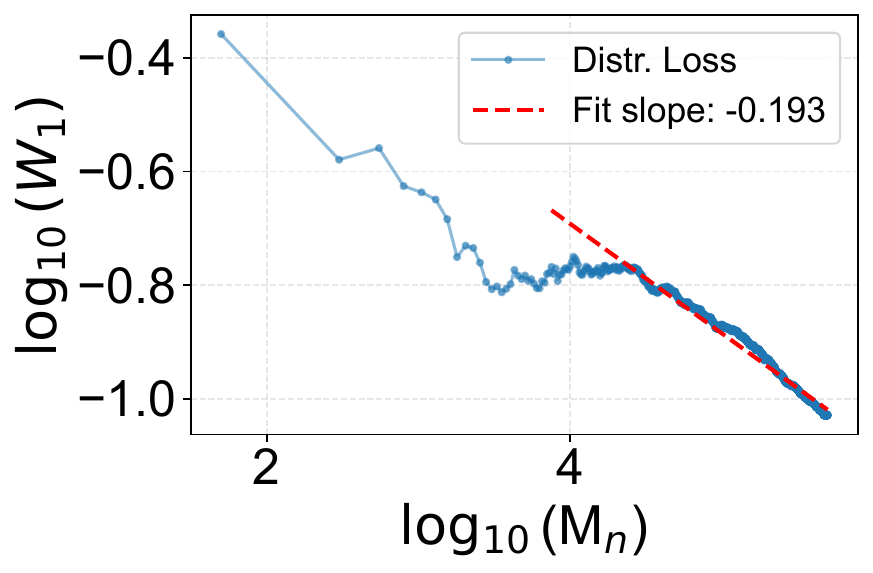}
        \end{minipage}
        \hfill
        \begin{minipage}{0.18\textwidth}
            \centering
            \includegraphics[width=\linewidth]{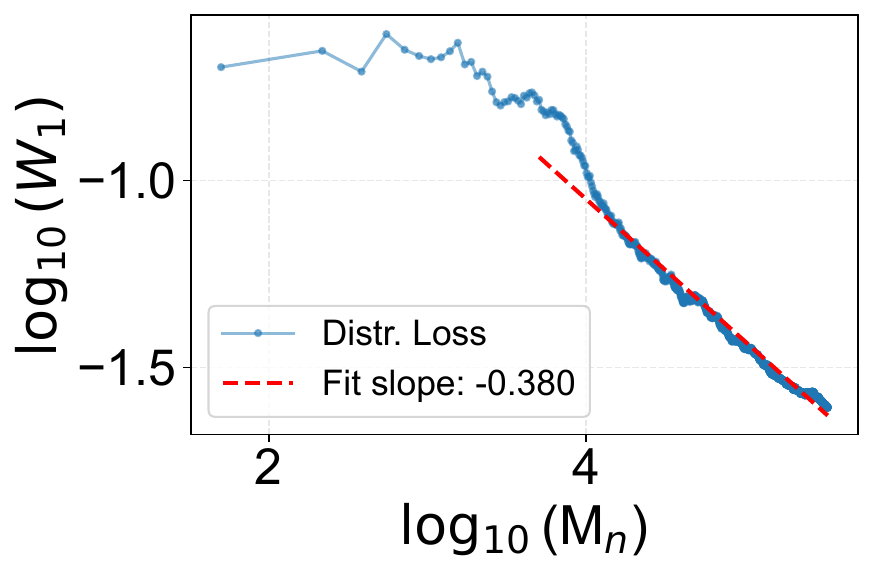}
        \end{minipage}
        \hfill
        \begin{minipage}{0.18\textwidth}
            \centering
            \includegraphics[width=\linewidth]{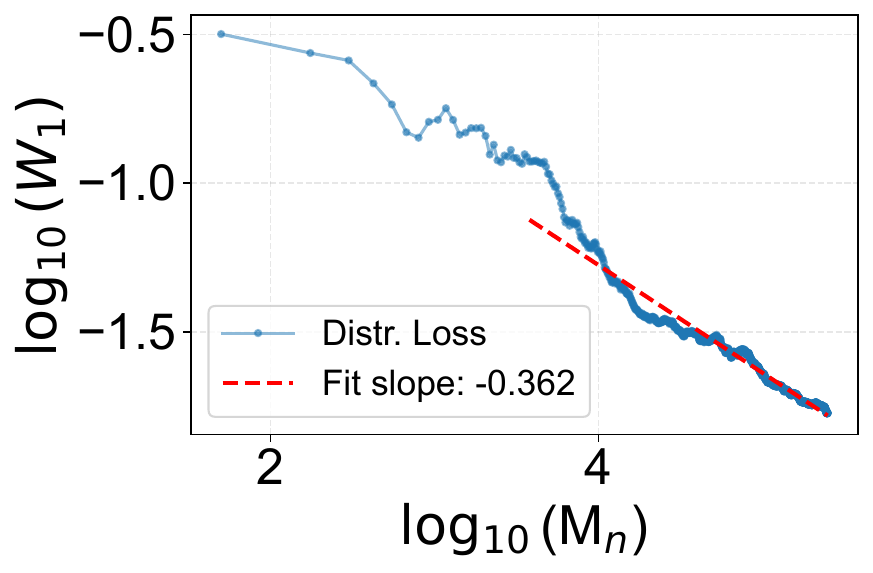}
        \end{minipage}
        \hfill
        \begin{minipage}{0.18\textwidth}
            \centering
            \includegraphics[width=\linewidth]{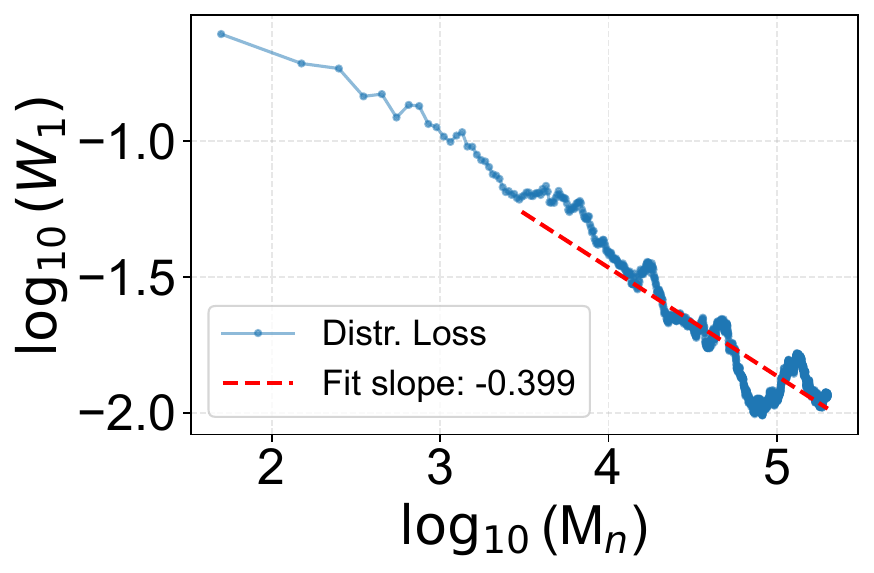}
        \end{minipage}
    
    
        \begin{minipage}{0.18\textwidth}
            \centering
            \includegraphics[width=\linewidth]{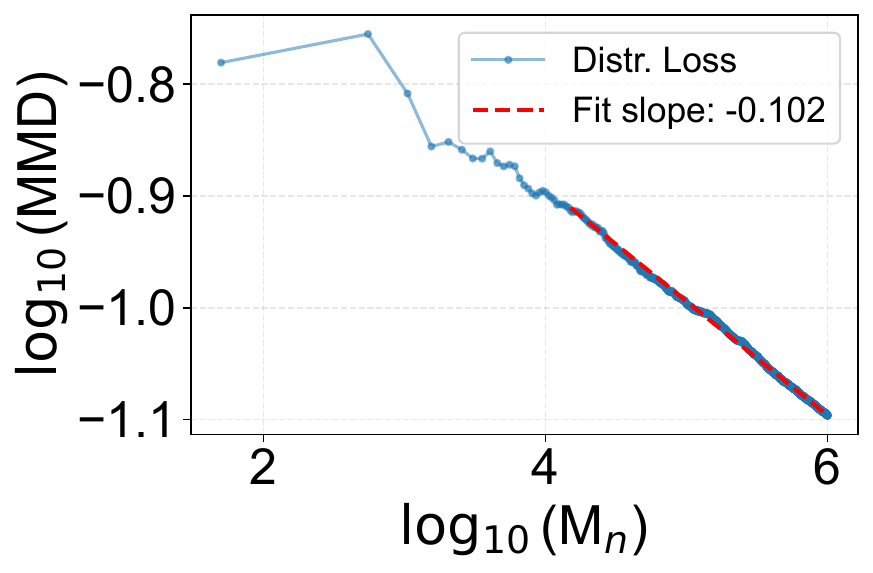}
            \caption*{\small $\alpha=0.10$}
        \end{minipage}
        \hfill
        \begin{minipage}{0.18\textwidth}
            \centering
            \includegraphics[width=\linewidth]{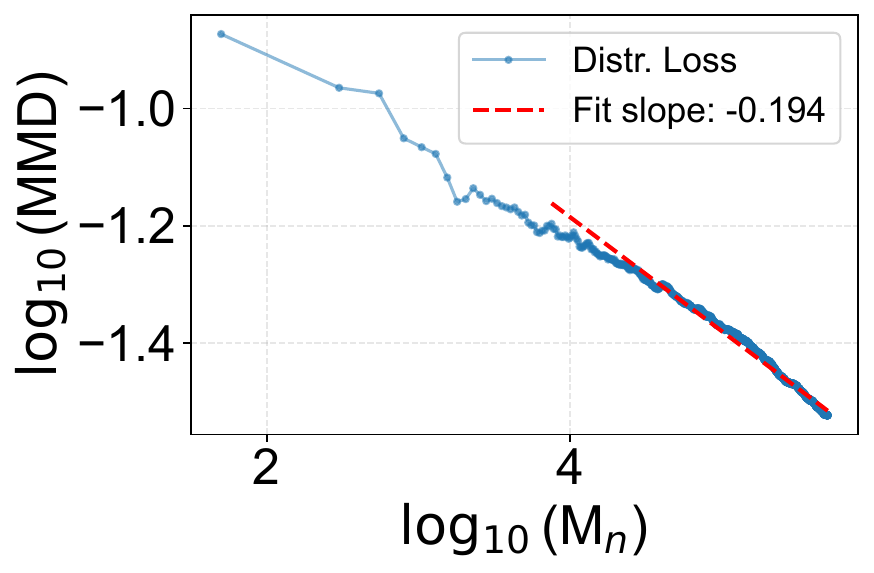}
            \caption*{\small $\alpha=0.20$}
        \end{minipage}
        \hfill
        \begin{minipage}{0.18\textwidth}
            \centering
            \includegraphics[width=\linewidth]{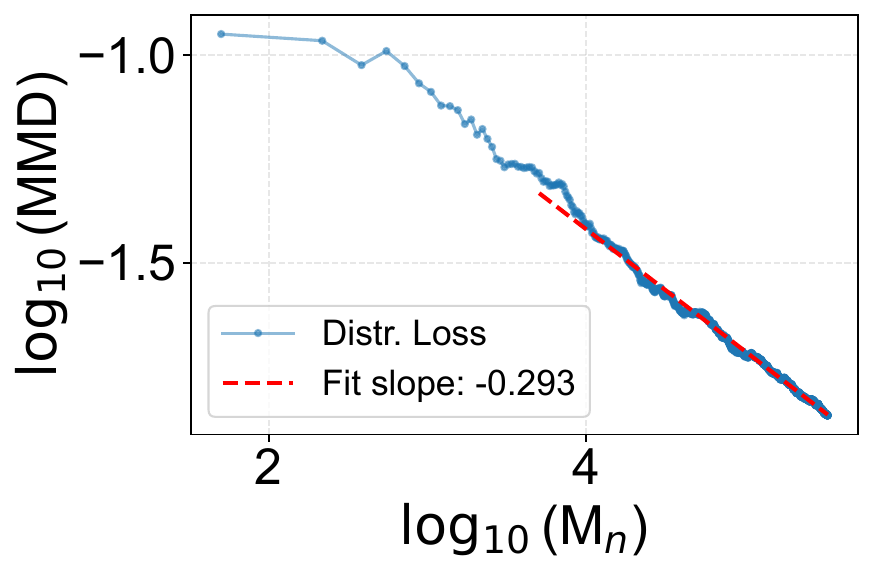}
            \caption*{\small $\alpha=0.30$}
        \end{minipage}
        \hfill
        \begin{minipage}{0.18\textwidth}
            \centering
            \includegraphics[width=\linewidth]{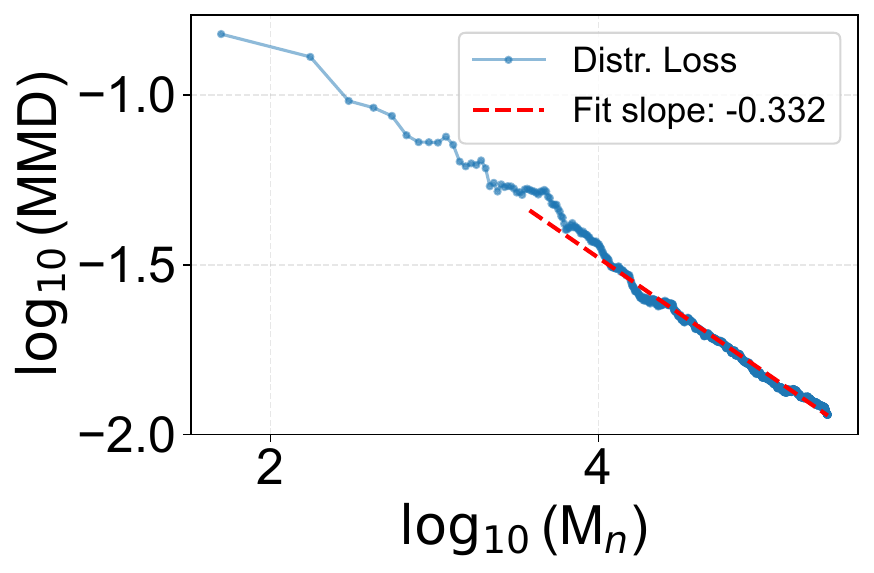}
            \caption*{\small $\alpha=0.40$}
        \end{minipage}
        \hfill
        \begin{minipage}{0.18\textwidth}
            \centering
            \includegraphics[width=\linewidth]{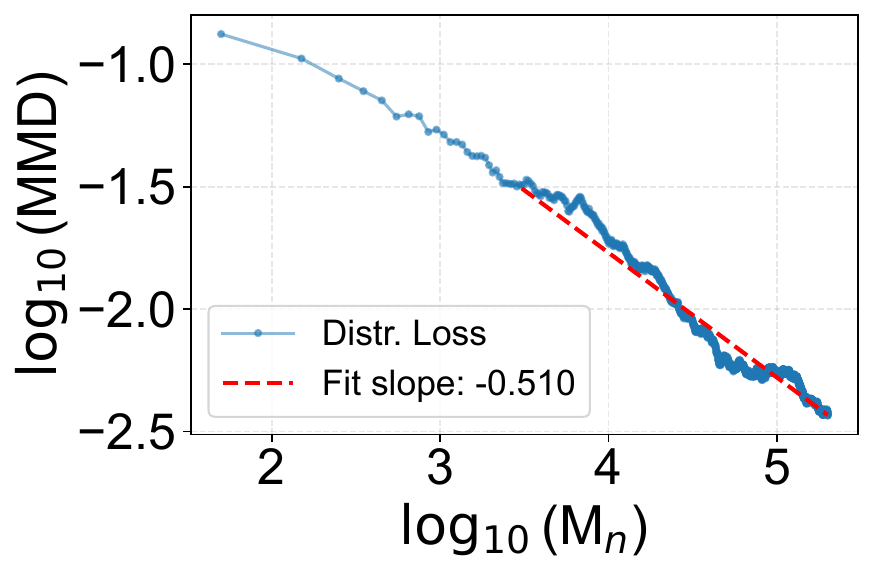}
            \caption*{\small $\alpha=0.50$}
        \end{minipage}

        \vspace{1em}

        \begin{minipage}{0.18\textwidth}
            \centering
            \includegraphics[width=\linewidth]{Figures/sim1-ecdf-progress/p=0.50_alpha=0.10_W1.pdf}
        \end{minipage}
        \hfill
        \begin{minipage}{0.18\textwidth}
            \centering
            \includegraphics[width=\linewidth]{Figures/sim1-ecdf-progress/p=0.50_alpha=0.20_W1.pdf}
        \end{minipage}
        \hfill
        \begin{minipage}{0.18\textwidth}
            \centering
            \includegraphics[width=\linewidth]{Figures/sim1-ecdf-progress/p=0.50_alpha=0.30_W1.pdf}
        \end{minipage}
        \hfill
        \begin{minipage}{0.18\textwidth}
            \centering
            \includegraphics[width=\linewidth]{Figures/sim1-ecdf-progress/p=0.50_alpha=0.40_W1.pdf}
        \end{minipage}
        \hfill
        \begin{minipage}{0.18\textwidth}
            \centering
            \includegraphics[width=\linewidth]{Figures/sim1-ecdf-progress/p=0.50_alpha=0.50_W1.pdf}
        \end{minipage}

    
        \begin{minipage}{0.18\textwidth}
            \centering
            \includegraphics[width=\linewidth]{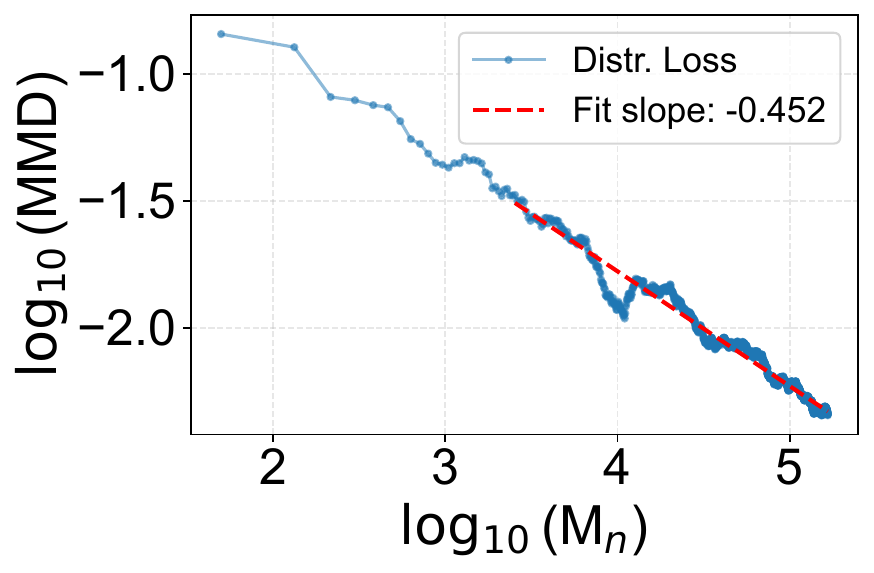}
            \caption*{\small $\alpha=0.60$}
        \end{minipage}
        \hfill
        \begin{minipage}{0.18\textwidth}
            \centering
            \includegraphics[width=\linewidth]{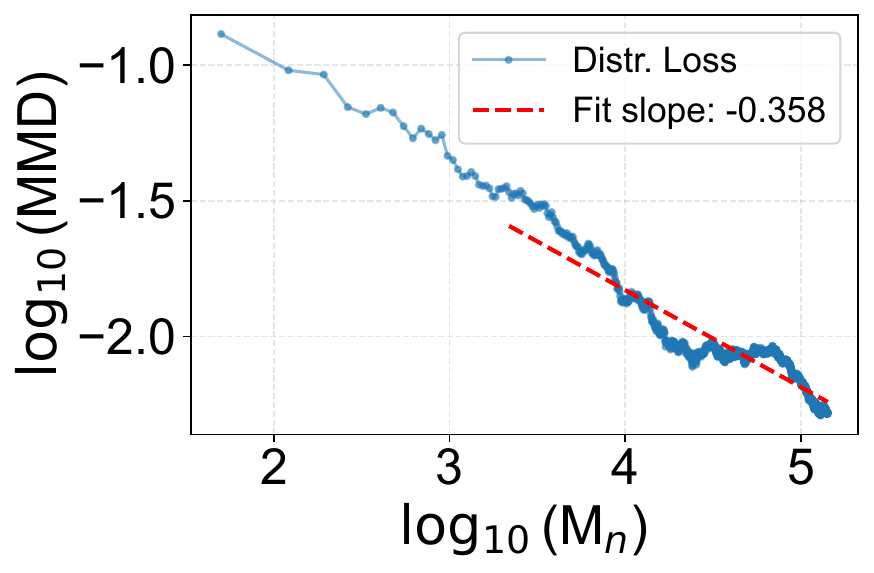}
            \caption*{\small $\alpha=0.70$}
        \end{minipage}
        \hfill
        \begin{minipage}{0.18\textwidth}
            \centering
            \includegraphics[width=\linewidth]{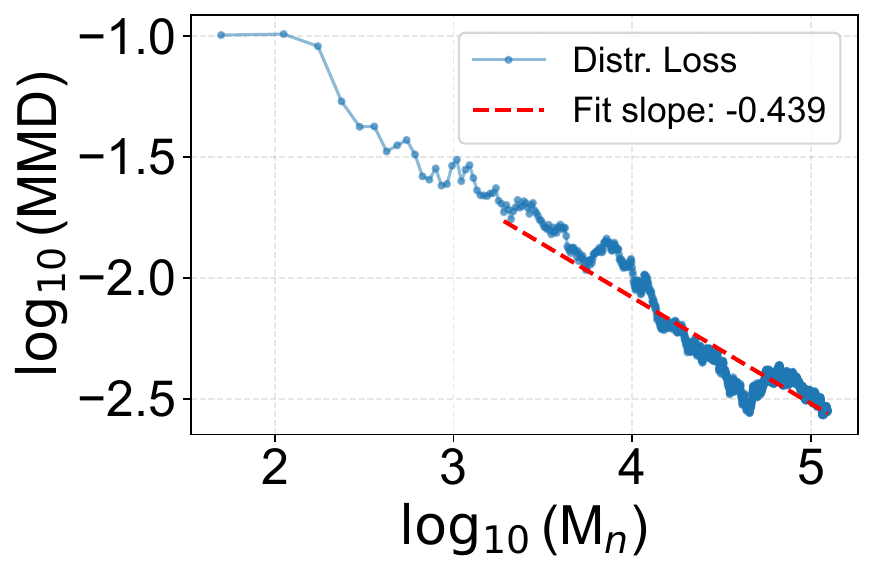}
            \caption*{\small $\alpha=0.80$}
        \end{minipage}
        \hfill
        \begin{minipage}{0.18\textwidth}
            \centering
            \includegraphics[width=\linewidth]{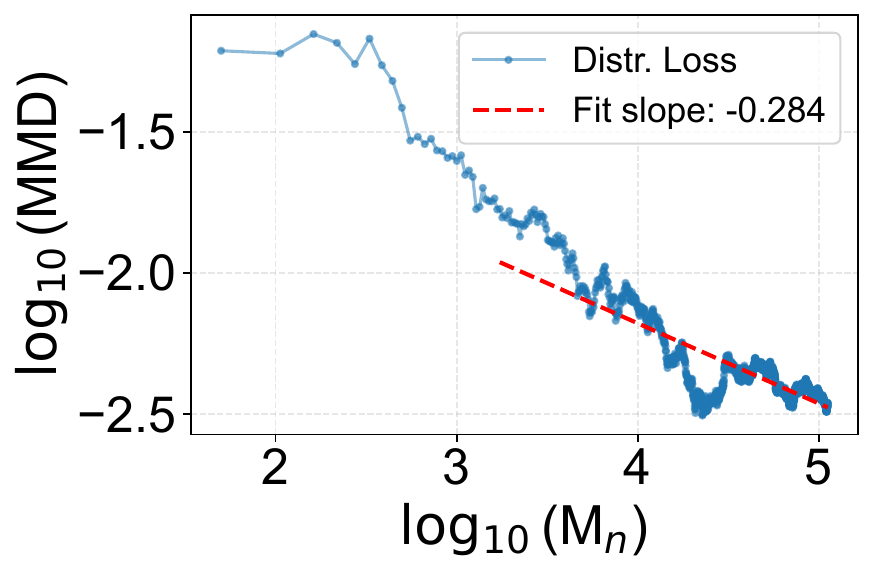}
            \caption*{\small $\alpha=0.90$}
        \end{minipage}
        \hfill
        \begin{minipage}{0.18\textwidth}
            \centering
            \includegraphics[width=\linewidth]{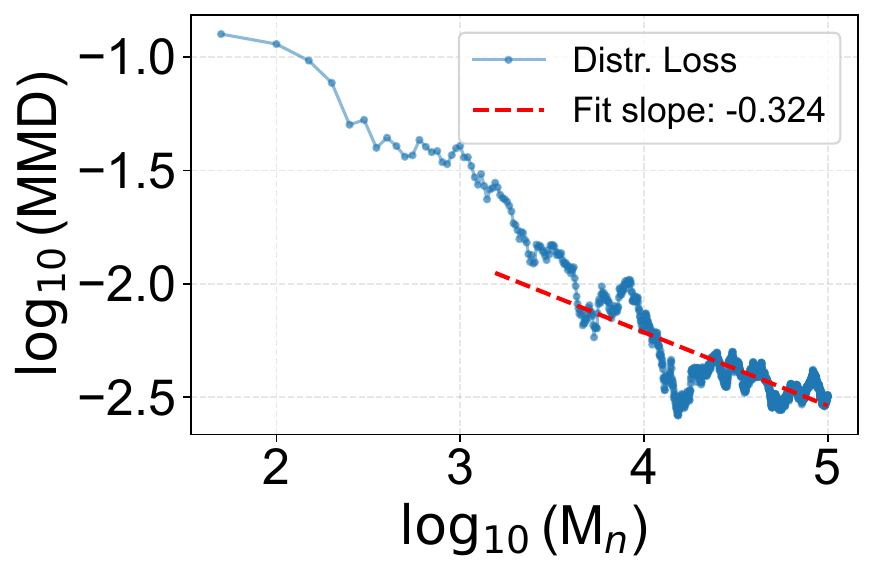}
            \caption*{\small $\alpha=1.00$}
        \end{minipage}

        \caption{CRT Simulation (ECDF Estimator). $W_1$ and MMD distributional losses and fitted slopes for varying values of $\alpha$ (real-data fraction) for a single run.
        }
        \label{fig:sim-crt-ecdf-loss-W1}
    \end{figure}

\begin{figure}[ht!]
    \centering

    \begin{minipage}{0.18\textwidth}
        \centering
        \includegraphics[width=\linewidth]{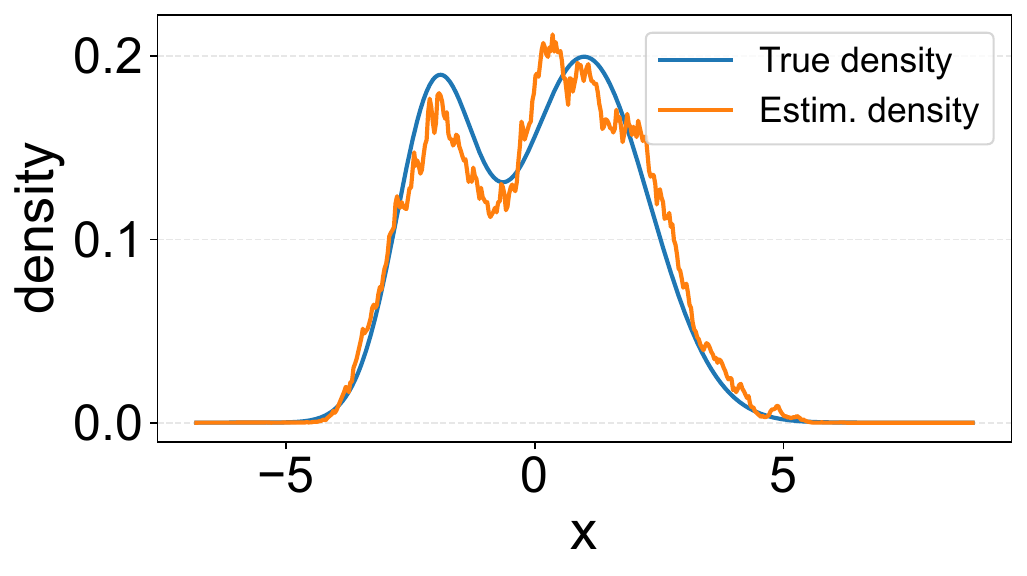}
        \caption*{$\alpha=0.1$}
    \end{minipage}
    \hfill
    \begin{minipage}{0.18\textwidth}
        \centering
        \includegraphics[width=\linewidth]{Figures/sim1-kde-progress/density_saved_p=0.50_alpha=0.10.pdf}
        \caption*{$\alpha=0.2$}
    \end{minipage}
    \hfill
    \begin{minipage}{0.18\textwidth}
        \centering
        \includegraphics[width=\linewidth]{Figures/sim1-kde-progress/density_saved_p=0.50_alpha=0.10.pdf}
        \caption*{$\alpha=0.3$}
    \end{minipage}
    \hfill
    \begin{minipage}{0.18\textwidth}
        \centering
        \includegraphics[width=\linewidth]{Figures/sim1-kde-progress/density_saved_p=0.50_alpha=0.10.pdf}
        \caption*{$\alpha=0.4$}
    \end{minipage}
    \hfill
    \begin{minipage}{0.18\textwidth}
        \centering
        \includegraphics[width=\linewidth]{Figures/sim1-kde-progress/density_saved_p=0.50_alpha=0.10.pdf}
        \caption*{$\alpha=0.5$}
    \end{minipage}

    \vspace{1em}

    \begin{minipage}{0.18\textwidth}
        \centering
        \includegraphics[width=\linewidth]{Figures/sim1-kde-progress/density_saved_p=0.50_alpha=0.10.pdf}
        \caption*{$\alpha=0.6$}
    \end{minipage}
    \hfill
    \begin{minipage}{0.18\textwidth}
        \centering
        \includegraphics[width=\linewidth]{Figures/sim1-kde-progress/density_saved_p=0.50_alpha=0.10.pdf}
        \caption*{$\alpha=0.7$}
    \end{minipage}
    \hfill
    \begin{minipage}{0.18\textwidth}
        \centering
        \includegraphics[width=\linewidth]{Figures/sim1-kde-progress/density_saved_p=0.50_alpha=0.10.pdf}
        \caption*{$\alpha=0.8$}
    \end{minipage}
    \hfill
    \begin{minipage}{0.18\textwidth}
        \centering
        \includegraphics[width=\linewidth]{Figures/sim1-kde-progress/density_saved_p=0.50_alpha=0.10.pdf}
        \caption*{$\alpha=0.9$}
    \end{minipage}
    \hfill
    \begin{minipage}{0.18\textwidth}
        \centering
        \includegraphics[width=\linewidth]{Figures/sim1-kde-progress/density_saved_p=0.50_alpha=0.10.pdf}
        \caption*{$\alpha=1.0$}
    \end{minipage}

    \caption{CRT Simulation (KDE Estimator). Final output distributions $\widehat{\PP}_T$ for varying values of $\alpha$ (real-data fraction) where all models are run for the same number of CRT iterations. This fixed compute budget for all models was chosen to minimize the effect of numerical plateaus in the loss when computing optimal transport distances after models have fully converged. Under this fixed compute budget, convergence improves as $\alpha$ increases; however, for all $\alpha$ we observe the predicted rate of convergence.}
    \label{fig:sim-crt-kde-progress}
\end{figure}

    \begin{figure}[ht!]
        \centering
    
        \begin{minipage}{0.18\textwidth}
            \centering
            \includegraphics[width=\linewidth]{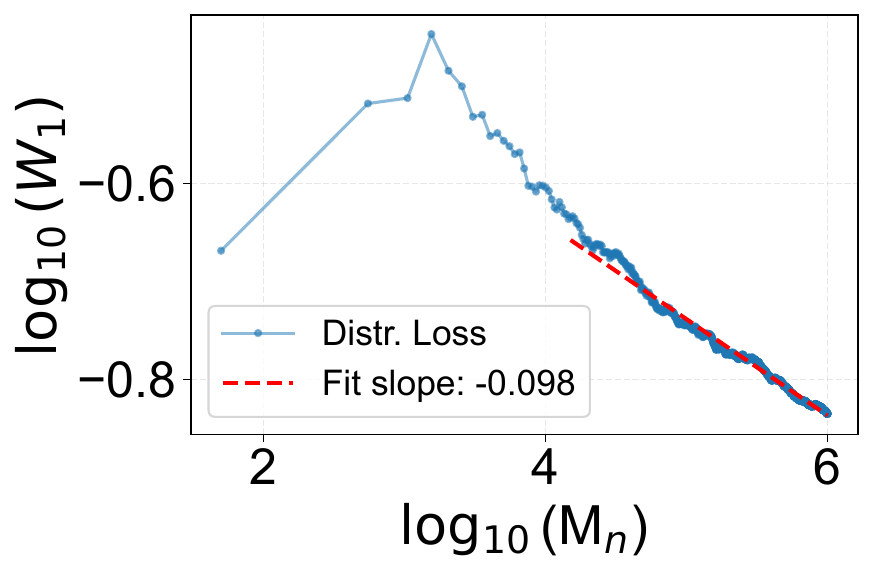}
        \end{minipage}
        \hfill
        \begin{minipage}{0.18\textwidth}
            \centering
            \includegraphics[width=\linewidth]{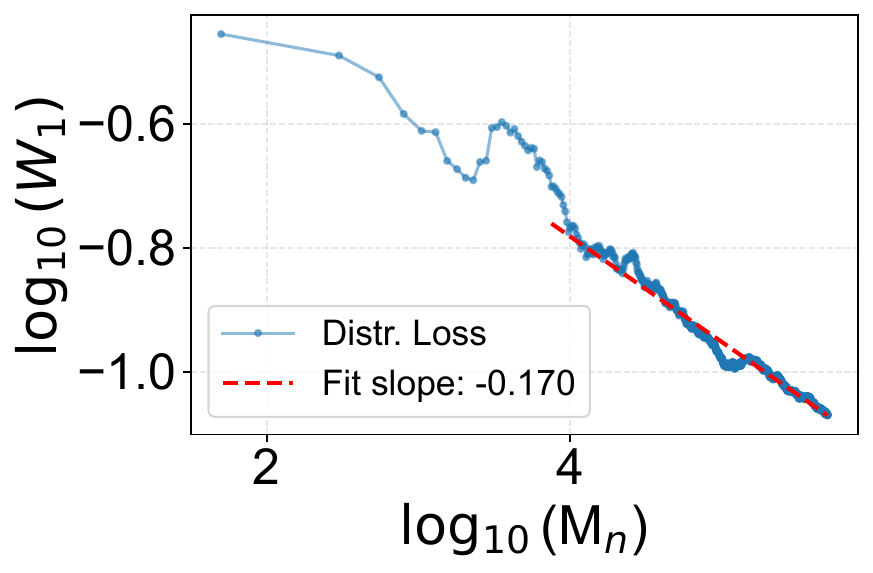}
        \end{minipage}
        \hfill
        \begin{minipage}{0.18\textwidth}
            \centering
            \includegraphics[width=\linewidth]{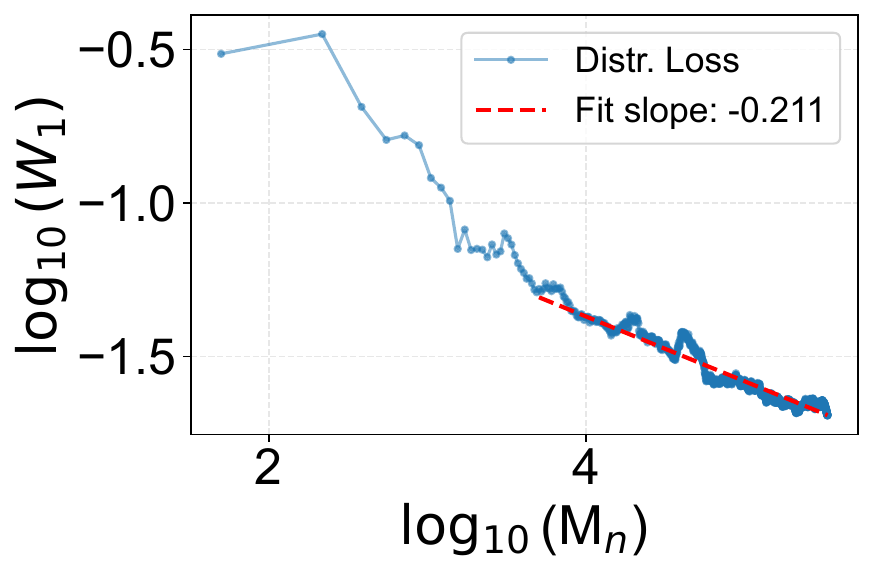}
        \end{minipage}
        \hfill
        \begin{minipage}{0.18\textwidth}
            \centering
            \includegraphics[width=\linewidth]{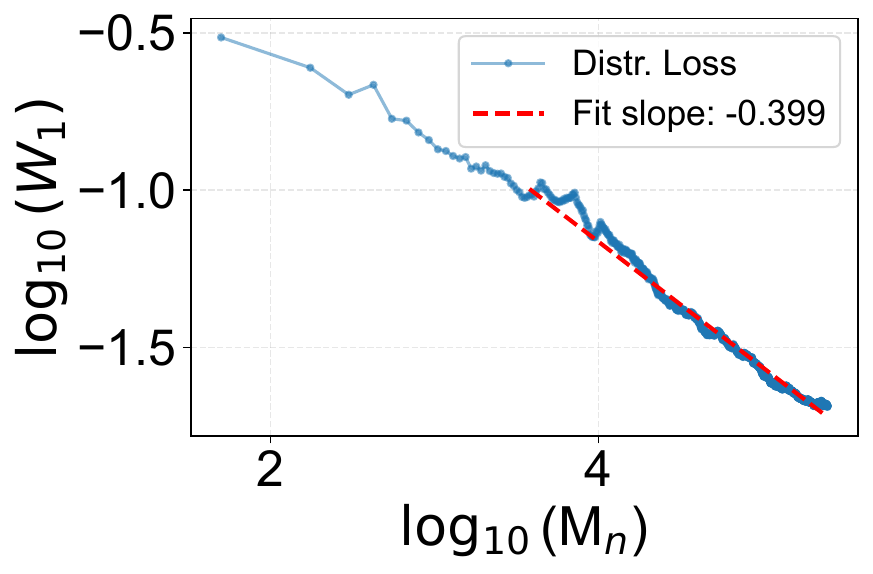}
        \end{minipage}
        \hfill
        \begin{minipage}{0.18\textwidth}
            \centering
            \includegraphics[width=\linewidth]{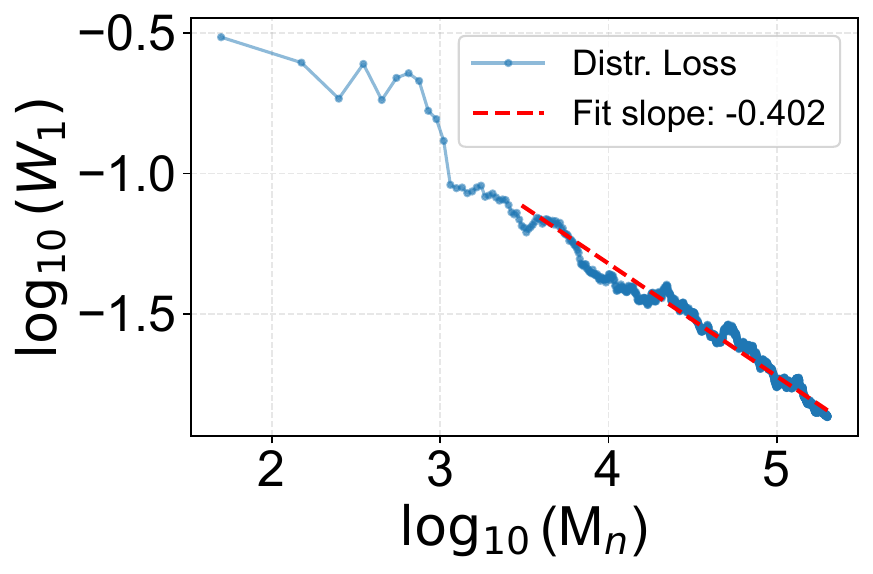}
        \end{minipage}
    
    
        \begin{minipage}{0.18\textwidth}
            \centering
            \includegraphics[width=\linewidth]{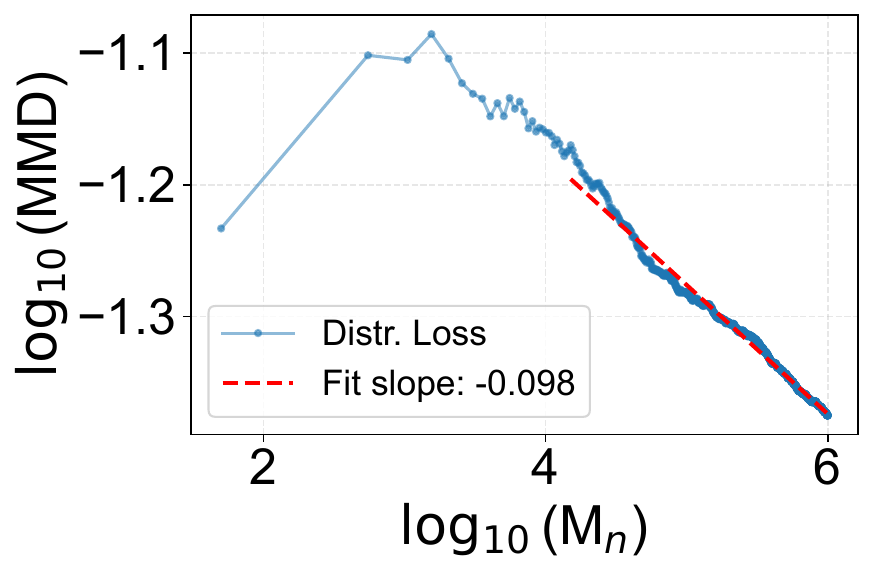}
            \caption*{\small $\alpha=0.10$}
        \end{minipage}
        \hfill
        \begin{minipage}{0.18\textwidth}
            \centering
            \includegraphics[width=\linewidth]{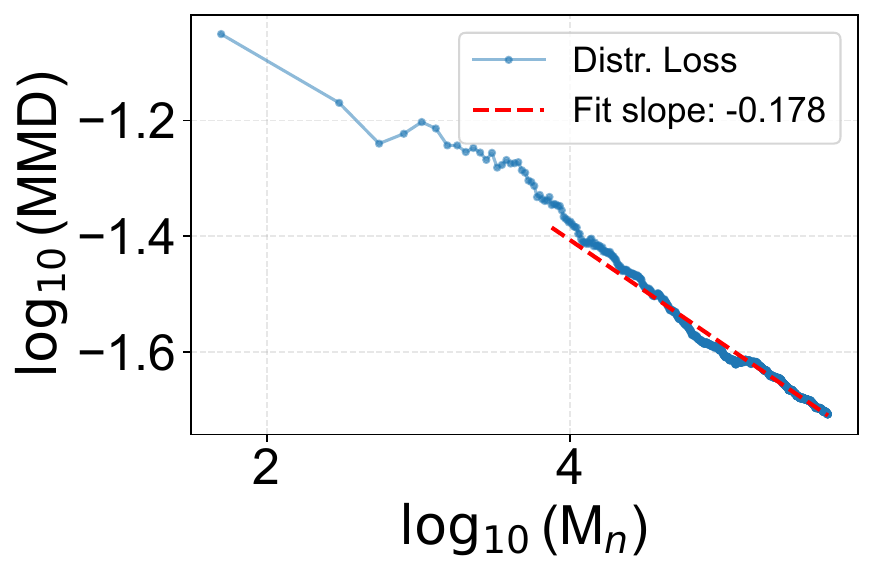}
            \caption*{\small $\alpha=0.20$}
        \end{minipage}
        \hfill
        \begin{minipage}{0.18\textwidth}
            \centering
            \includegraphics[width=\linewidth]{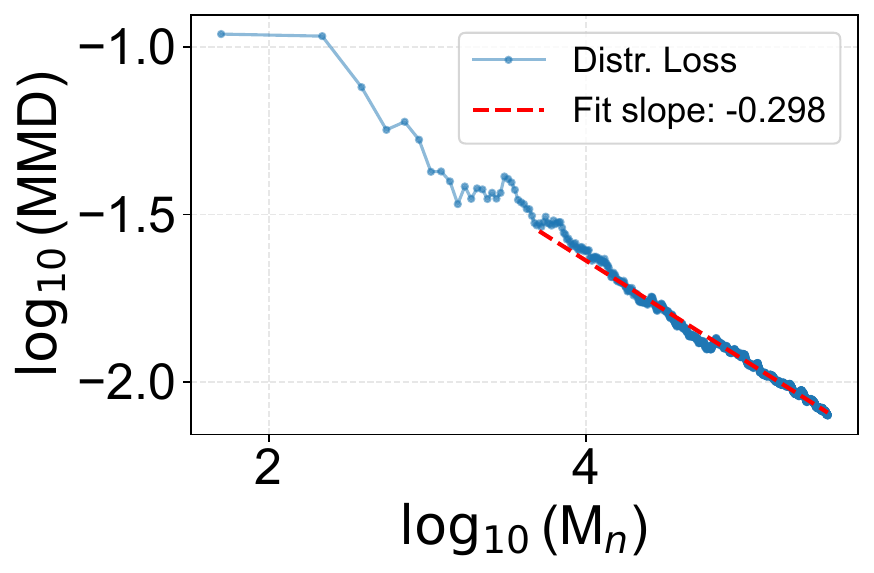}
            \caption*{\small $\alpha=0.30$}
        \end{minipage}
        \hfill
        \begin{minipage}{0.18\textwidth}
            \centering
            \includegraphics[width=\linewidth]{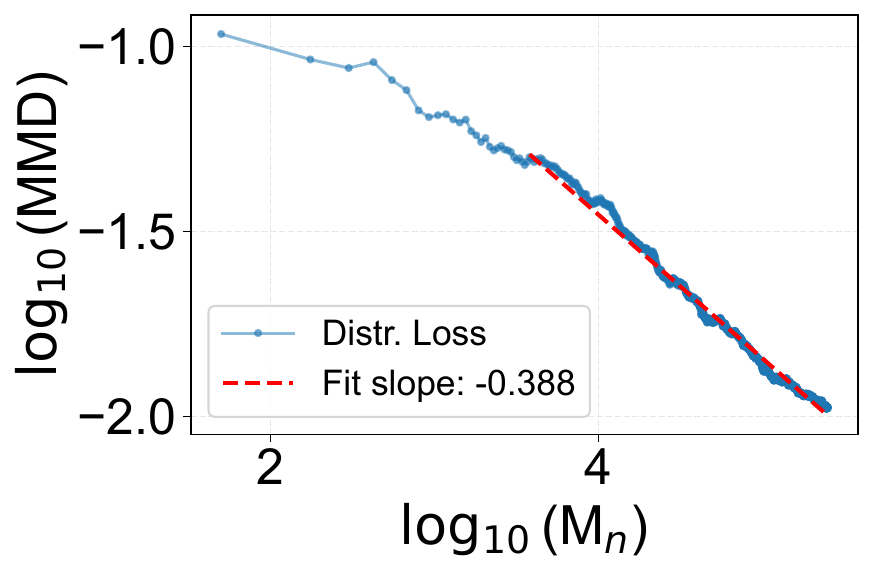}
            \caption*{\small $\alpha=0.40$}
        \end{minipage}
        \hfill
        \begin{minipage}{0.18\textwidth}
            \centering
            \includegraphics[width=\linewidth]{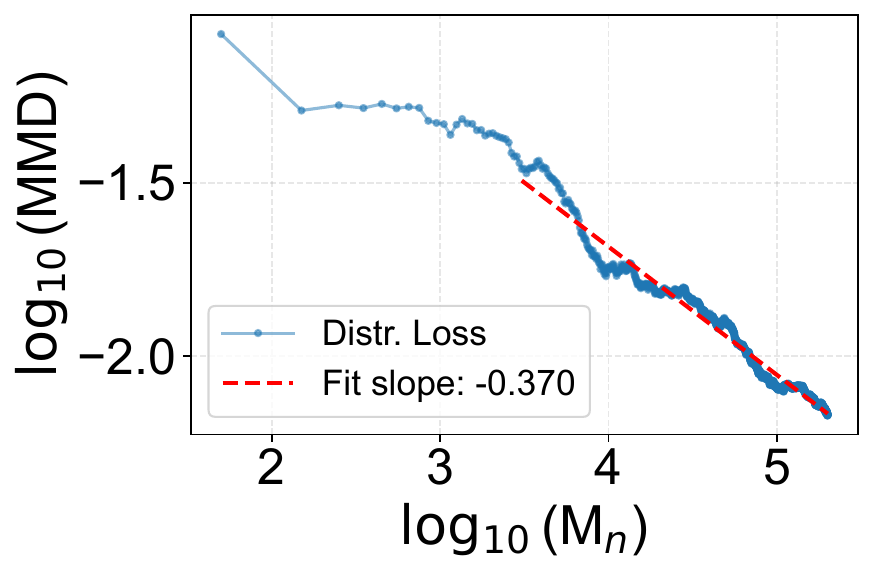}
            \caption*{\small $\alpha=0.50$}
        \end{minipage}
        
        \vspace{1em}
        
        \begin{minipage}{0.18\textwidth}
            \centering
            \includegraphics[width=\linewidth]{Figures/sim1-kde-progress/p=0.50_alpha=0.10_W1.pdf}
        \end{minipage}
        \hfill
        \begin{minipage}{0.18\textwidth}
            \centering
            \includegraphics[width=\linewidth]{Figures/sim1-kde-progress/p=0.50_alpha=0.20_W1.pdf}
        \end{minipage}
        \hfill
        \begin{minipage}{0.18\textwidth}
            \centering
            \includegraphics[width=\linewidth]{Figures/sim1-kde-progress/p=0.50_alpha=0.30_W1.pdf}
        \end{minipage}
        \hfill
        \begin{minipage}{0.18\textwidth}
            \centering
            \includegraphics[width=\linewidth]{Figures/sim1-kde-progress/p=0.50_alpha=0.40_W1.pdf}
        \end{minipage}
        \hfill
        \begin{minipage}{0.18\textwidth}
            \centering
            \includegraphics[width=\linewidth]{Figures/sim1-kde-progress/p=0.50_alpha=0.50_W1.pdf}
        \end{minipage}
    
    
        \begin{minipage}{0.18\textwidth}
            \centering
            \includegraphics[width=\linewidth]{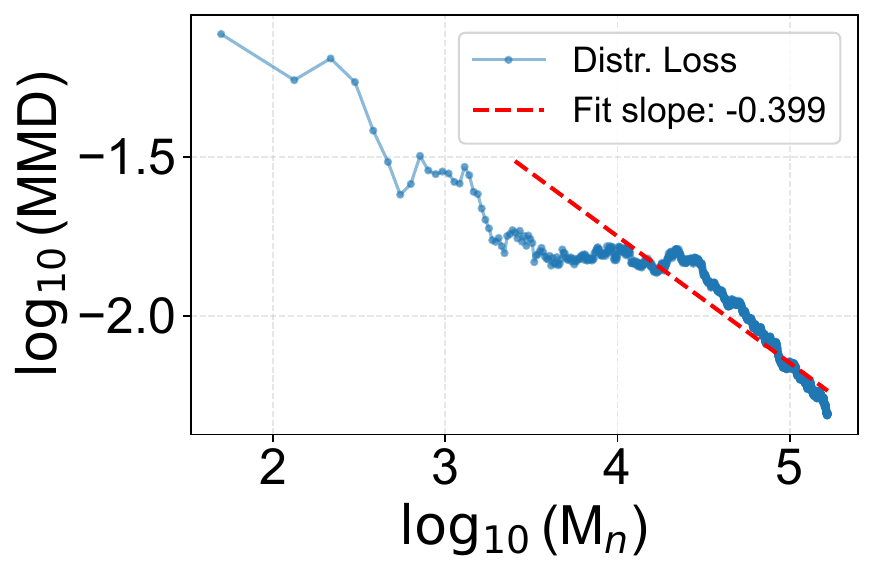}
            \caption*{\small $\alpha=0.60$}
        \end{minipage}
        \hfill
        \begin{minipage}{0.18\textwidth}
            \centering
            \includegraphics[width=\linewidth]{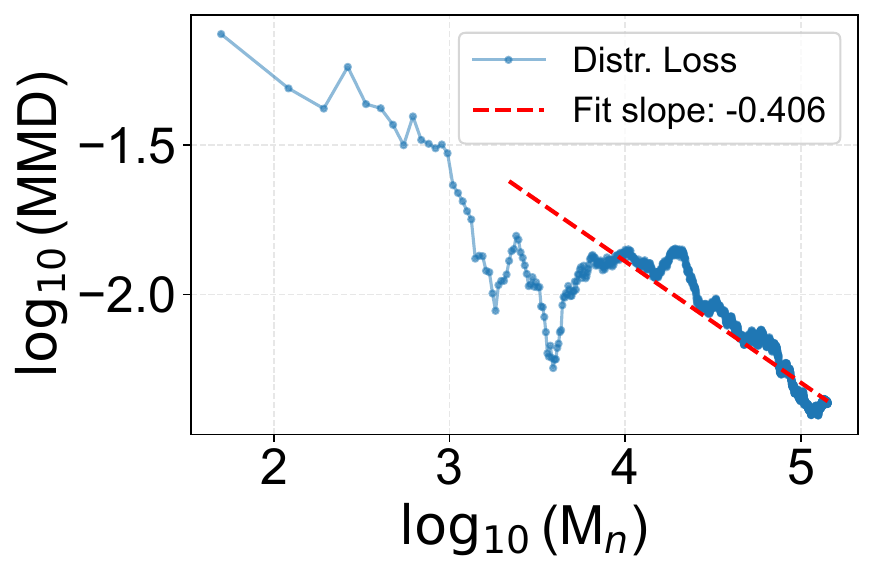}
            \caption*{\small $\alpha=0.70$}
        \end{minipage}
        \hfill
        \begin{minipage}{0.18\textwidth}
            \centering
            \includegraphics[width=\linewidth]{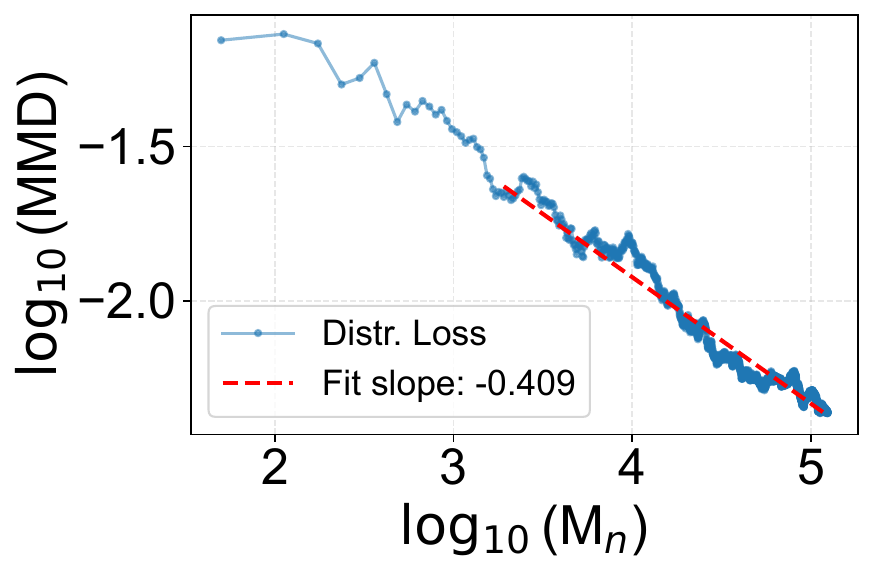}
            \caption*{\small $\alpha=0.80$}
        \end{minipage}
        \hfill
        \begin{minipage}{0.18\textwidth}
            \centering
            \includegraphics[width=\linewidth]{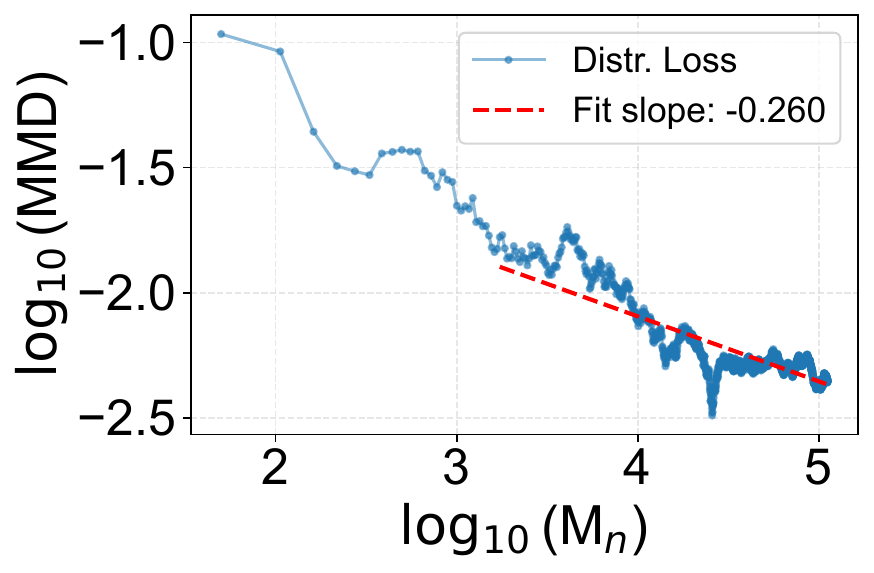}
            \caption*{\small $\alpha=0.90$}
        \end{minipage}
        \hfill
        \begin{minipage}{0.18\textwidth}
            \centering
            \includegraphics[width=\linewidth]{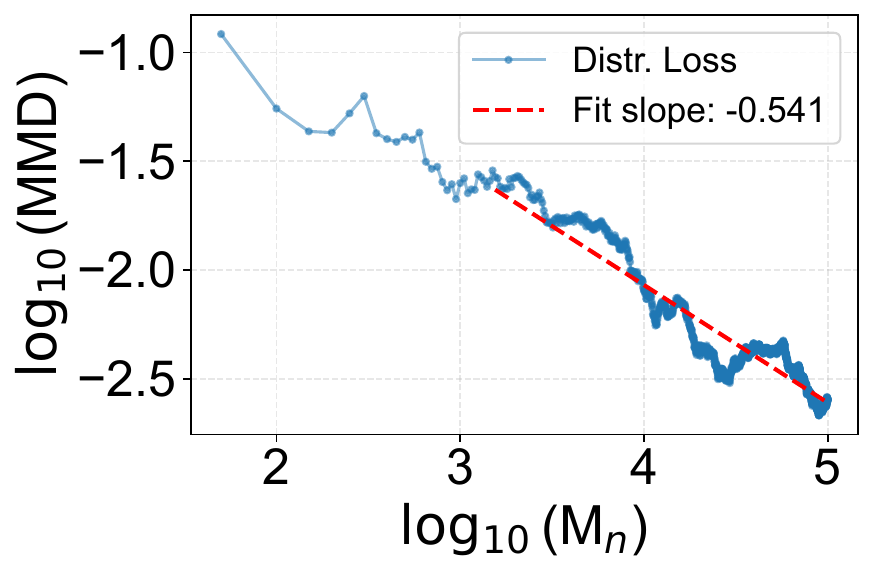}
            \caption*{\small $\alpha=1.00$}
        \end{minipage}
    
        \caption{CRT Simulation (KDE Estimator). $W_1$ and MMD distributional losses and fitted slopes for varying values of $\alpha$ (real-data fraction) for a single run.
        }
        \label{fig:sim-crt-kde-loss-W1}
    \end{figure}

\clearpage

\subsection{Additional CRT Simulation Details: WGAN}
    
    In this simulation, the target distribution $\PP_0$ is the same as in the previous experiment. The generator $\widehat{\PP}_t$ is implemented as a fully connected feedforward network. A latent vector $z \in \text{Unif}(0,1)$ is mapped through three 64-neuron hidden layers, LeakyReLU activations, and a final linear layer producing points in $\RR^1$.
    
    At iteration $t$, a batch of $m_1$ new samples from $\PP_0$ is appended to the dataset, together with $m_2 = ((1-\alpha)/\alpha)m_1$ synthetic samples generated from the previous distribution $\widehat{\PP}_{t-1}$. At each CRT iteration, the generator is completely re-initialized, and is then trained for $k$ epochs using minibatch stochastic gradient descent on the current accumulated dataset of real and synthetic samples. Optimization uses Adam with a fixed learning rate and weight decay. 
    
    The training loss is a differentiable optimal-transport discrepancy. We employ an empirical $W_1$ loss computed via quantile (inverse CDF) matching from the Python OT package~\citep{flamary2021pot}. At each iteration, convergence is assessed by computing an empirical $W_1$ distance between $\widehat{\PP}_t$ and $\PP_0$, using fresh evaluation batches of synthetic samples against a fixed large sample of the target distribution. This produces a sequence  $W_1(\widehat{\PP}_t,\PP_0)$ indexed by $M_t$.
    
    To estimate convergence rates, we fit a power law of the form
    \[
        \log W_1(\widehat{\PP}_t,\PP_0)
            = a + b \log M_t,
    \]
    after discarding an initial burn-in period. When $\alpha = p$, CRT theory predicts a logarithmic phase transition; in this regime, the losses are normalized by $\log(t)$ prior to regression. The fitted slope $b$ is reported as the observed convergence rate of the recursive neural estimator. Examples of final densities and loss curves are shown in \Cref{fig:sim-crt-wgan-progress}, \ref{fig:sim-crt-wgan-loss}. All experimental parameters for model training and sampling are summarized in \Cref{tab:sim12-wgan-params}.
    
            \begin{table}[H]
            \centering
            \small
            \begin{tabular}{lcl}
            \toprule
            \textbf{Parameter} & \textbf{Value} & \textbf{Description} \\
            \midrule
            $m_1+m_2$ & $500$ & Total samples per iteration \\
            $\alpha$ & $\{0.1,0.2,\dots,1.0\}$ & real-data fraction \\
            $T$ & $500$ & Total CRT (outer training loop) iterations \\
            $k$ & $25$ & Training epochs per CRT iteration (inner training loop) \\
            Latent dimension & $1$ & Dimension of $z$ \\
            Latent distribution & Uniform(0,1) & Input to generator \\
            Network width & $64$ & Hidden layer width \\
            Network layers & $3$ & Number of hidden layers \\
            Activation & LeakyReLU$(0.02)$ & Nonlinearity \\
            Optimizer & Adam & Generator optimization \\
            Learning rate & $2\times 10^{-4}$ &  Step size \\
            Weight decay & $1\times 10^{-3}$ &  Optimizer weight decay \\
            Batch size & $1024$ & Training minibatch size \\
            Loss type & Quantile $W_1$ & $W_1$ loss using POT package \\
            Evaluation Samples & 200000 & Number of samples (real, synthetic) for evaluation $W_1$ \\
            \bottomrule
            \end{tabular}
            \caption{CRT Simulation parameters for WGAN estimator.}
            \label{tab:sim12-wgan-params}
            \end{table}

    \begin{figure}[ht!]
        \centering
        \begin{minipage}{0.18\textwidth}
            \centering
            \includegraphics[width=\linewidth]{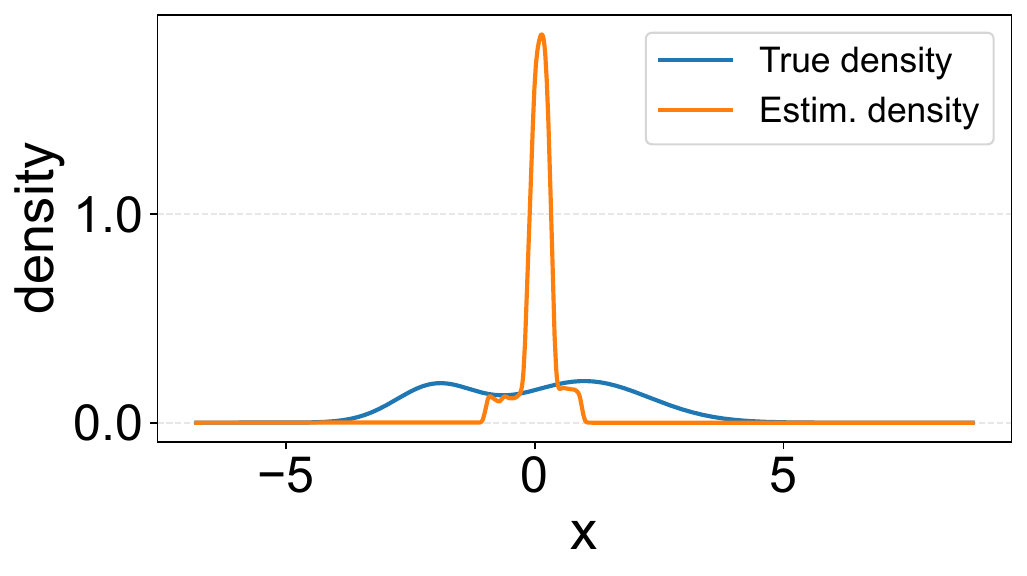}
            \caption*{\small$\alpha=0.1$}
        \end{minipage}
        \hfill
        \begin{minipage}{0.18\textwidth}
            \centering
            \includegraphics[width=\linewidth]{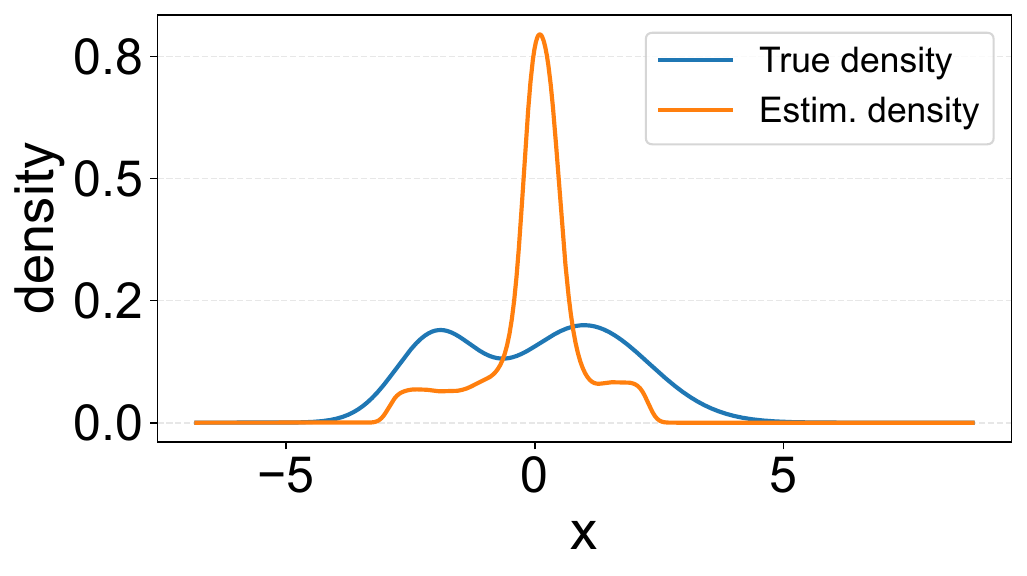}
            \caption*{\small $\alpha=0.2$}
        \end{minipage}
        \hfill
        \begin{minipage}{0.18\textwidth}
            \centering
            \includegraphics[width=\linewidth]{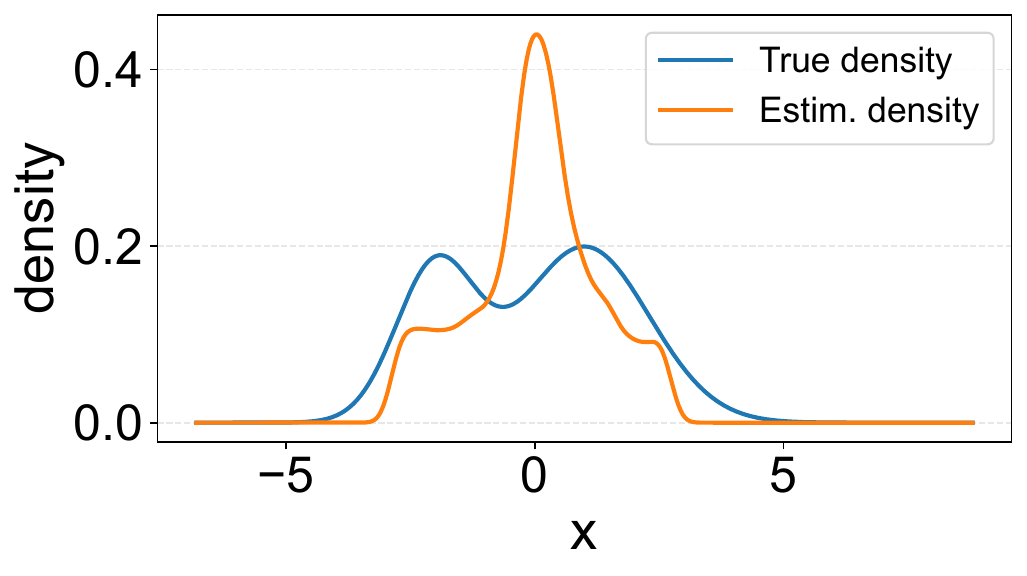}
            \caption*{\small $\alpha=0.3$}
        \end{minipage}
        \hfill
        \begin{minipage}{0.18\textwidth}
            \centering
            \includegraphics[width=\linewidth]{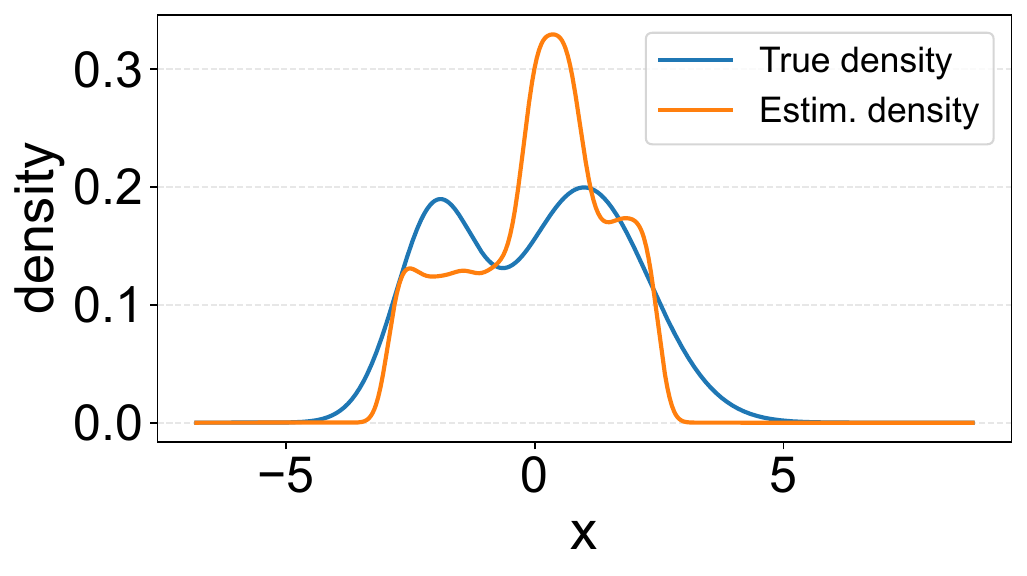}
            \caption*{\small $\alpha=0.4$}
        \end{minipage}
        \hfill
        \begin{minipage}{0.18\textwidth}
            \centering
            \includegraphics[width=\linewidth]{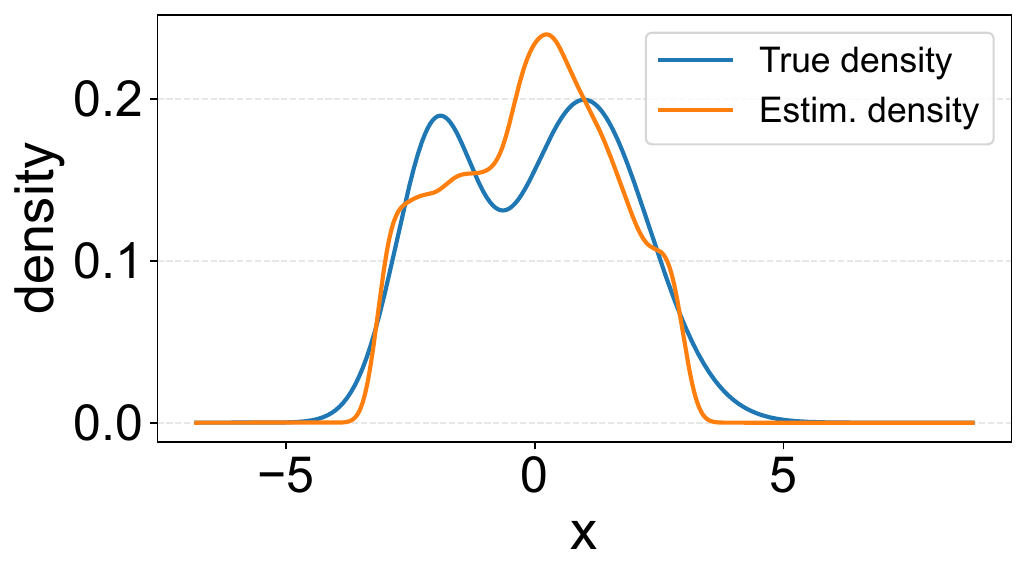}
            \caption*{\small $\alpha=0.5$}
        \end{minipage}

        \vspace{1em}
    
        \begin{minipage}{0.18\textwidth}
            \centering
            \includegraphics[width=\linewidth]{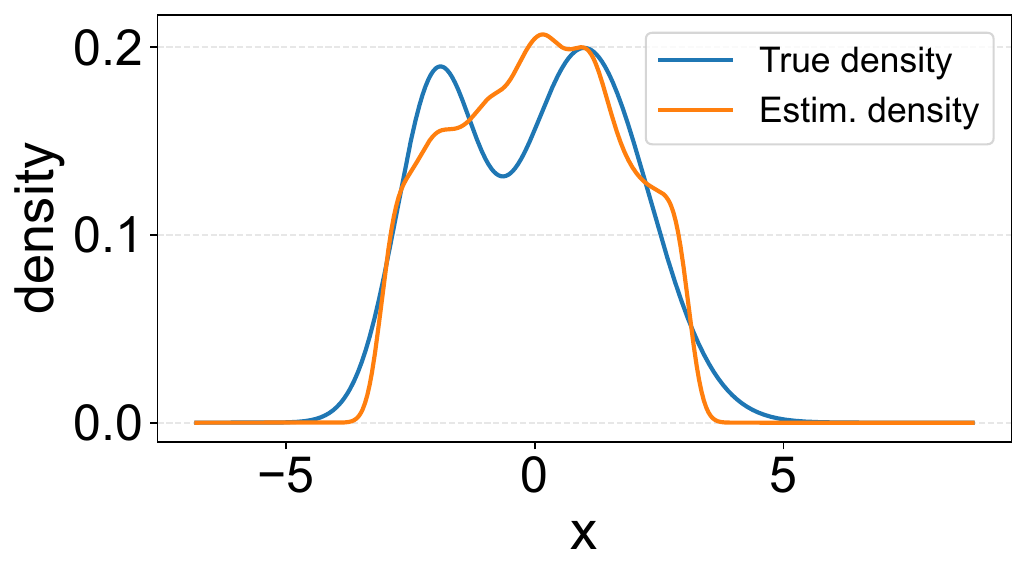}
            \caption*{\small $\alpha=0.6$}
        \end{minipage}
        \hfill
        \begin{minipage}{0.18\textwidth}
            \centering
            \includegraphics[width=\linewidth]{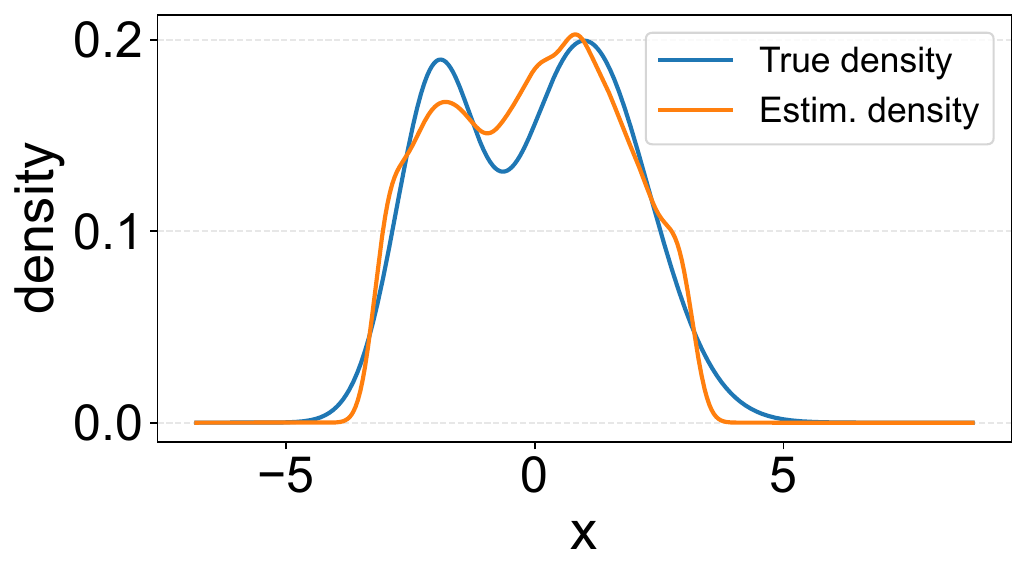}
            \caption*{\small $\alpha=0.7$}
        \end{minipage}
        \hfill
        \begin{minipage}{0.18\textwidth}
            \centering
            \includegraphics[width=\linewidth]{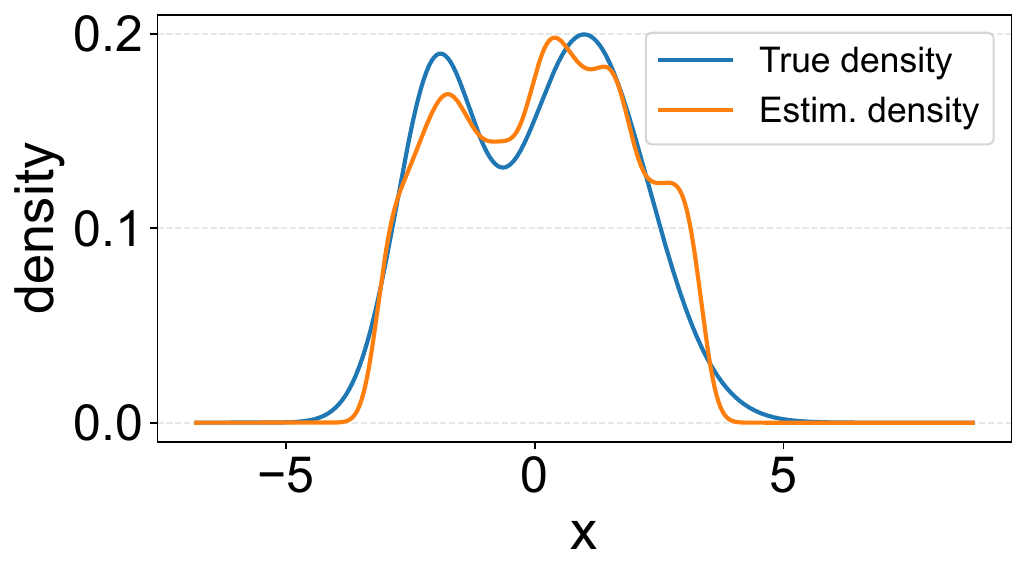}
            \caption*{\small $\alpha=0.8$}
        \end{minipage}
        \hfill
        \begin{minipage}{0.18\textwidth}
            \centering
            \includegraphics[width=\linewidth]{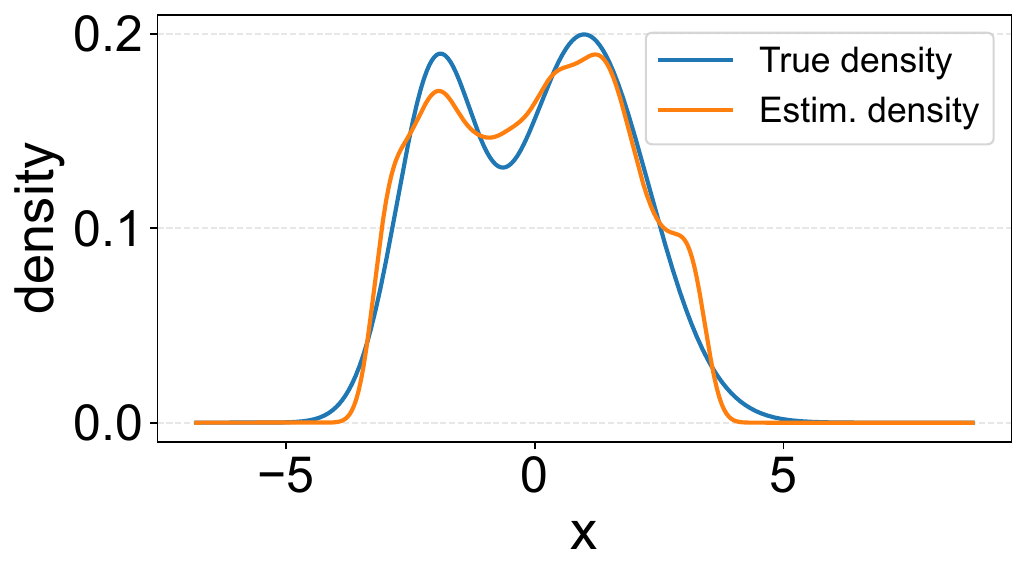}
            \caption*{\small $\alpha=0.9$}
        \end{minipage}
        \hfill
        \begin{minipage}{0.18\textwidth}
            \centering
            \includegraphics[width=\linewidth]{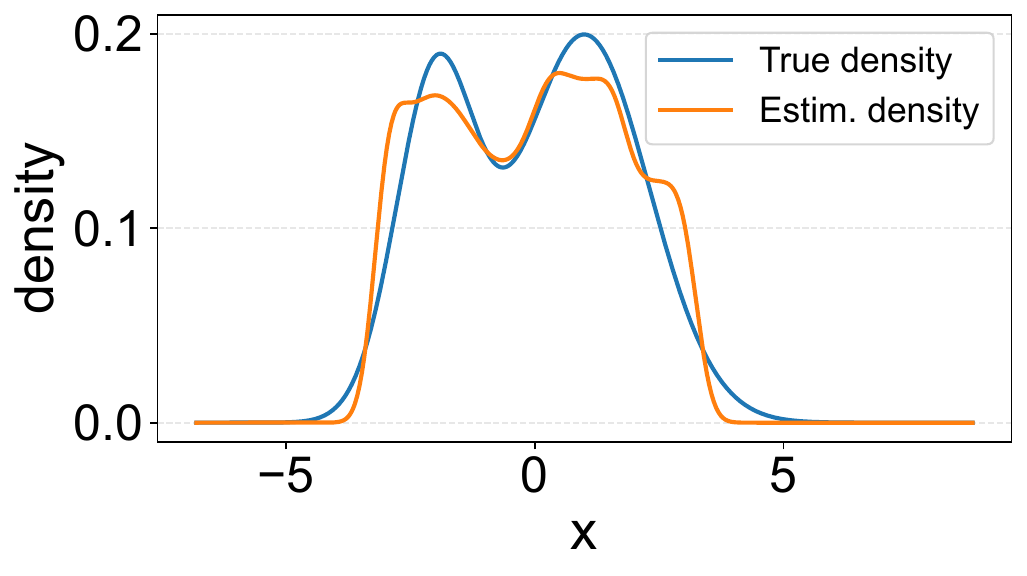}
            \caption*{\small $\alpha=1.0$}
        \end{minipage}
        \caption{CRT Simulation (WGAN Estimator). Final output distributions $\widehat{\PP}_T$ for varying values of $\alpha$ (real-data fraction) visualized using KDE, where all models are run for the same number of CRT iterations.  This fixed compute budget for all models was chosen to minimize the effect of numerical plateaus in the loss when computing optimal transport distances. Under this fixed compute budget, models at greater $\alpha$ are shown to converge more easily, however, for all $\alpha$ we observe the predicted rate of convergence.}
        \label{fig:sim-crt-wgan-progress}
    \end{figure}

    \begin{figure}[ht!]
        \centering
        \begin{minipage}{0.18\textwidth}
            \centering
            \includegraphics[width=\linewidth]{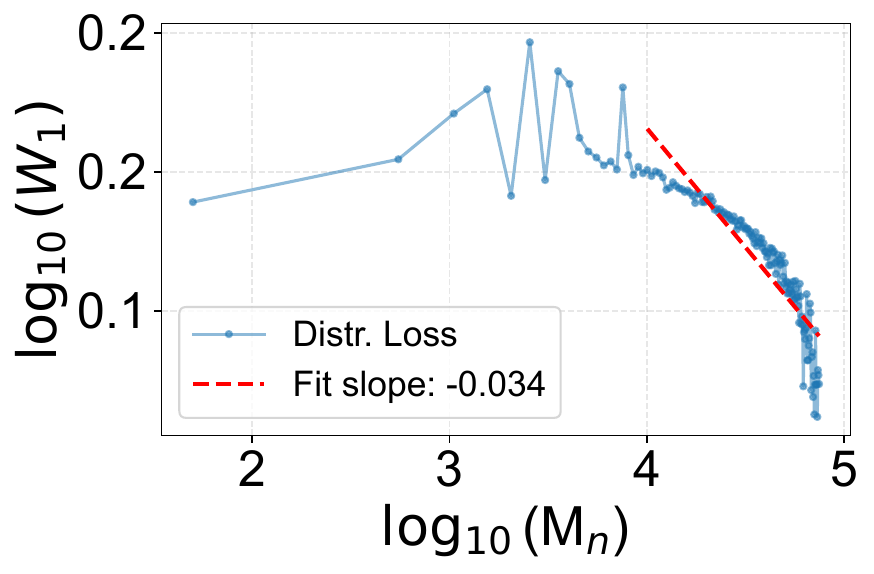}
        \end{minipage}
        \hfill
        \begin{minipage}{0.18\textwidth}
            \centering
            \includegraphics[width=\linewidth]{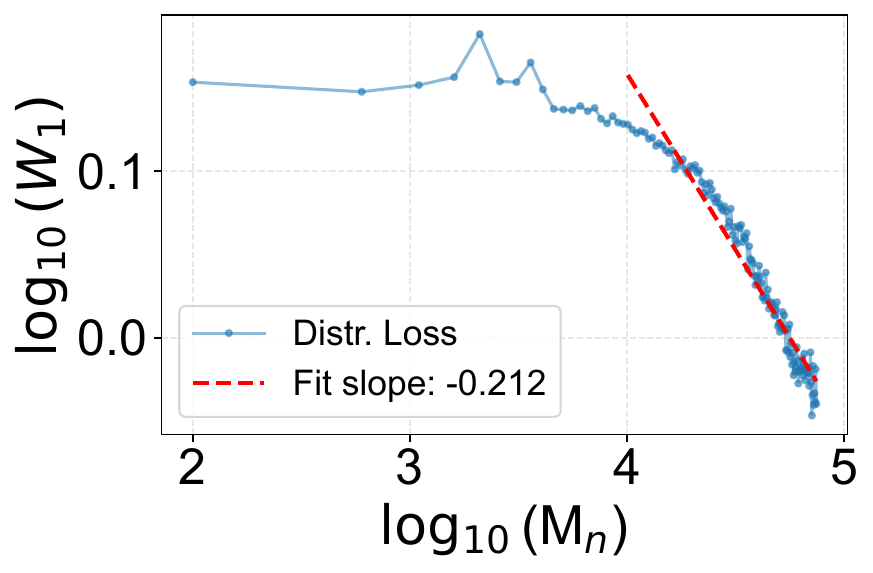}
        \end{minipage}
        \hfill
        \begin{minipage}{0.18\textwidth}
            \centering
            \includegraphics[width=\linewidth]{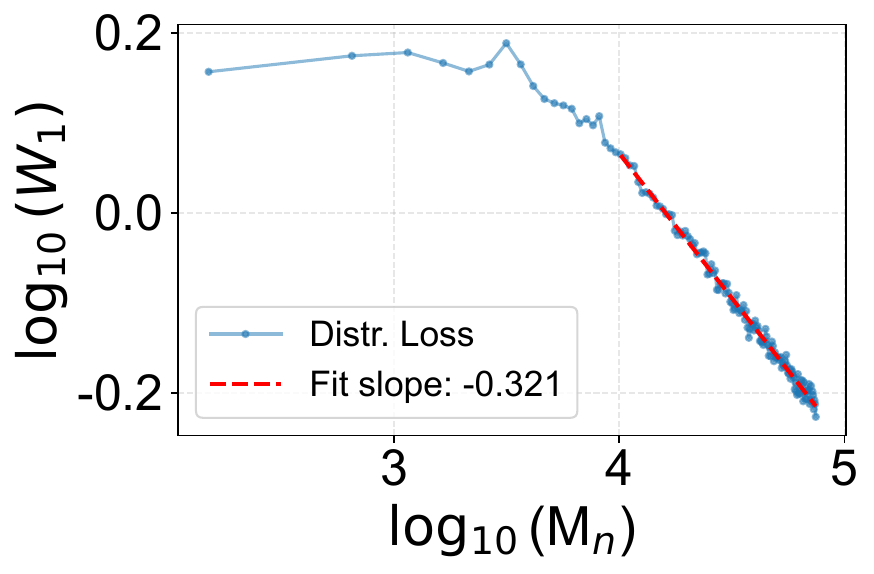}
        \end{minipage}
        \hfill
        \begin{minipage}{0.18\textwidth}
            \centering
            \includegraphics[width=\linewidth]{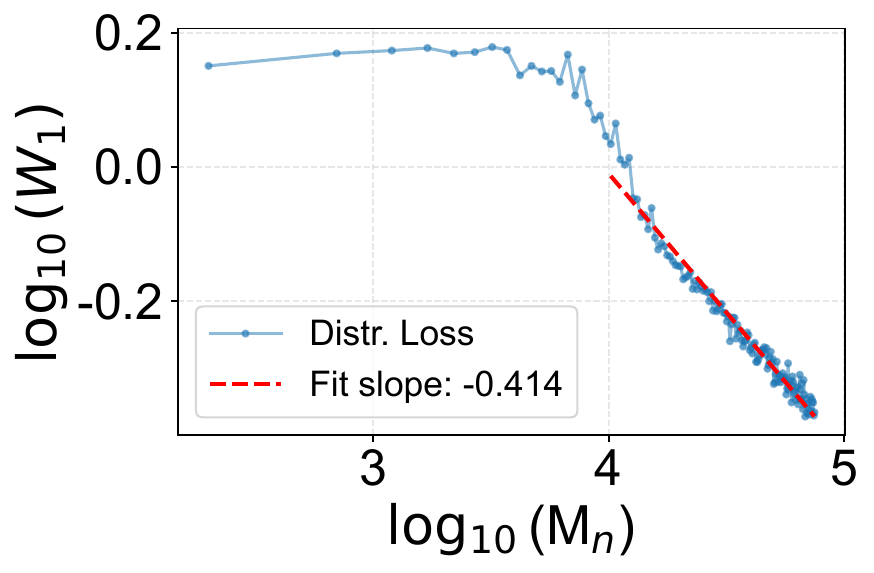}
        \end{minipage}
        \hfill
        \begin{minipage}{0.18\textwidth}
            \centering
            \includegraphics[width=\linewidth]{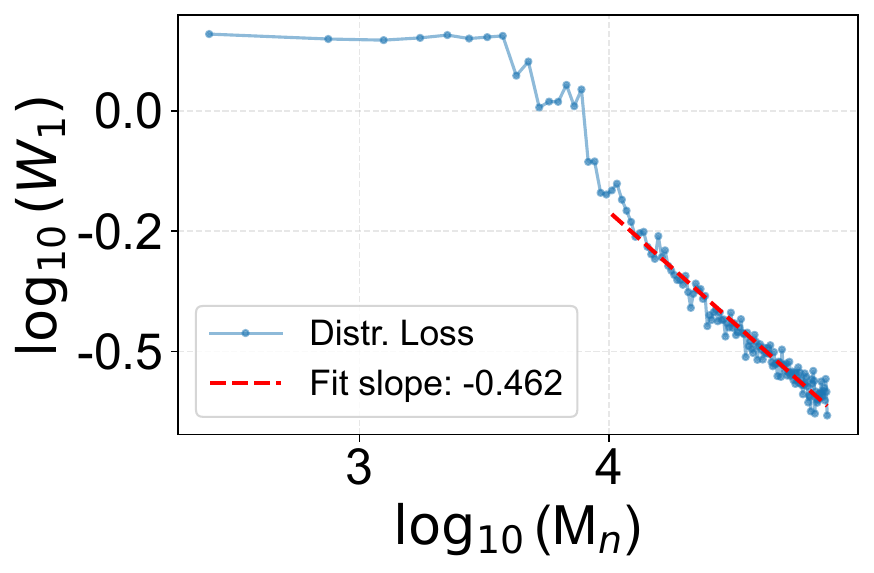}
        \end{minipage}

        \begin{minipage}{0.18\textwidth}
            \centering
            \includegraphics[width=\linewidth]{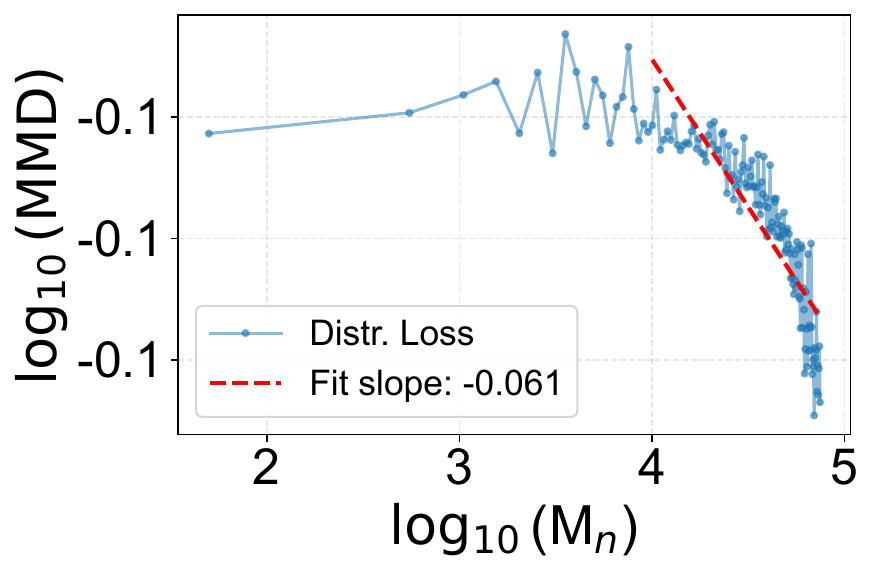}
            \caption*{\small $\alpha=0.10$}
        \end{minipage}
        \hfill
        \begin{minipage}{0.18\textwidth}
            \centering
            \includegraphics[width=\linewidth]{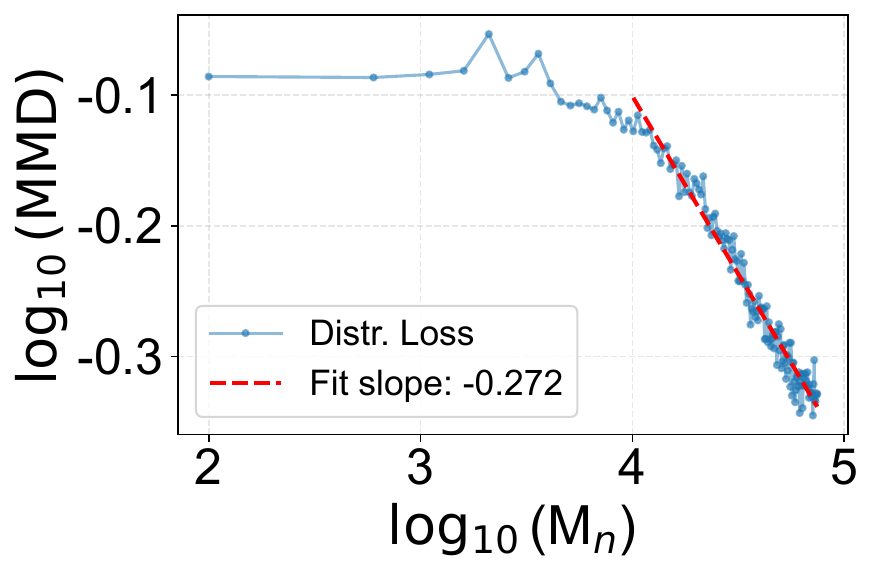}
            \caption*{\small $\alpha=0.20$}
        \end{minipage}
        \hfill
        \begin{minipage}{0.18\textwidth}
            \centering
            \includegraphics[width=\linewidth]{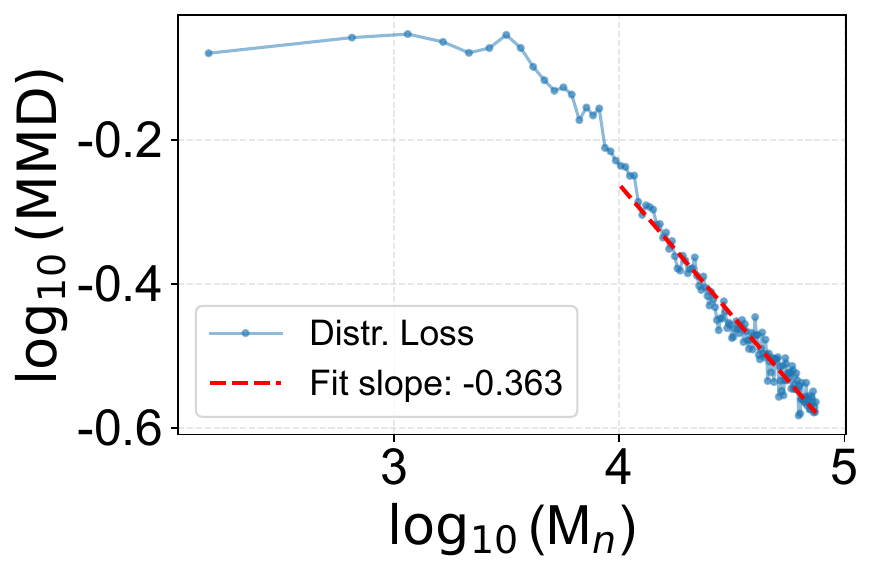}
            \caption*{\small $\alpha=0.30$}
        \end{minipage}
        \hfill
        \begin{minipage}{0.18\textwidth}
            \centering
            \includegraphics[width=\linewidth]{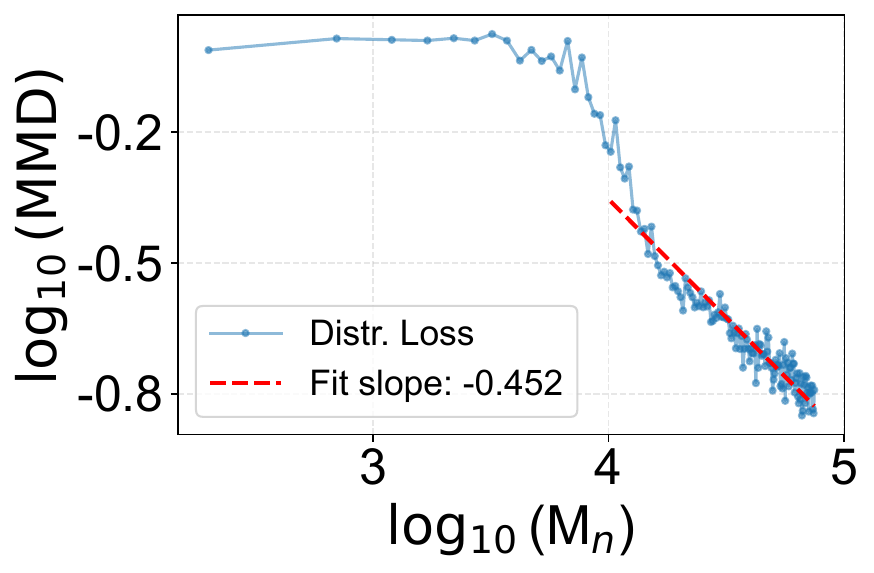}
            \caption*{\small $\alpha=0.40$}
        \end{minipage}
        \hfill
        \begin{minipage}{0.18\textwidth}
            \centering
            \includegraphics[width=\linewidth]{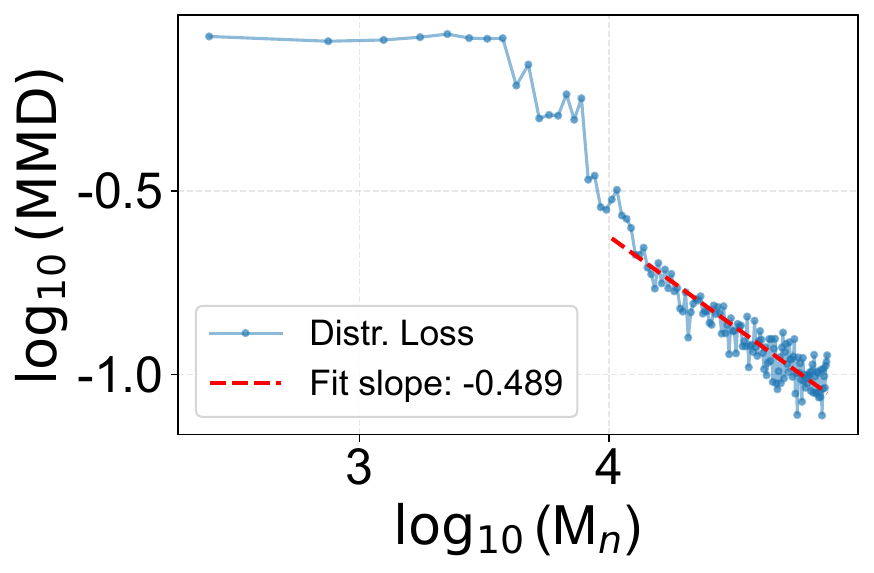}
            \caption*{\small $\alpha=0.50$}
        \end{minipage}
        
        \vspace{1em}
        
        \begin{minipage}{0.18\textwidth}
            \centering
            \includegraphics[width=\linewidth]{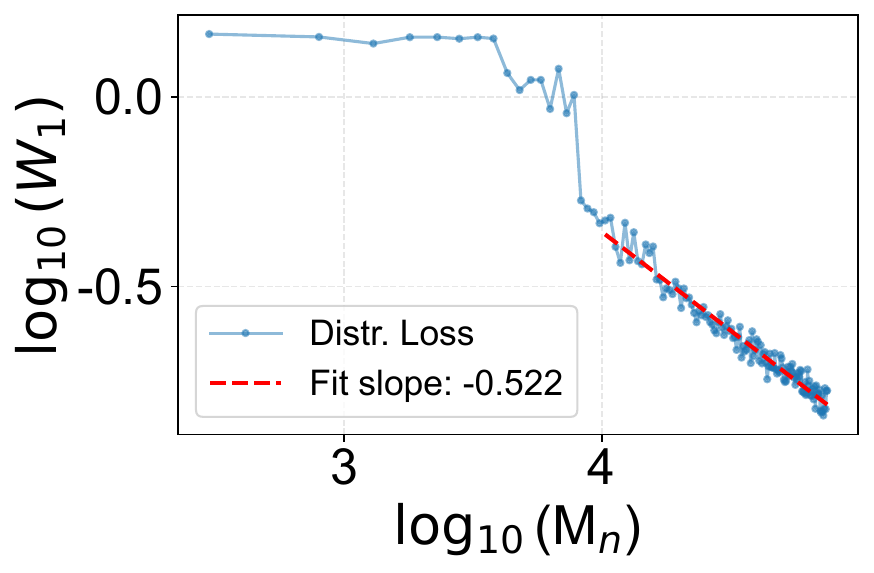}
        \end{minipage}
        \hfill
        \begin{minipage}{0.18\textwidth}
            \centering
            \includegraphics[width=\linewidth]{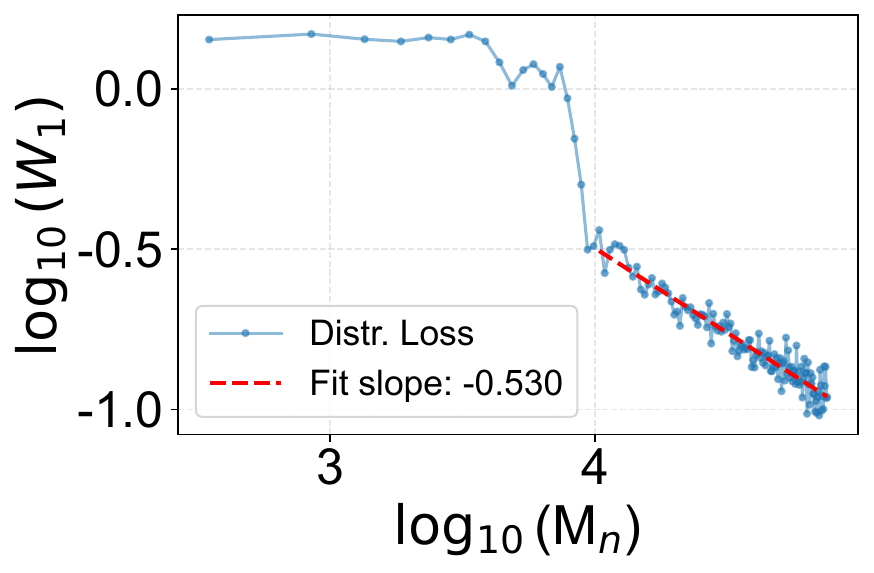}
        \end{minipage}
        \hfill
        \begin{minipage}{0.18\textwidth}
            \centering
            \includegraphics[width=\linewidth]{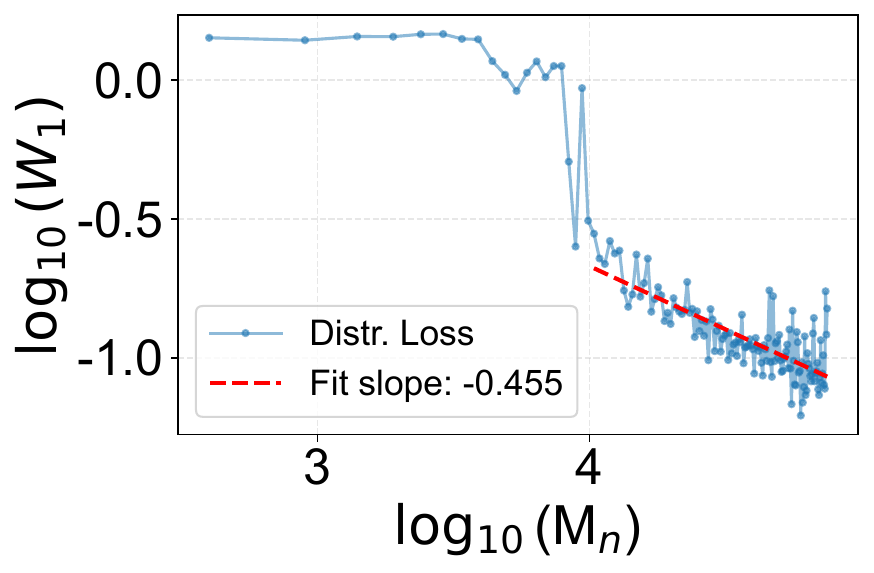}
        \end{minipage}
        \hfill
        \begin{minipage}{0.18\textwidth}
            \centering
            \includegraphics[width=\linewidth]{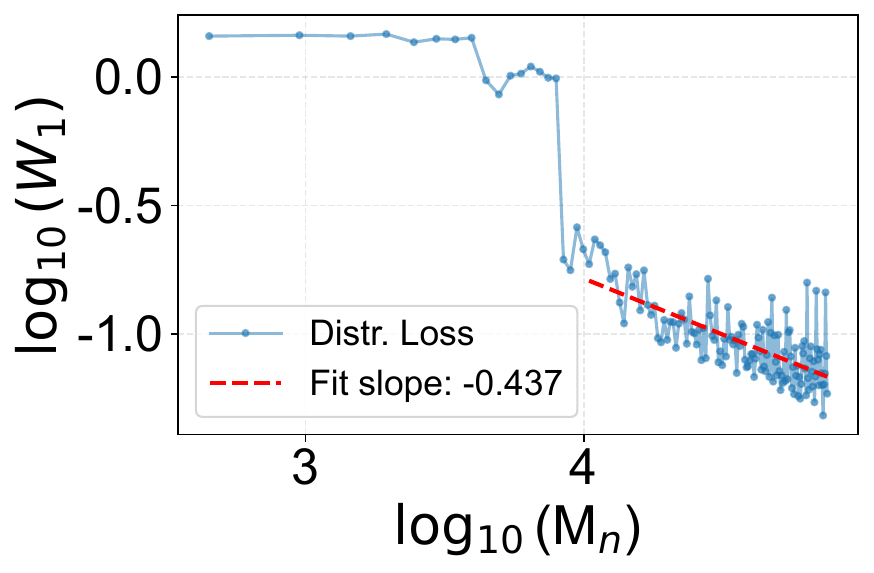}
        \end{minipage}
        \hfill
        \begin{minipage}{0.18\textwidth}
            \centering
            \includegraphics[width=\linewidth]{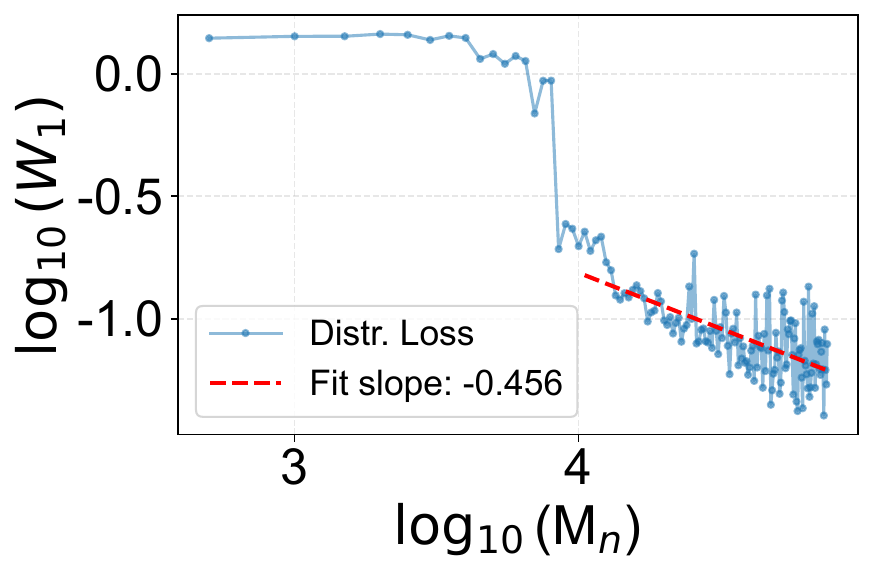}
        \end{minipage}

        \begin{minipage}{0.18\textwidth}
            \centering
            \includegraphics[width=\linewidth]{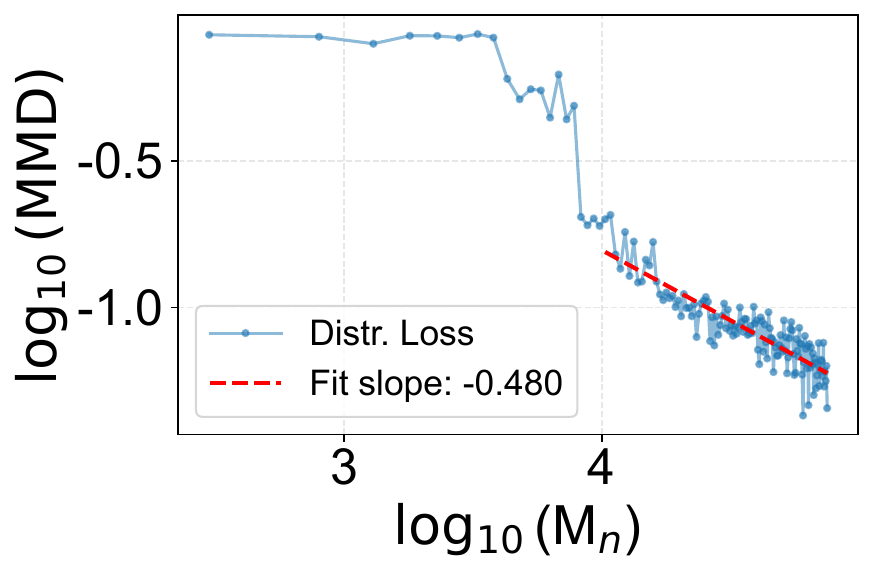}
            \caption*{\small $\alpha=0.60$}
        \end{minipage}
        \hfill
        \begin{minipage}{0.18\textwidth}
            \centering
            \includegraphics[width=\linewidth]{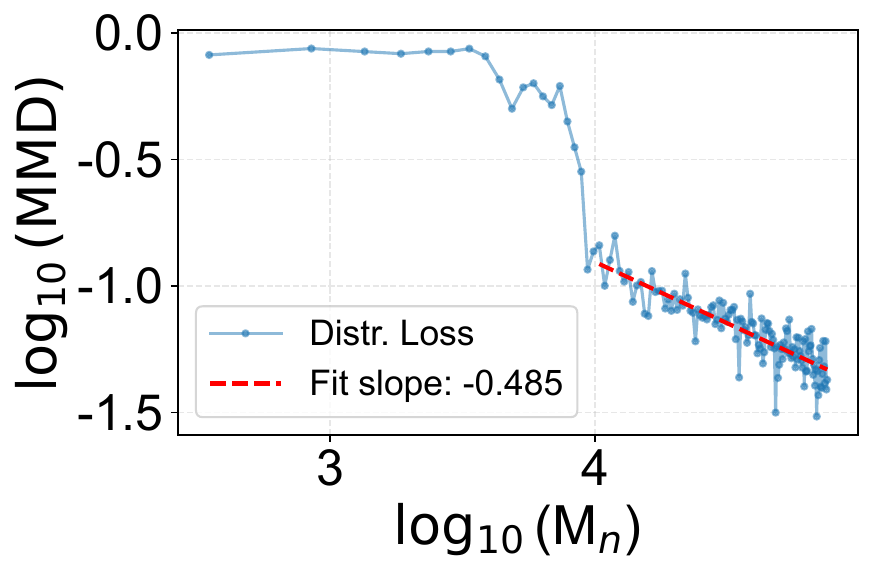}
            \caption*{\small $\alpha=0.70$}
        \end{minipage}
        \hfill
        \begin{minipage}{0.18\textwidth}
            \centering
            \includegraphics[width=\linewidth]{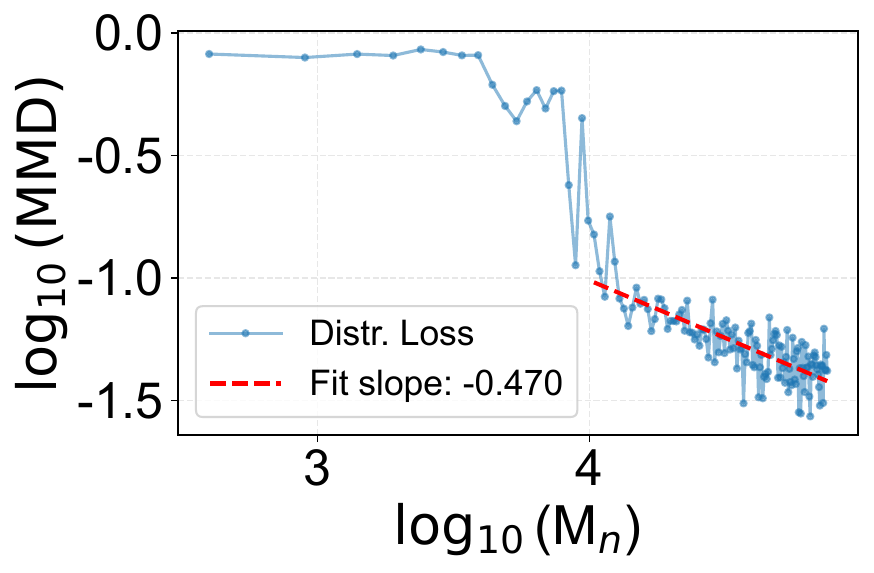}
            \caption*{\small $\alpha=0.80$}
        \end{minipage}
        \hfill
        \begin{minipage}{0.18\textwidth}
            \centering
            \includegraphics[width=\linewidth]{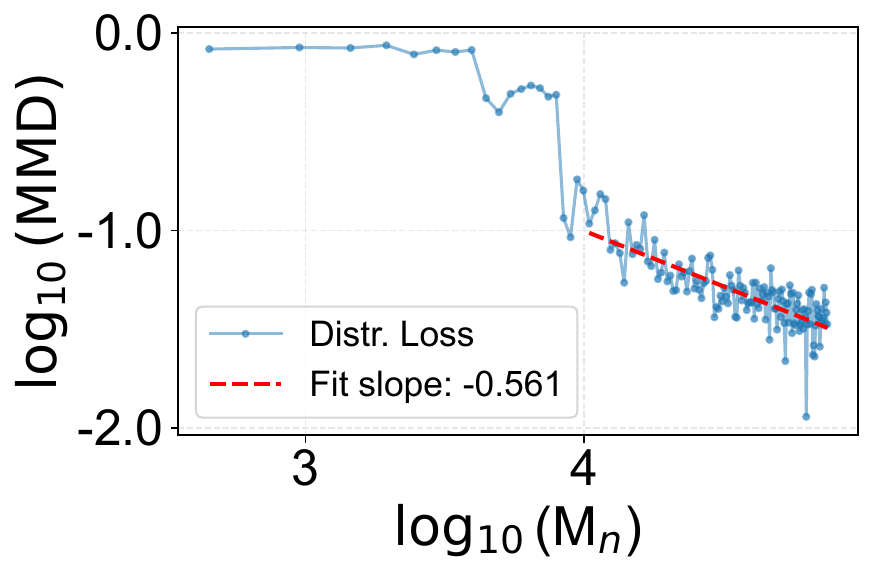}
            \caption*{\small $\alpha=0.90$}
        \end{minipage}
        \hfill
        \begin{minipage}{0.18\textwidth}
            \centering
            \includegraphics[width=\linewidth]{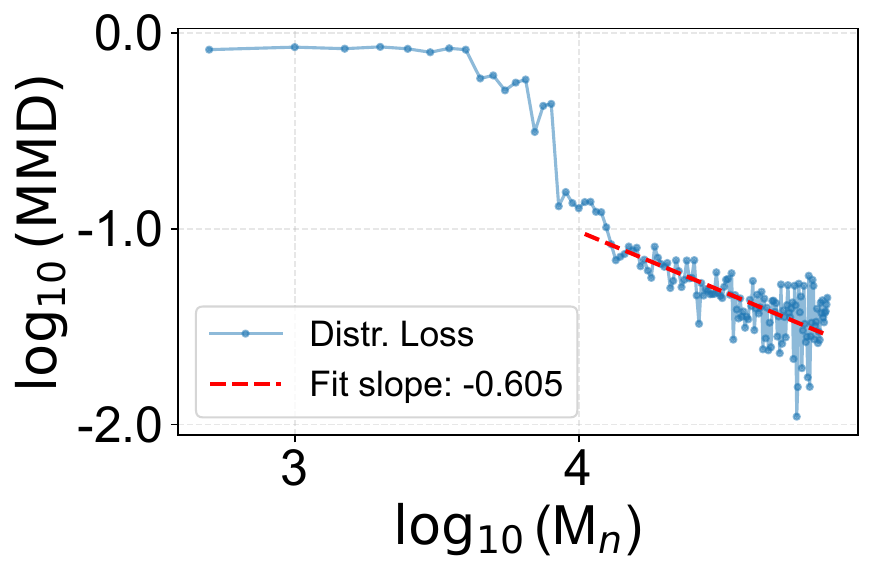}
            \caption*{\small $\alpha=1.00$}
        \end{minipage}
    
        \caption{CRT Simulation (WGAN Estimator). $W_1$ and MMD distributional losses and fitted slopes for varying values of $\alpha$ (real-data fraction) for a single run.
        }
        \label{fig:sim-crt-wgan-loss}
    \end{figure}

\clearpage
\subsection{Additional BCRT Simulation Details}\label{app:kdebias}

In the BCRT setting, we use the same one-dimensional Gaussian mixture target distribution $\PP_0$ as in the CRT simulations, with identical mixture weights and component parameters. The numerical grid used to evaluate distributional discrepancies, as well as the closed-form density and CDF of $\PP_0$ on this grid, are constructed exactly as described in the previous subsection and reused throughout.

We follow the BCRT procedure of \textbf{Biased Contaminated Recursive Training}. At iteration $t$, a batch of $m_1$ real samples is added to the accumulated dataset, together with $m_2 = ((1-\alpha_{\mathrm{real}})/\alpha_{\mathrm{real}})m_1$ synthetic samples drawn from the previous distribution $\widehat{\PP}_{t-1}$. Synthetic samples are generated by resampling from the accumulated dataset and adding Gaussian noise with variance equal to the current KDE bandwidth. Unlike in the CRT simulations, the real data stream here is biased. At iteration $t$, real samples are drawn from a contaminated distribution
\[
\PP_t^{\textnormal{bias}}
= (1-\textnormal{bias}_t)\PP_0
+ \textnormal{bias}_t\,\mathcal{N}(\mu_3,\sigma_3),
\]
where $(\mu_3,\sigma_3)=(3.0,1.0)$. The contamination level decays polynomially as
\[
\textnormal{bias}_t = 0.2\,(t+5)^{-q},
\]
with decay rate parameter $q > 0$.

The estimator $\widehat{\PP}_t$ is a KDE, as in the CRT case. The bandwidth follows a deterministic decay schedule
\[
h_t = h_0\,t^{-p},
\]
with base bandwidth $h_0=2$ and smoothness parameter $p=0.5$. The estimator defines a density-CDF pair $(\widehat f_t,\widehat F_t)$ evaluated on the fixed grid, using trapezoidal integration and renormalization as in previous simulations.

Performance is evaluated using the same deterministic plug-in metrics as in Simulations~1 and~2: the $W_1$ distance,
\[
W_1(\widehat{\PP}_t,\PP_0)
\approx \int \bigl| F_0(x) - \widehat F_t(x) \bigr|\,dx,
\]
and squared MMD, with the MMD computed using a Gaussian kernel on the evaluation grid.

To estimate convergence rates, we record the sequence of losses $\{d(\widehat{\PP}_t,\PP_0)\}$ over iterations and fit a power law of the form
\[
\log d(\widehat{\PP}_t,\PP_0)
= a + b \log M_t,
\]
where $M_t$ denotes the total number of accumulated samples. After discarding an initial burn-in period, the fitted slope $b$ yields the estimated empirical convergence rate. The theory predicts that the effective rate is governed by $\min(p,\alpha,q)$; empirical estimates are compared against this benchmark in the main text. Simulation parameters specific to this section are summarized below. Additionally, examples of final fitted models are visualized in \Cref{fig:ECDF-three-by-three-alpha-q} and \Cref{fig:KDE-three-by-three-alpha-q}, and examples of training curves in \Cref{fig:ECDF-three-by-three-alpha-q-loss} and \Cref{fig:KDE-three-by-three-alpha-q-loss}. All experimental parameters for model training and sampling are summarized in \Cref{tab:sim4-ecdf-params} and \Cref{tab:sim4-kde-params}.

    \begin{table}[H]
        \centering
        \footnotesize
        \begin{tabular}{lcl}
        \toprule
        \textbf{Parameter} & \textbf{Value} & \textbf{Description} \\
        \midrule
        $m_1$ & $25$ & Real samples per iteration \\
        $\alpha$ & $\{0.25,0.5,0.75\}$ & real-data fraction \\
        $q$ & $\{0.25,0.5,0.75\}$ & Bias Decay Rate \\
        $T$ & $3000$ & Total BCRT iterations \\
        $n_{\mathrm{reps}}$ & $100$ & Number of repetitions \\
        $m_{\mathrm{grid}}$ & $200$ & Grid size for deterministic evaluation \\
        $[x_{\min},x_{\max}]$ & Mixture-based & Grid interval for density/CDF evaluation \\
        $w_1$ & $0.35$ & Mixture weight \\
        $\mu_1,\sigma_1$ & $-2.0,\; 0.8$ & First Gaussian component \\
        $\mu_2,\sigma_2$ & $1.0,\; 1.3$ & Second Gaussian component \\
        $\mu_3,\sigma_3$ & $3.0,\; 1.0$ & Bias Gaussian component \\
        $h_0$ & $2.0$ & Base KDE bandwidth \\
        \bottomrule
        \end{tabular}
         \caption{BCRT Simulation (ECDF Estimator) experimental parameters.}
        \label{tab:sim4-ecdf-params}
        \end{table}
    
\begin{figure}[ht!]
    \centering

    \begin{minipage}{0.30\textwidth}
        \centering
        \includegraphics[width=\linewidth]{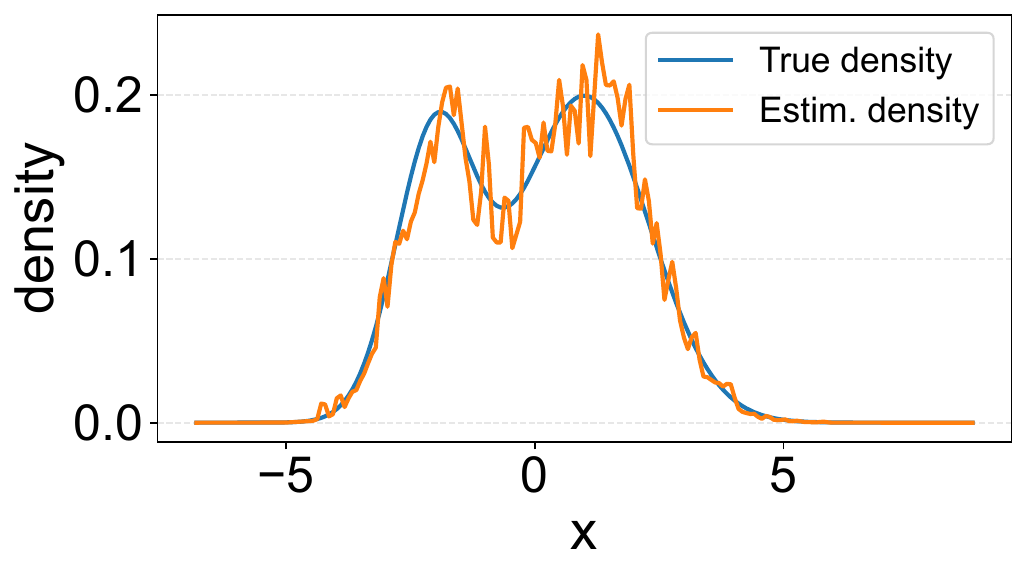}
        \\[-0.5em]
        {\small $\alpha=0.25,\ q=0.25$}
    \end{minipage}
    \hfill
    \begin{minipage}{0.30\textwidth}
        \centering
        \includegraphics[width=\linewidth]{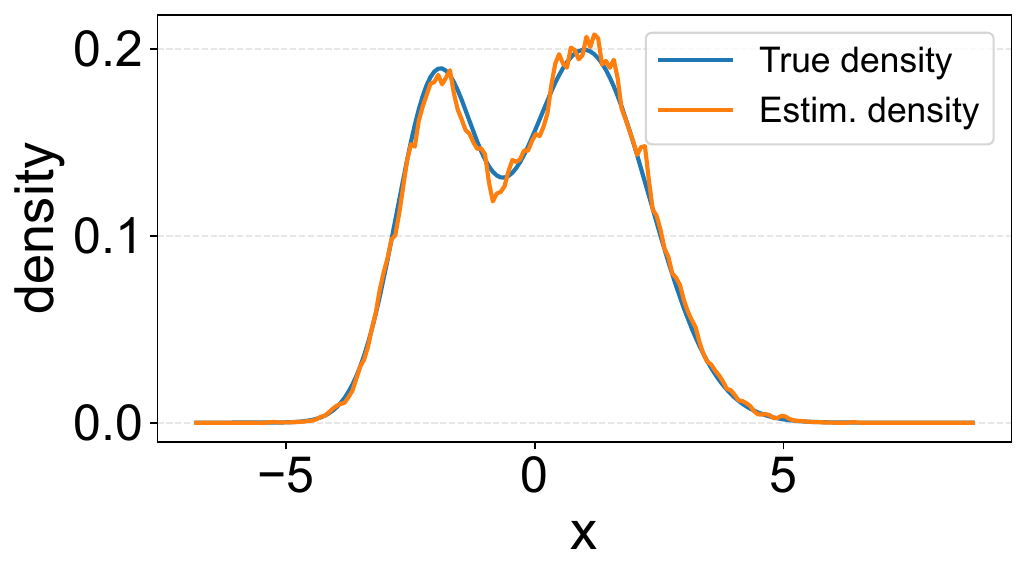}
        \\[-0.5em]
        {\small $\alpha=0.50,\ q=0.25$}
    \end{minipage}
    \hfill
    \begin{minipage}{0.30\textwidth}
        \centering
        \includegraphics[width=\linewidth]{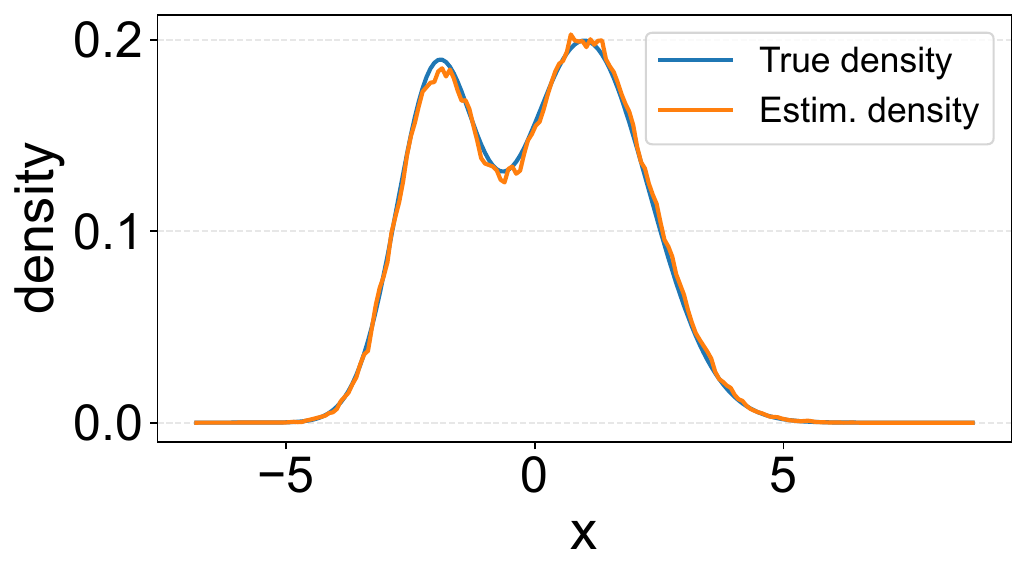}
        \\[-0.5em]
        {\small $\alpha=0.75,\ q=0.25$}
    \end{minipage}

    \vspace{1em}

    \begin{minipage}{0.30\textwidth}
        \centering
        \includegraphics[width=\linewidth]{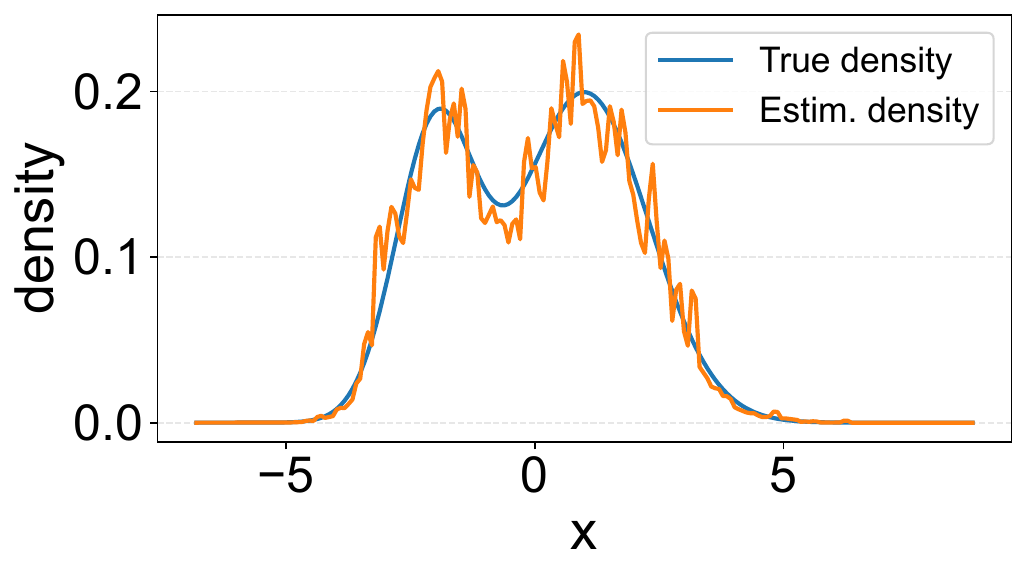}
        \\[-0.5em]
        {\small $\alpha=0.25,\ q=0.50$}
    \end{minipage}
    \hfill
    \begin{minipage}{0.30\textwidth}
        \centering
        \includegraphics[width=\linewidth]{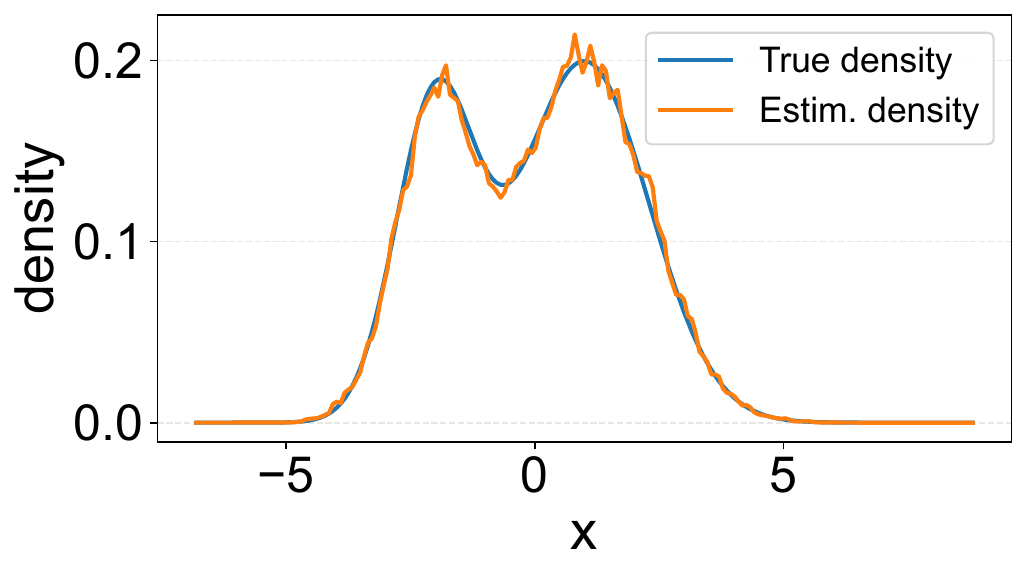}
        \\[-0.5em]
        {\small $\alpha=0.50,\ q=0.50$}
    \end{minipage}
    \hfill
    \begin{minipage}{0.30\textwidth}
        \centering
        \includegraphics[width=\linewidth]{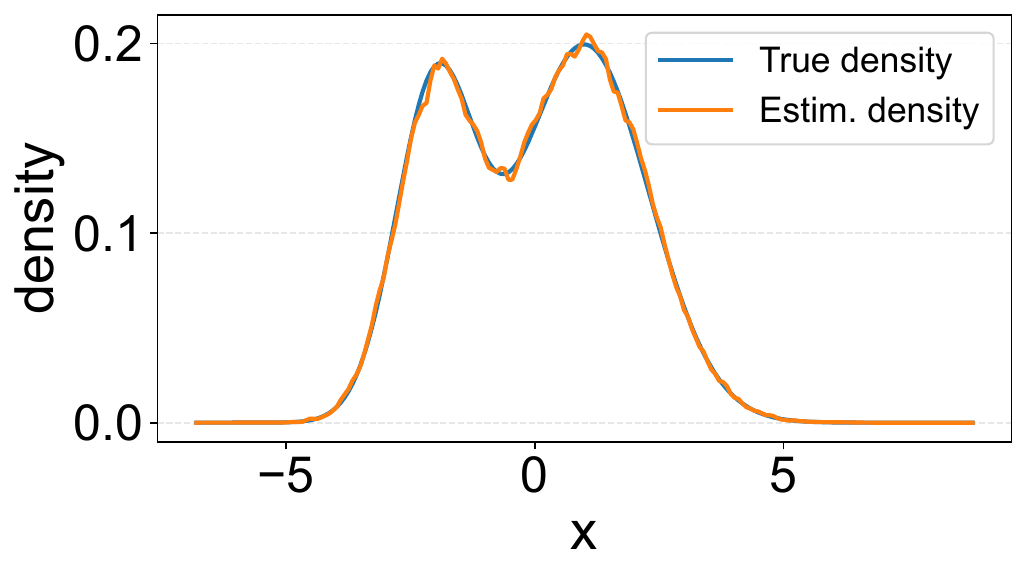}
        \\[-0.5em]
        {\small $\alpha=0.75,\ q=0.50$}
    \end{minipage}

    \vspace{1em}

    \begin{minipage}{0.30\textwidth}
        \centering
        \includegraphics[width=\linewidth]{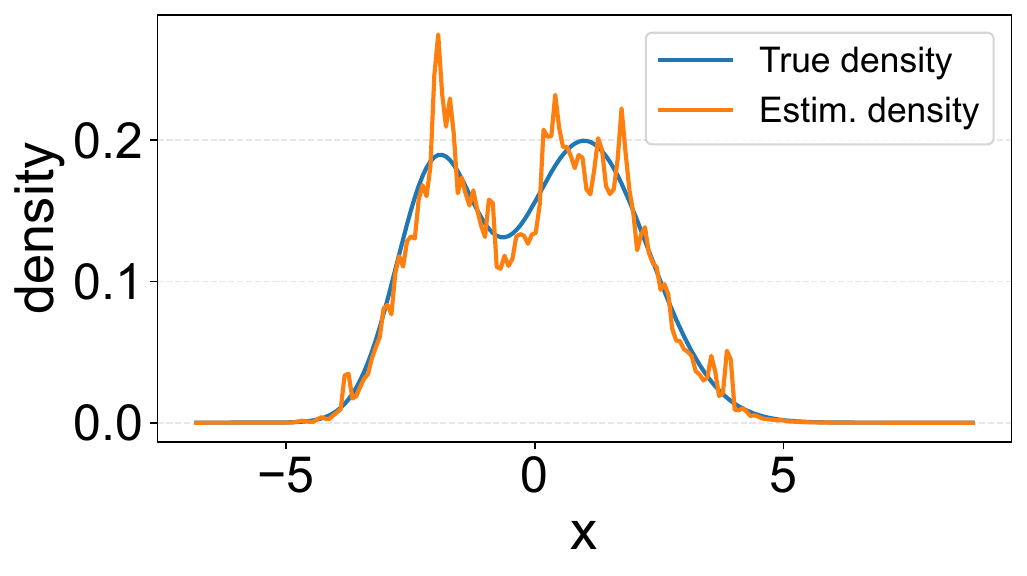}
        \\[-0.5em]
        {\small $\alpha=0.25,\ q=0.75$}
    \end{minipage}
    \hfill
    \begin{minipage}{0.30\textwidth}
        \centering
        \includegraphics[width=\linewidth]{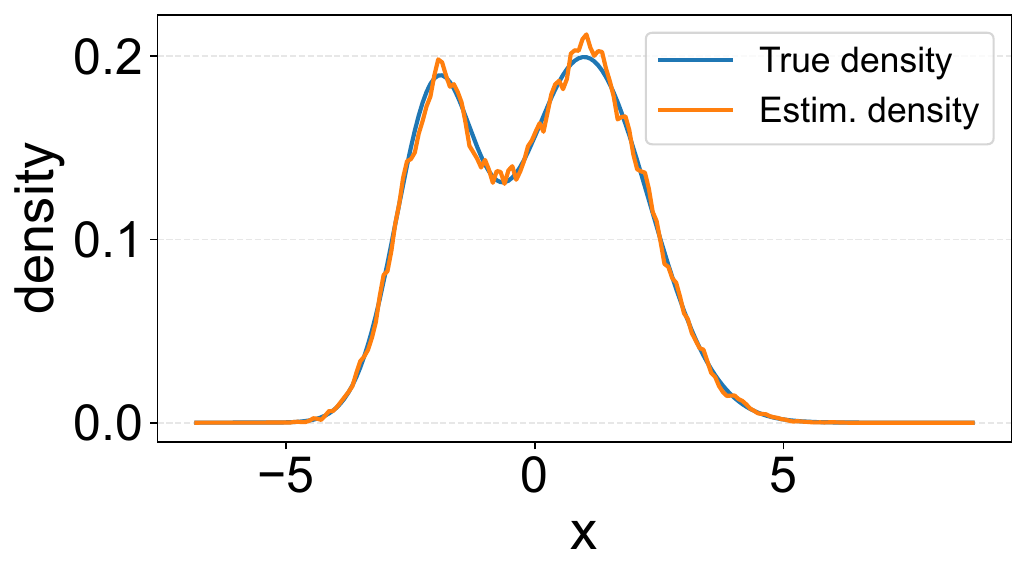}
        \\[-0.5em]
        {\small $\alpha=0.50,\ q=0.75$}
    \end{minipage}
    \hfill
    \begin{minipage}{0.30\textwidth}
        \centering
        \includegraphics[width=\linewidth]{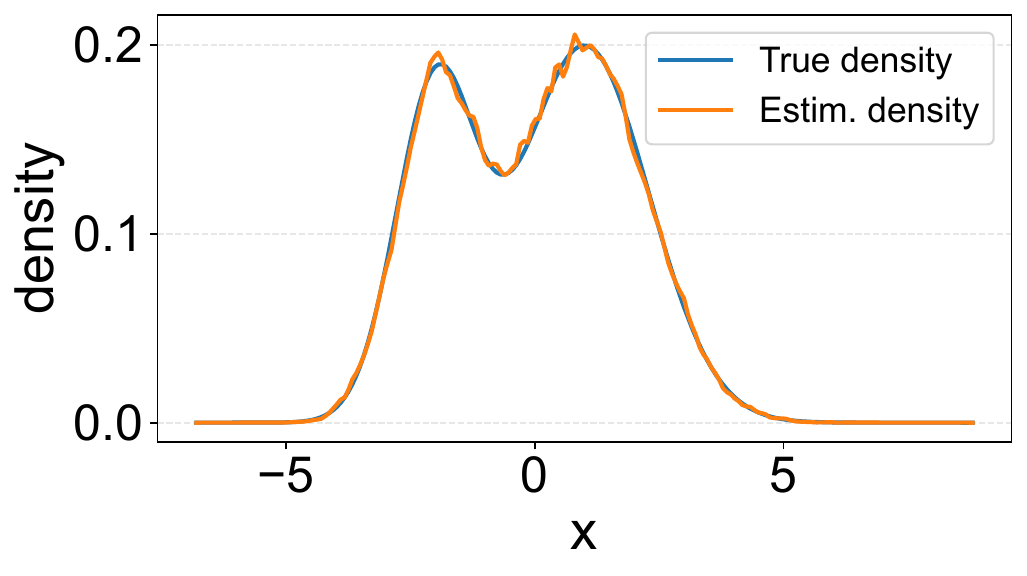}
        \\[-0.5em]
        {\small $\alpha=0.75,\ q=0.75$}
    \end{minipage}

    \caption{BCRT Simulation (ECDF Estimator). Final output distributions across combinations of real-data fraction $\alpha \in \{0.25,0.5,0.75\}$ and bias decay rate $q \in \{0.25,0.5,0.75\}$. }
    \label{fig:ECDF-three-by-three-alpha-q}
\end{figure}

\begin{figure}[ht!]
    \centering

    \begin{minipage}{0.275\textwidth}
        \centering
        \includegraphics[width=\linewidth]{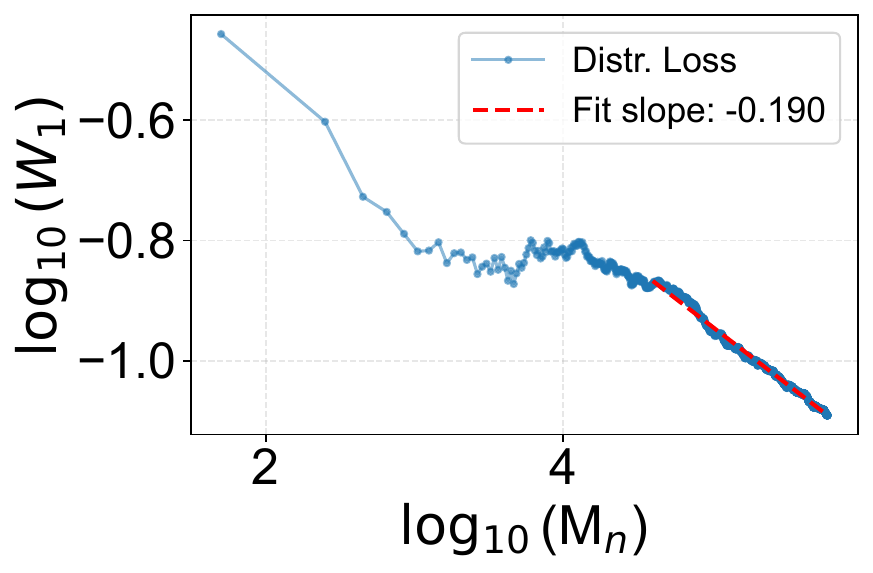}
    \end{minipage}
    \hfill
    \begin{minipage}{0.275\textwidth}
        \centering
        \includegraphics[width=\linewidth]{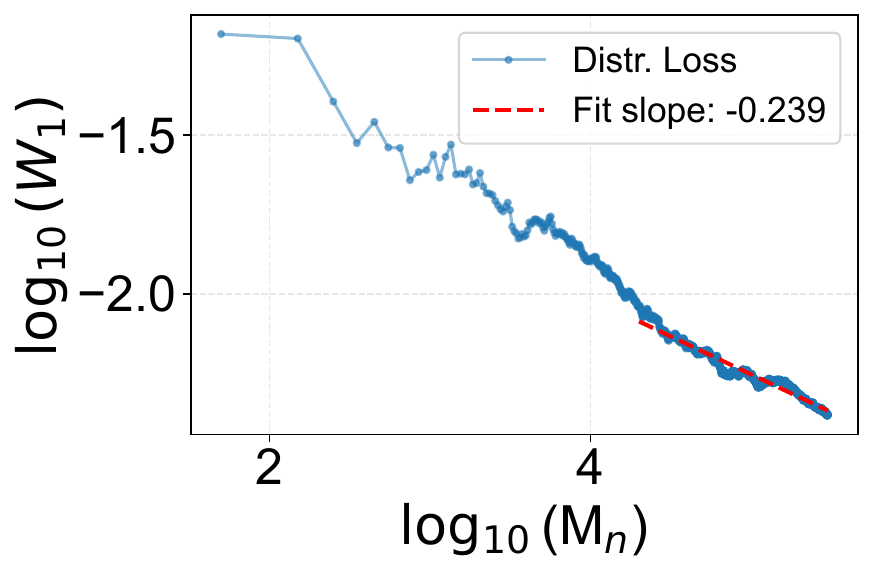}
        \end{minipage}
    \hfill
    \begin{minipage}{0.275\textwidth}
        \centering
        \includegraphics[width=\linewidth]{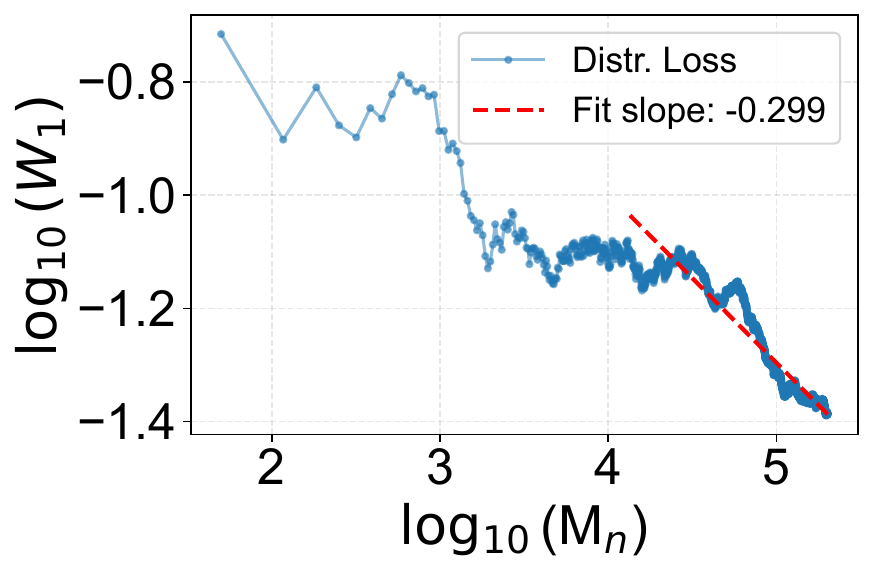}
        \end{minipage}
     \begin{minipage}{0.275\textwidth}
        \centering
        \includegraphics[width=\linewidth]{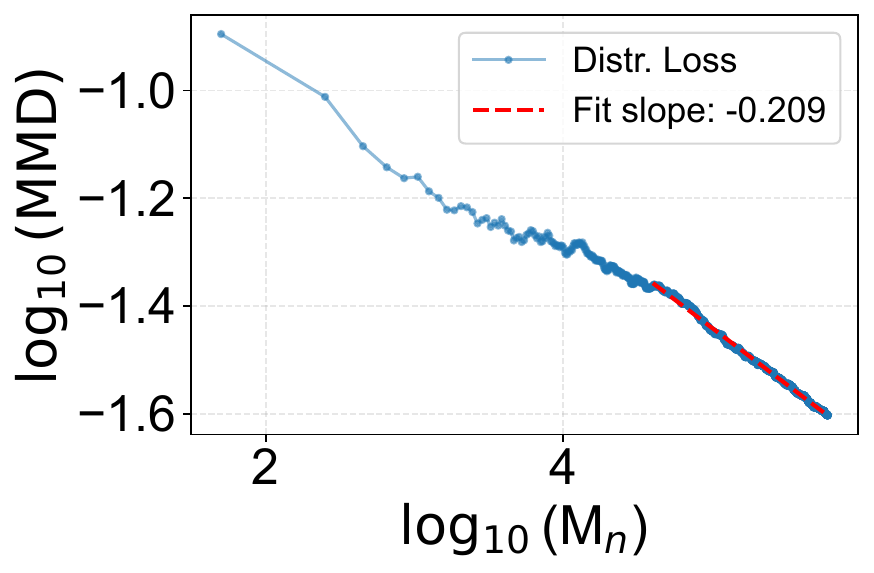}
        \caption*{\small $\alpha=0.25,\ q=0.25$}
    \end{minipage}
    \hfill
    \begin{minipage}{0.275\textwidth}
        \centering
        \includegraphics[width=\linewidth]{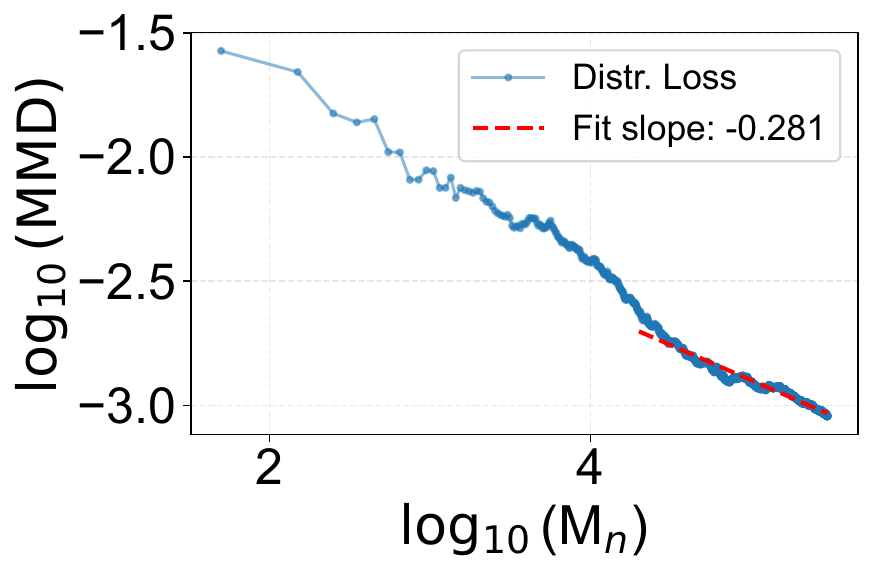}
        \caption*{\small $\alpha=0.50,\ q=0.25$}
        \end{minipage}
    \hfill
    \begin{minipage}{0.275\textwidth}
        \centering
        \includegraphics[width=\linewidth]{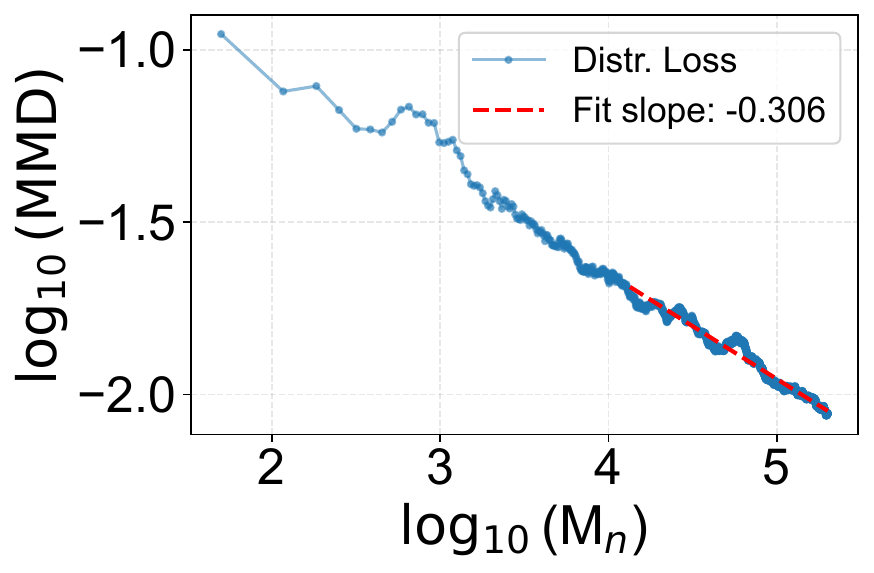}
        \caption*{\small $\alpha=0.75,\ q=0.25$}
        \end{minipage}
        
    \vspace{1em}

    \begin{minipage}{0.275\textwidth}
        \centering
        \includegraphics[width=\linewidth]{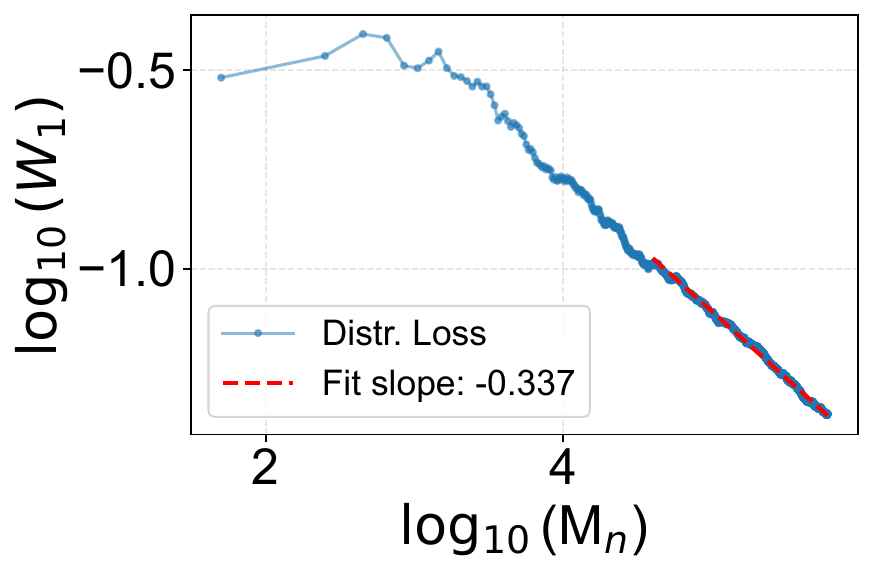}
    \end{minipage}
    \hfill
    \begin{minipage}{0.275\textwidth}
        \centering
        \includegraphics[width=\linewidth]{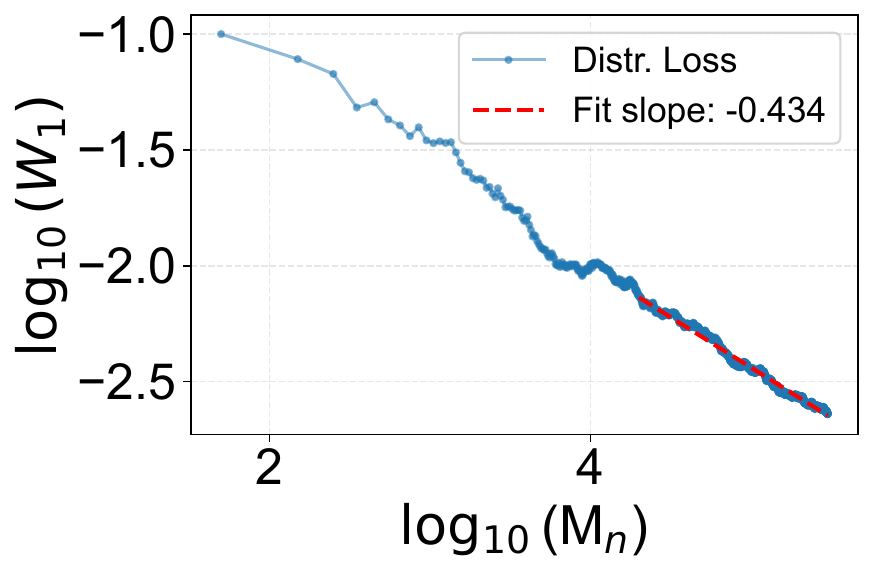}
    \end{minipage}
    \hfill
    \begin{minipage}{0.275\textwidth}
        \centering
        \includegraphics[width=\linewidth]{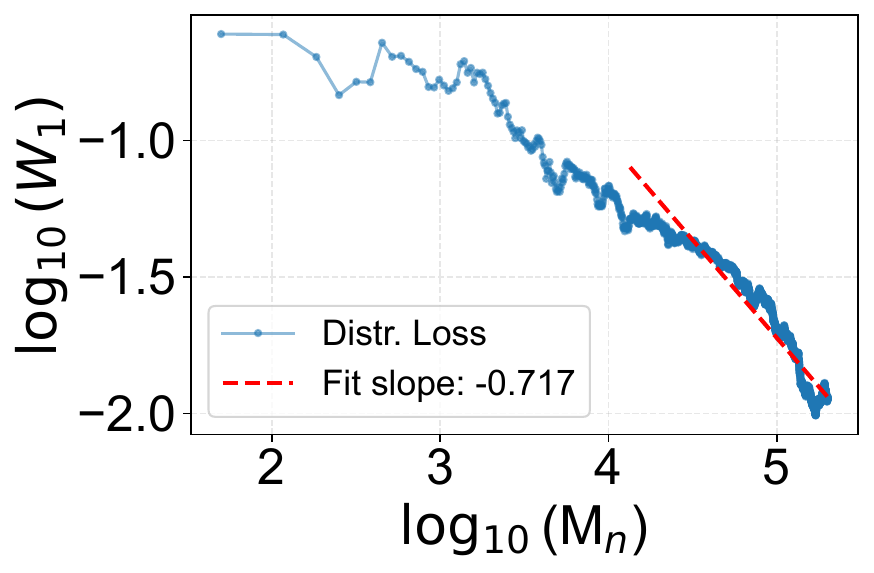}
    \end{minipage}
     \begin{minipage}{0.275\textwidth}
        \centering
        \includegraphics[width=\linewidth]{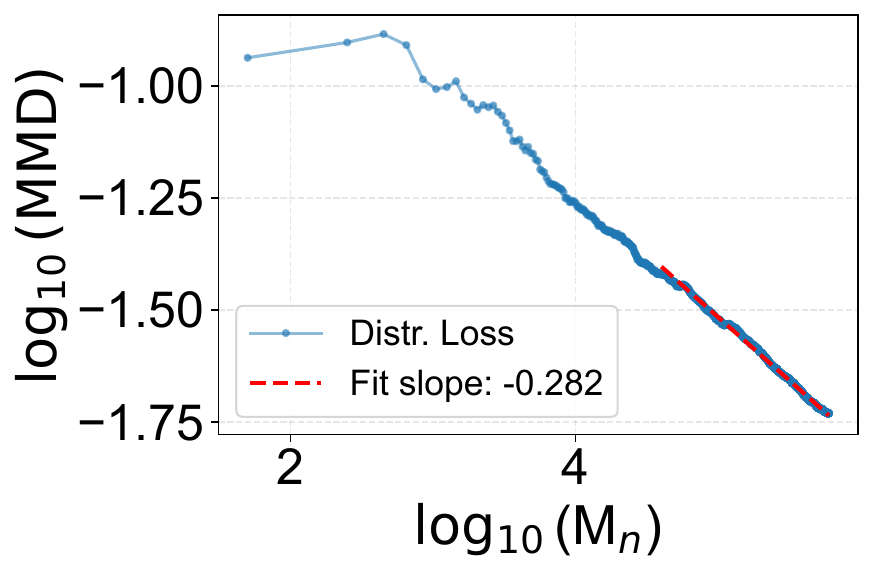}
        {\small $\alpha=0.25,\ q=0.50$}
    \end{minipage}
    \hfill
    \begin{minipage}{0.275\textwidth}
        \centering
        \includegraphics[width=\linewidth]{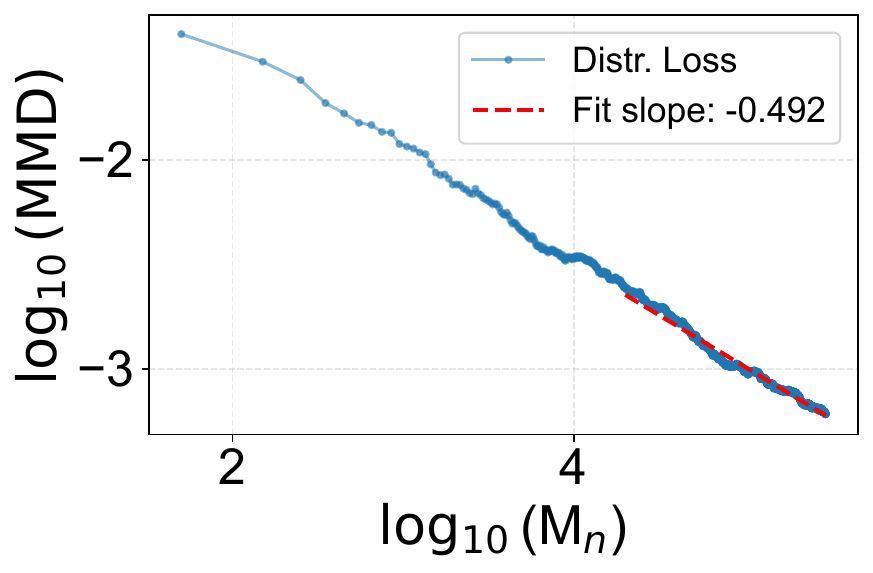}
        {\small $\alpha=0.50,\ q=0.50$}
    \end{minipage}
    \hfill
    \begin{minipage}{0.275\textwidth}
        \centering
        \includegraphics[width=\linewidth]{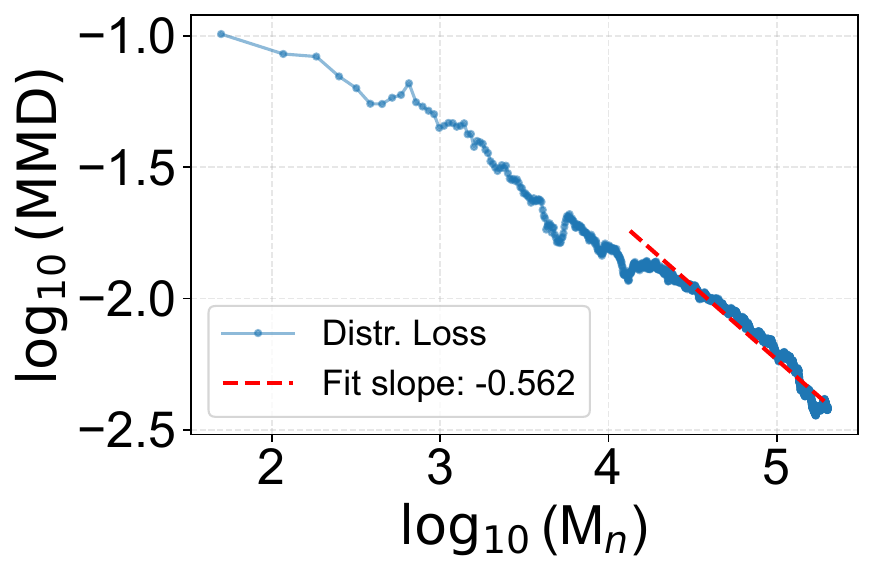}
        {\small $\alpha=0.75,\ q=0.50$}
    \end{minipage}

    \vspace{1em}

    \begin{minipage}{0.275\textwidth}
        \centering
        \includegraphics[width=\linewidth]{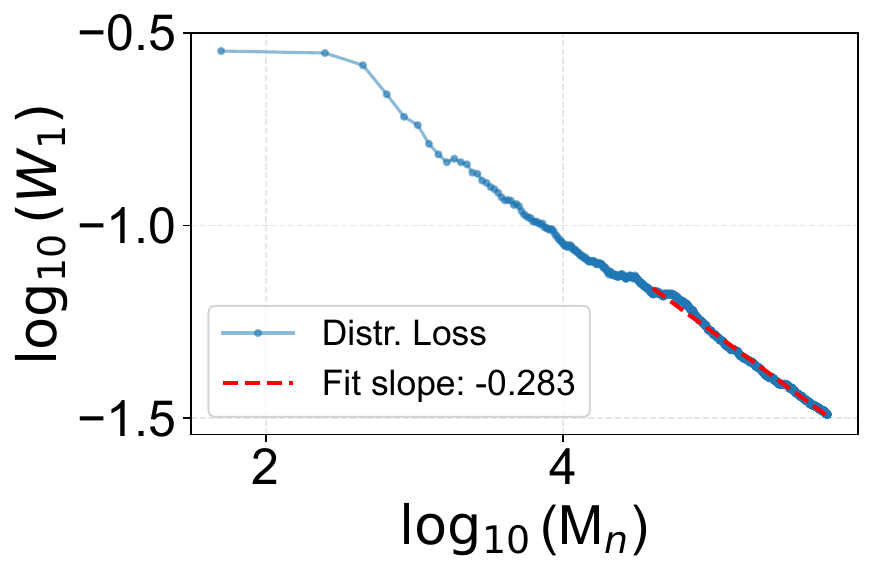}
    \end{minipage}
    \hfill
    \begin{minipage}{0.275\textwidth}
        \centering
        \includegraphics[width=\linewidth]{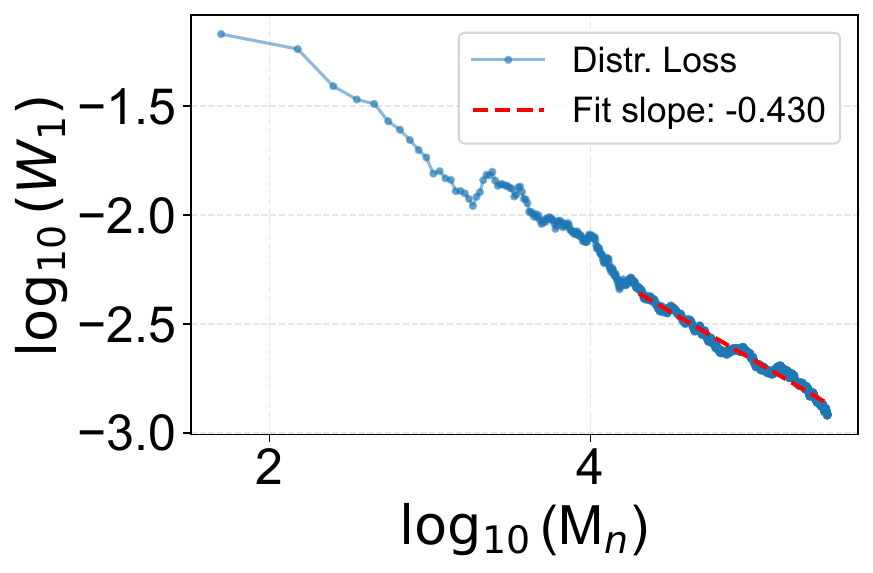}
    \end{minipage}
    \hfill
    \begin{minipage}{0.275\textwidth}
        \centering
        \includegraphics[width=\linewidth]{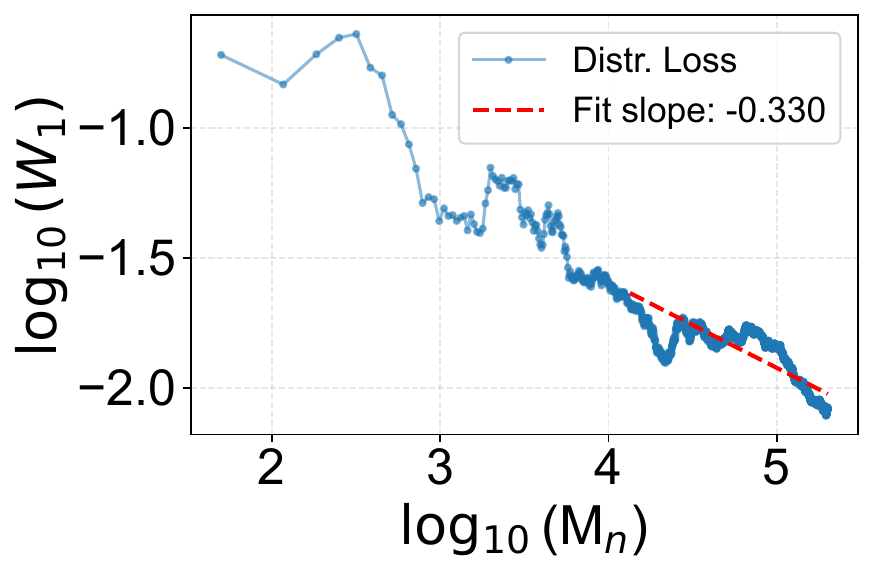}
    \end{minipage}
     \begin{minipage}{0.275\textwidth}
        \centering
        \includegraphics[width=\linewidth]{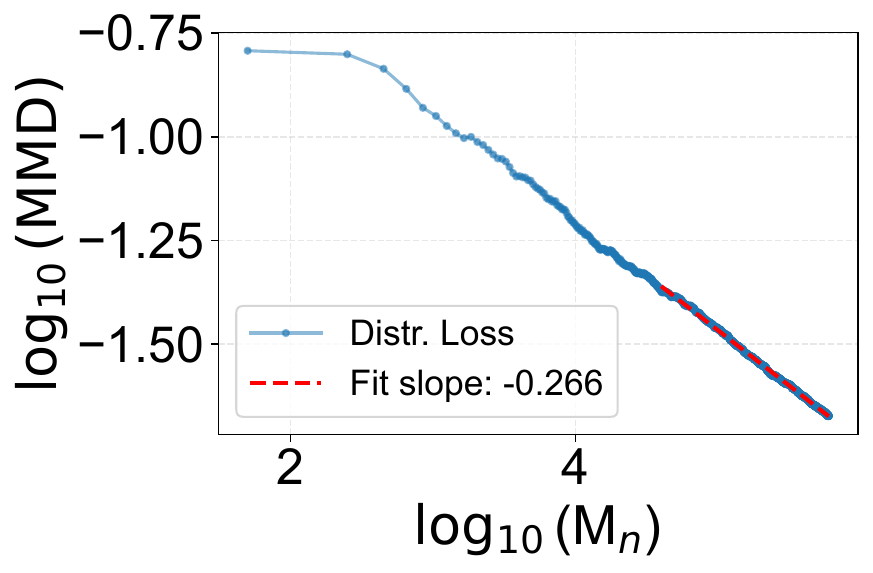}
        {\small $\alpha=0.25,\ q=0.75$}
    \end{minipage}
    \hfill
    \begin{minipage}{0.275\textwidth}
        \centering
        \includegraphics[width=\linewidth]{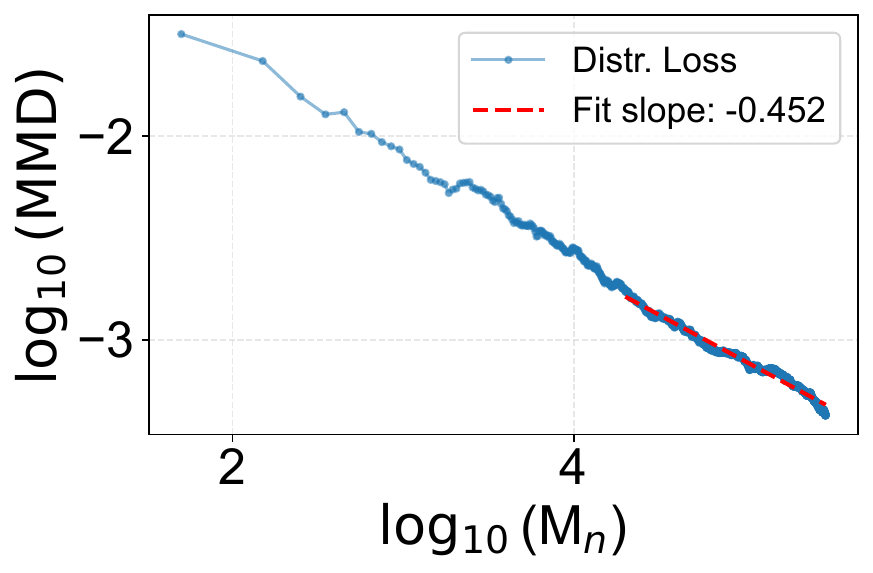}
        {\small $\alpha=0.50,\ q=0.75$}
    \end{minipage}
    \hfill
    \begin{minipage}{0.275\textwidth}
        \centering
        \includegraphics[width=\linewidth]{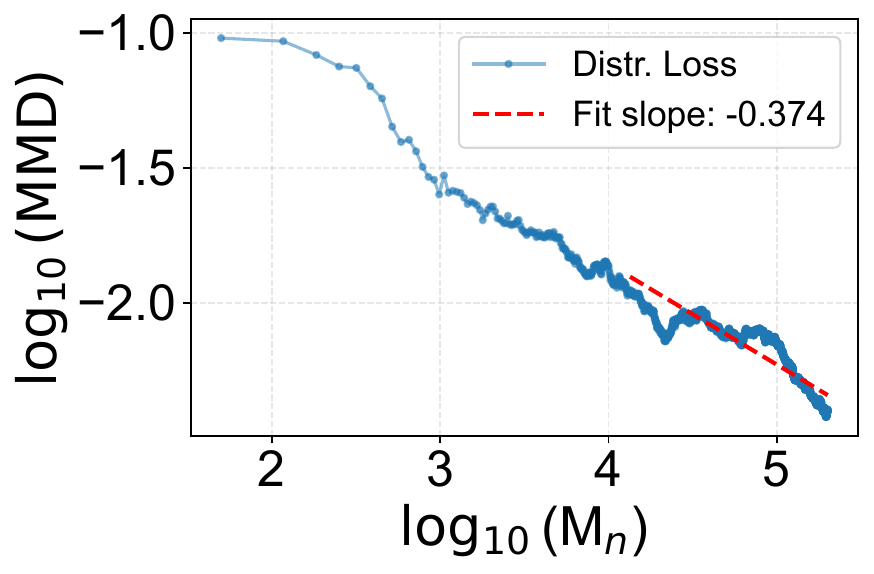}
        {\small $\alpha=0.75,\ q=0.75$}
    \end{minipage}

    \caption{BCRT Simulation (ECDF Estimator). $W_1$ and MMD distributional losses and fitted slopes for varying values of $\alpha$ and $q$ for a single run. Slopes are fitted after dropping the first 200 out of 3000 iterations.}
    \label{fig:ECDF-three-by-three-alpha-q-loss}
\end{figure}

\clearpage
    \begin{table}[ht!]
        \centering
        \footnotesize
        \begin{tabular}{lcl}
        \toprule
        \textbf{Parameter} & \textbf{Value} & \textbf{Description} \\
        \midrule
        $m_1$ & $50$ & Real samples per iteration \\
        $\alpha$ & $\{0.25,0.5,0.75\}$ & real-data fraction \\
        $q$ & $\{0.25,0.5,0.75\}$ & Bias Decay Rate \\
        $T$ & $3000$ & Total BCRT iterations \\
        $n_{\mathrm{reps}}$ & $100$ & Number of repetitions \\
        $m_{\mathrm{grid}}$ & $200$ & Grid size for deterministic evaluation \\
        $[x_{\min},x_{\max}]$ & Mixture-based & Grid interval for density/CDF evaluation \\
        $w_1$ & $0.35$ & Mixture weight \\
        $\mu_1,\sigma_1$ & $-2.0,\; 0.8$ & First Gaussian component \\
        $\mu_2,\sigma_2$ & $1.0,\; 1.3$ & Second Gaussian component \\
        $\mu_3,\sigma_3$ & $3.0,\; 1.0$ & Bias Gaussian component \\
        $h_0$ & $2.0$ & Base KDE bandwidth \\
        \bottomrule
        \end{tabular}
        \caption{BCRT Simulation (KDE Estimator) experimental parameters}
        \label{tab:sim4-kde-params}
        \end{table}

\begin{figure}[ht!]
    \centering

    \begin{minipage}{0.30\textwidth}
        \centering
        \includegraphics[width=\linewidth]{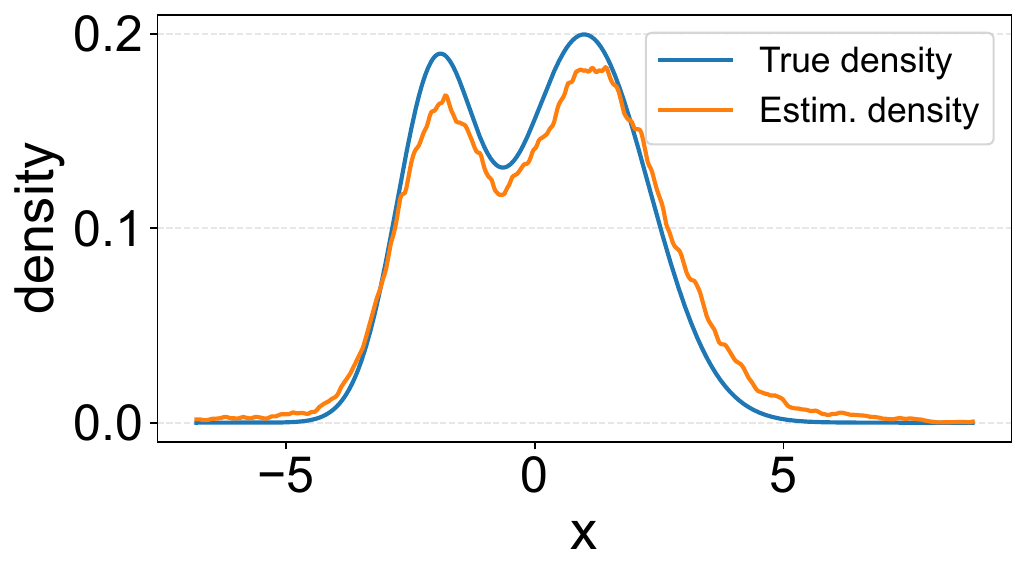}
        \\[-0.5em]
        {\small $\alpha=0.25,\ q=0.25$}
    \end{minipage}
    \hfill
    \begin{minipage}{0.30\textwidth}
        \centering
        \includegraphics[width=\linewidth]{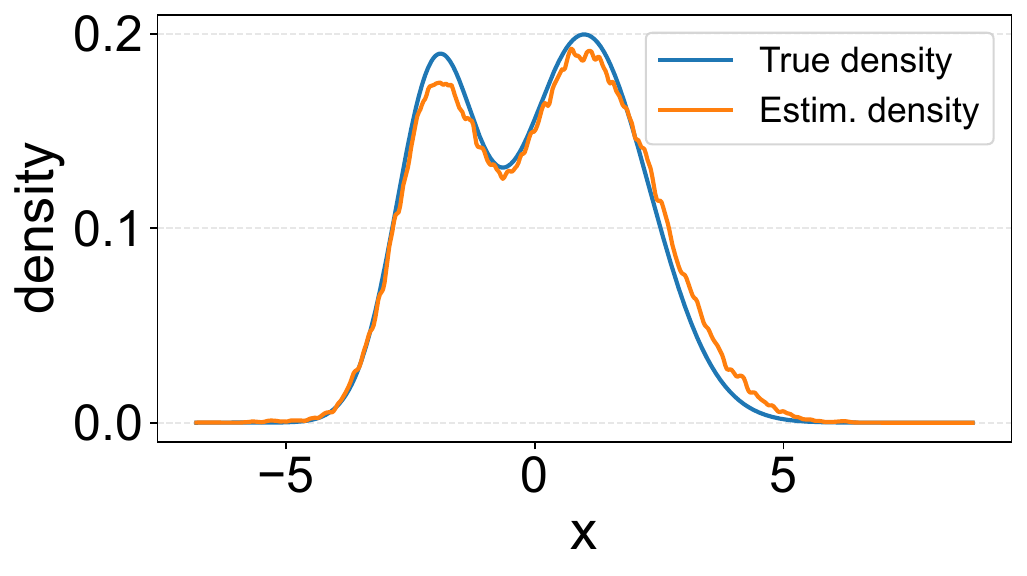}
        \\[-0.5em]
        {\small $\alpha=0.50,\ q=0.25$}
    \end{minipage}
    \hfill
    \begin{minipage}{0.30\textwidth}
        \centering
        \includegraphics[width=\linewidth]{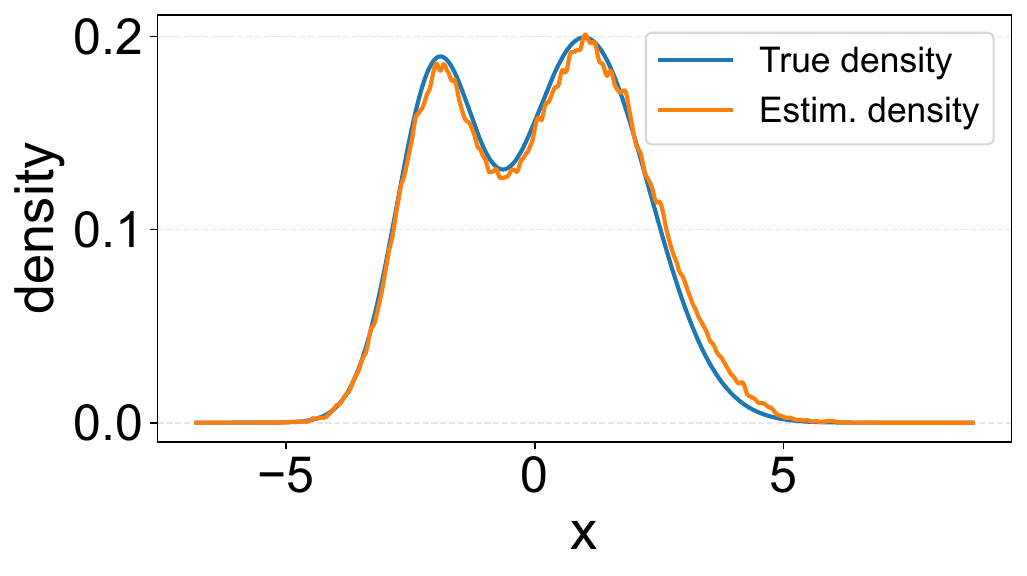}
        \\[-0.5em]
        {\small $\alpha=0.75,\ q=0.25$}
    \end{minipage}

    \vspace{1em}

    \begin{minipage}{0.30\textwidth}
        \centering
        \includegraphics[width=\linewidth]{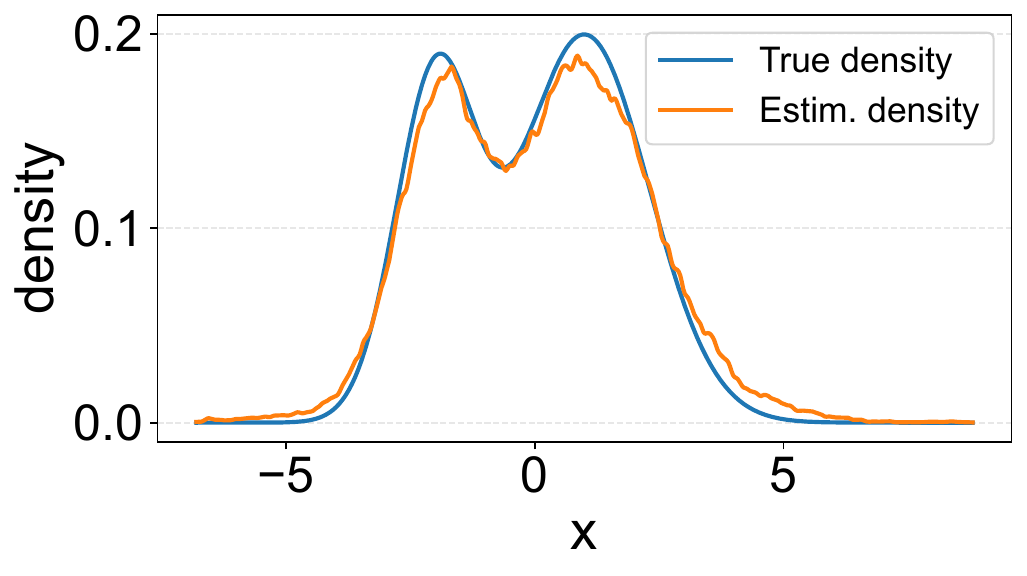}
        \\[-0.5em]
        {\small $\alpha=0.25,\ q=0.50$}
    \end{minipage}
    \hfill
    \begin{minipage}{0.30\textwidth}
        \centering
        \includegraphics[width=\linewidth]{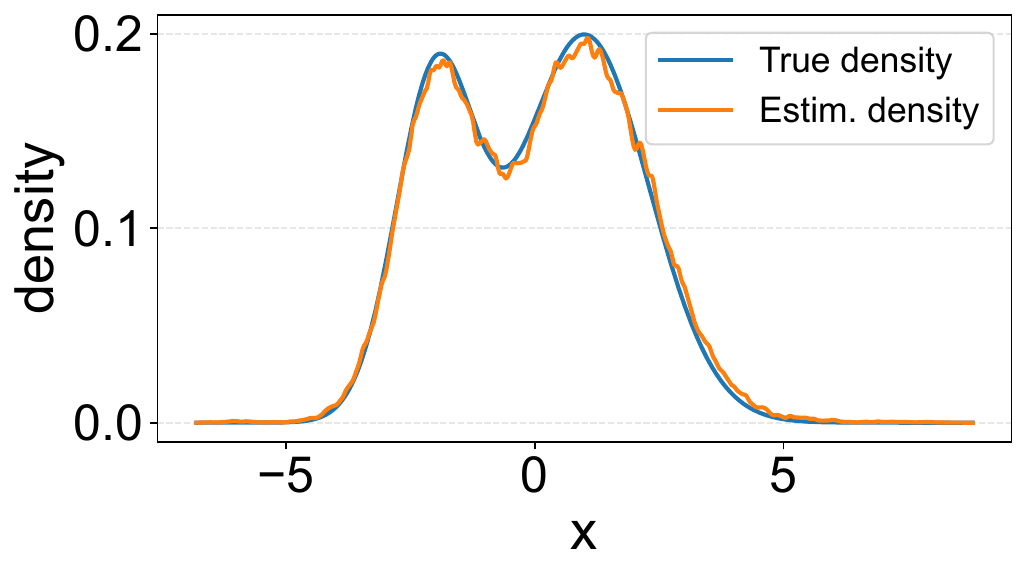}
        \\[-0.5em]
        {\small $\alpha=0.50,\ q=0.50$}
    \end{minipage}
    \hfill
    \begin{minipage}{0.30\textwidth}
        \centering
        \includegraphics[width=\linewidth]{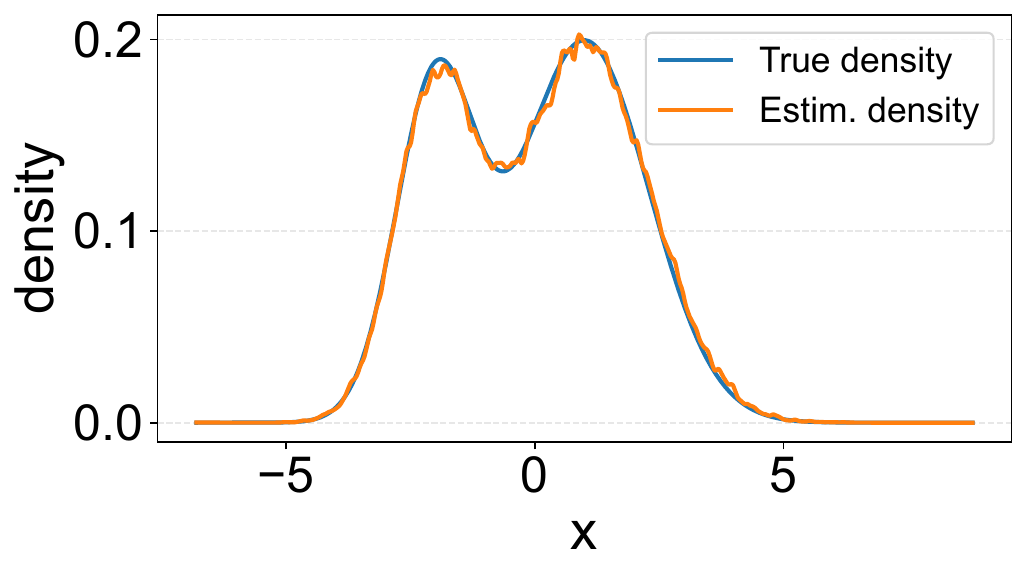}
        \\[-0.5em]
        {\small $\alpha=0.75,\ q=0.50$}
    \end{minipage}

    \vspace{1em}

    \begin{minipage}{0.30\textwidth}
        \centering
        \includegraphics[width=\linewidth]{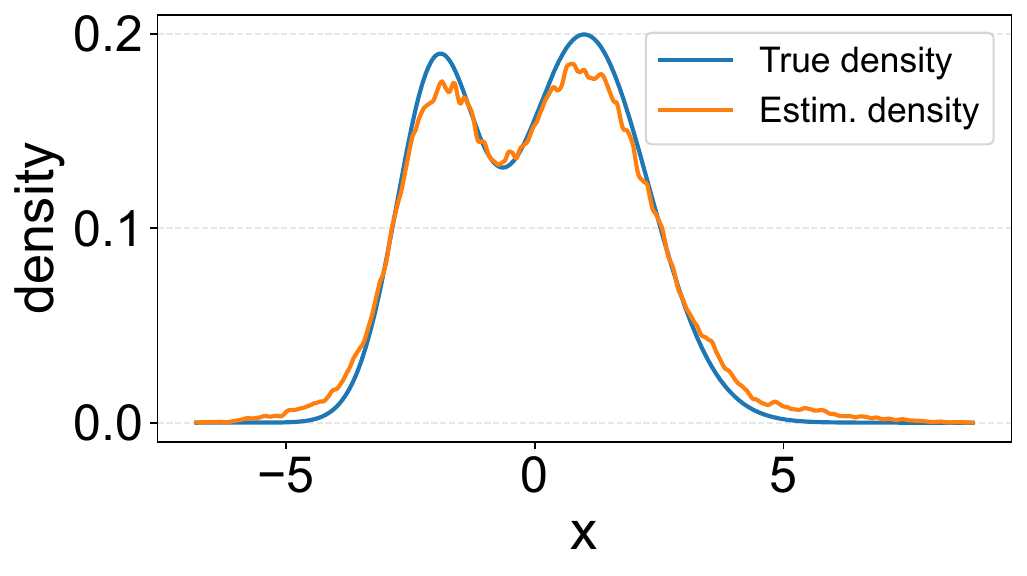}
        \\[-0.5em]
        {\small $\alpha=0.25,\ q=0.75$}
    \end{minipage}
    \hfill
    \begin{minipage}{0.30\textwidth}
        \centering
        \includegraphics[width=\linewidth]{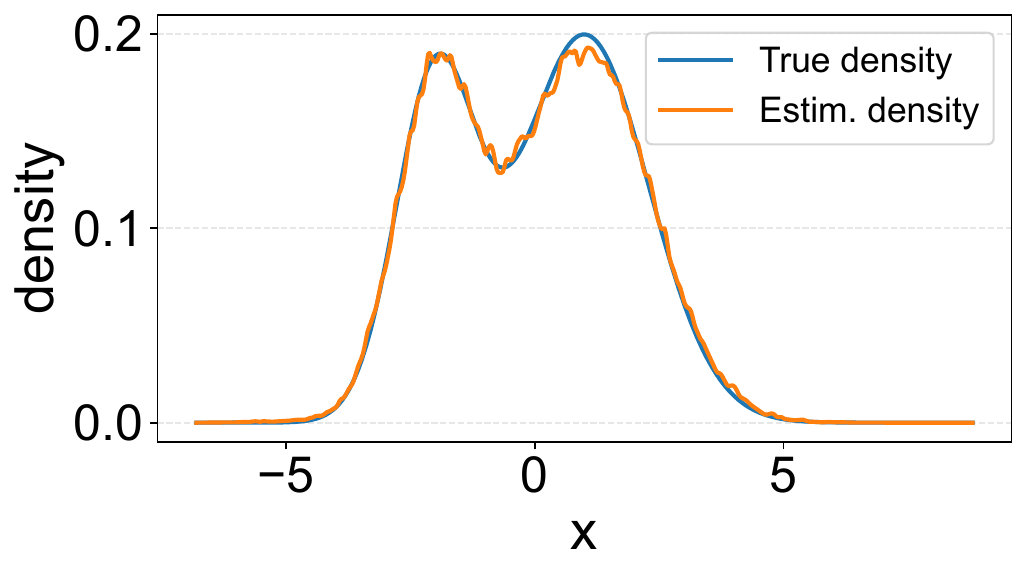}
        \\[-0.5em]
        {\small $\alpha=0.50,\ q=0.75$}
    \end{minipage}
    \hfill
    \begin{minipage}{0.30\textwidth}
        \centering
        \includegraphics[width=\linewidth]{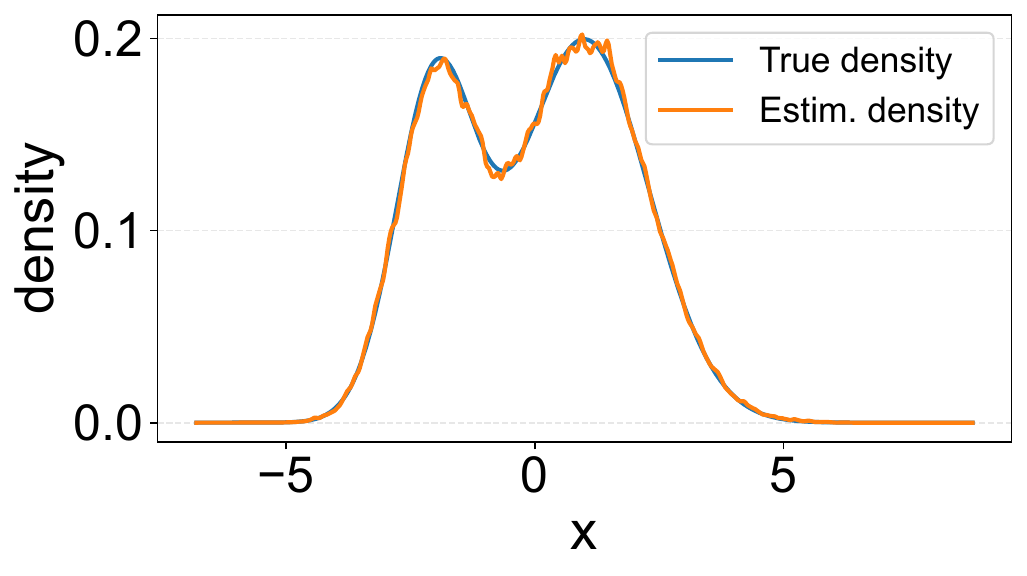}
        \\[-0.5em]
        {\small $\alpha=0.75,\ q=0.75$}
    \end{minipage}

    \caption{BCRT Simulation (KDE Estimator). Final output distributions across combinations of real-data fraction $\alpha \in \{0.25,0.5,0.75\}$ and bias decay rate $q \in \{0.25,0.5,0.75\}$.}
    \label{fig:KDE-three-by-three-alpha-q}
\end{figure}

\begin{figure}[ht!]
    \centering

    \begin{minipage}{0.275\textwidth}
        \centering
        \includegraphics[width=\linewidth]{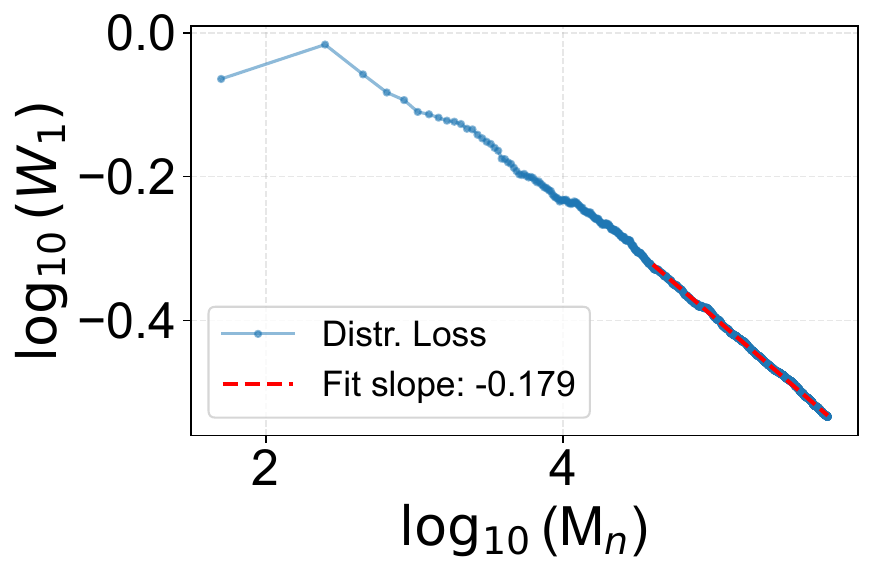}
    \end{minipage}
    \hfill
    \begin{minipage}{0.275\textwidth}
        \centering
        \includegraphics[width=\linewidth]{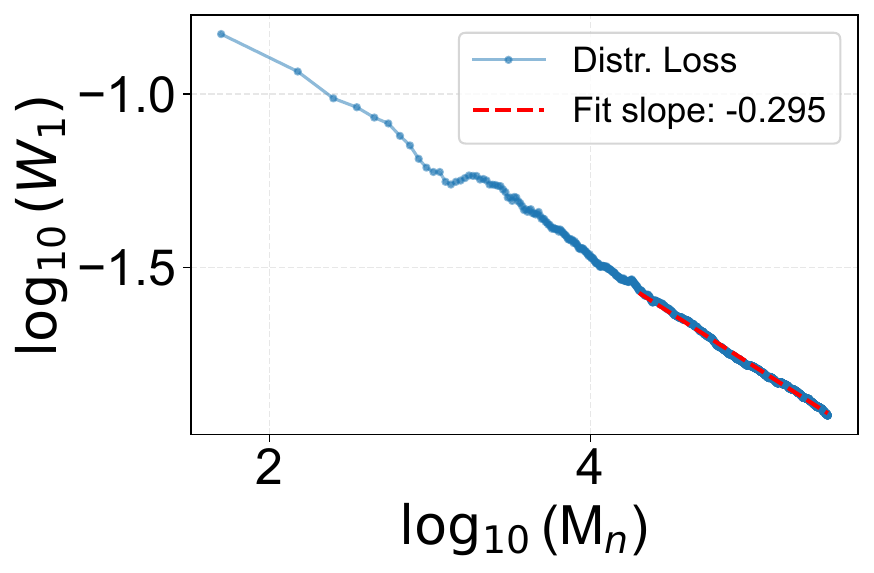}
        \end{minipage}
    \hfill
    \begin{minipage}{0.275\textwidth}
        \centering
        \includegraphics[width=\linewidth]{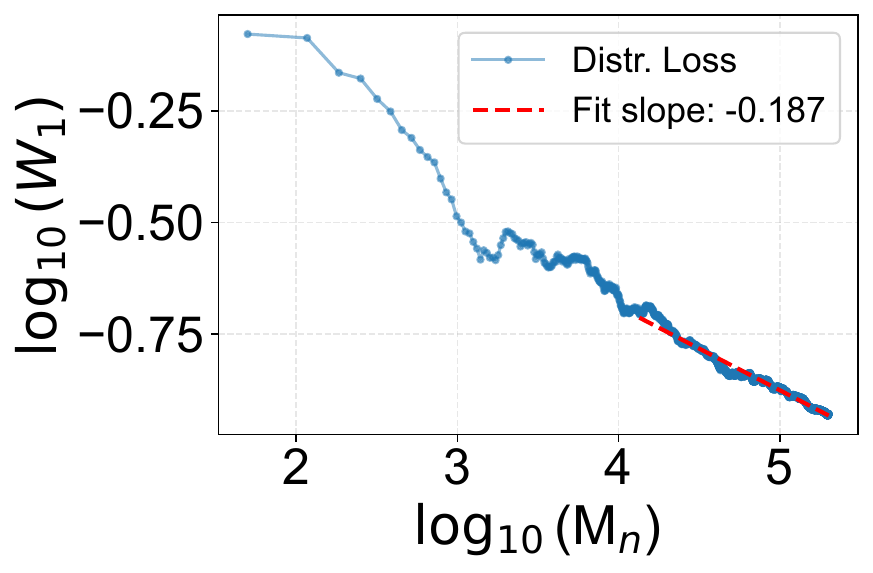}
        \end{minipage}
     \begin{minipage}{0.275\textwidth}
        \centering
        \includegraphics[width=\linewidth]{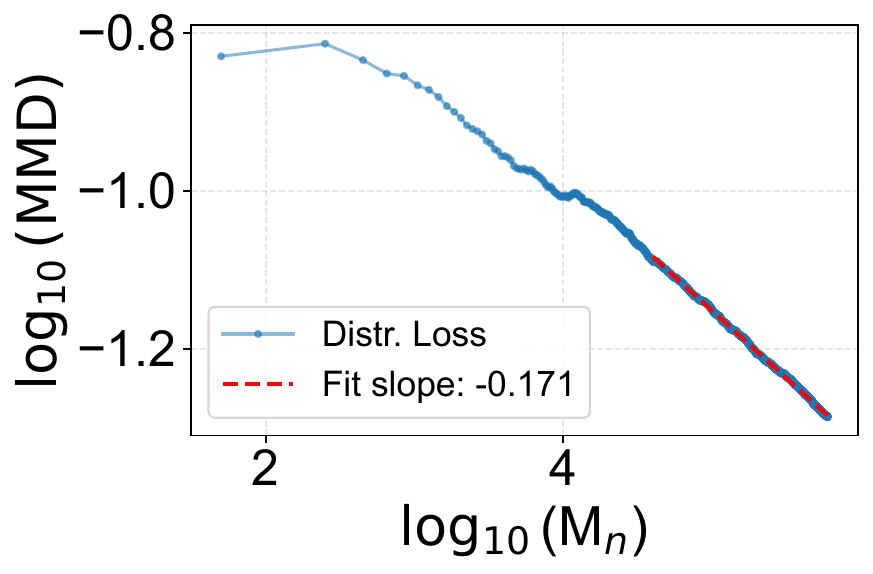}
        \caption*{\small $\alpha=0.25,\ q=0.25$}
    \end{minipage}
    \hfill
    \begin{minipage}{0.275\textwidth}
        \centering
        \includegraphics[width=\linewidth]{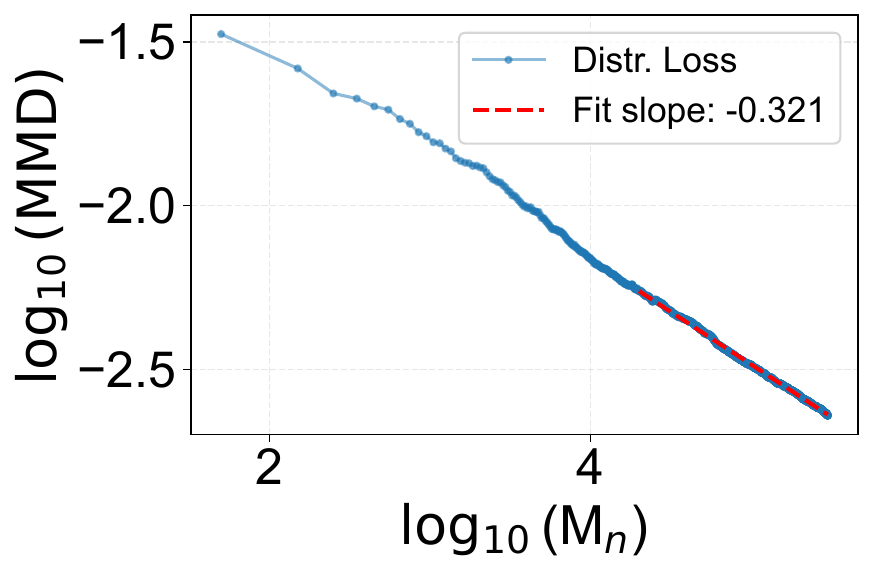}
        \caption*{\small $\alpha=0.50,\ q=0.25$}
        \end{minipage}
    \hfill
    \begin{minipage}{0.275\textwidth}
        \centering
        \includegraphics[width=\linewidth]{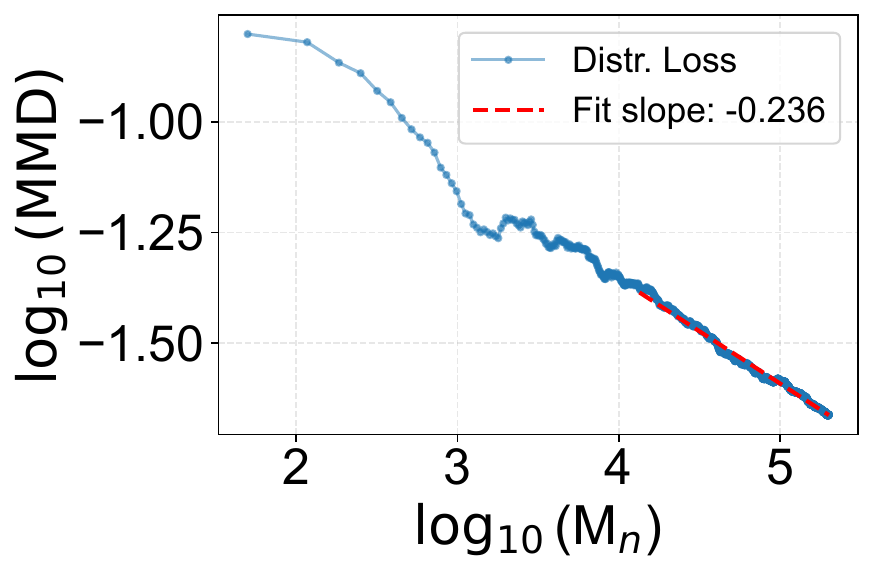}
        \caption*{\small $\alpha=0.75,\ q=0.25$}
        \end{minipage}
        
    \vspace{1em}

    \begin{minipage}{0.275\textwidth}
        \centering
        \includegraphics[width=\linewidth]{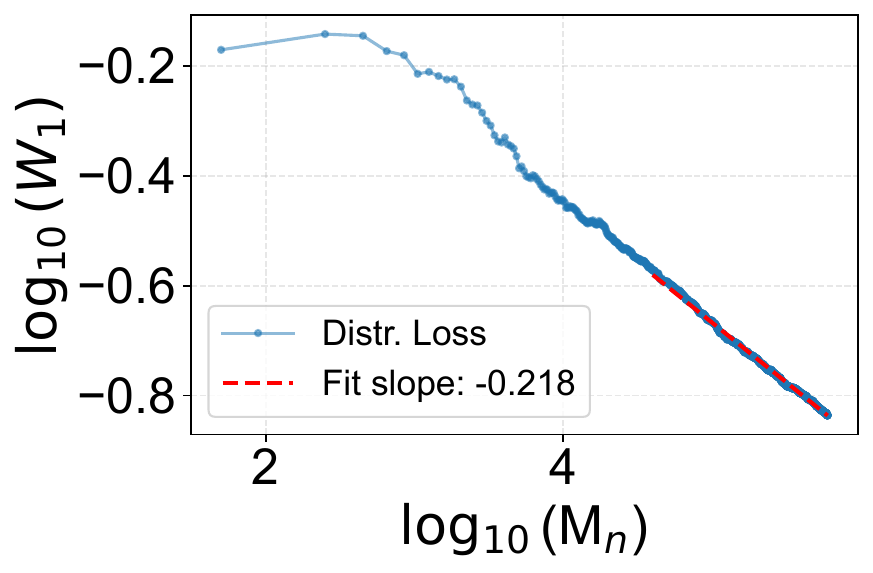}
    \end{minipage}
    \hfill
    \begin{minipage}{0.275\textwidth}
        \centering
        \includegraphics[width=\linewidth]{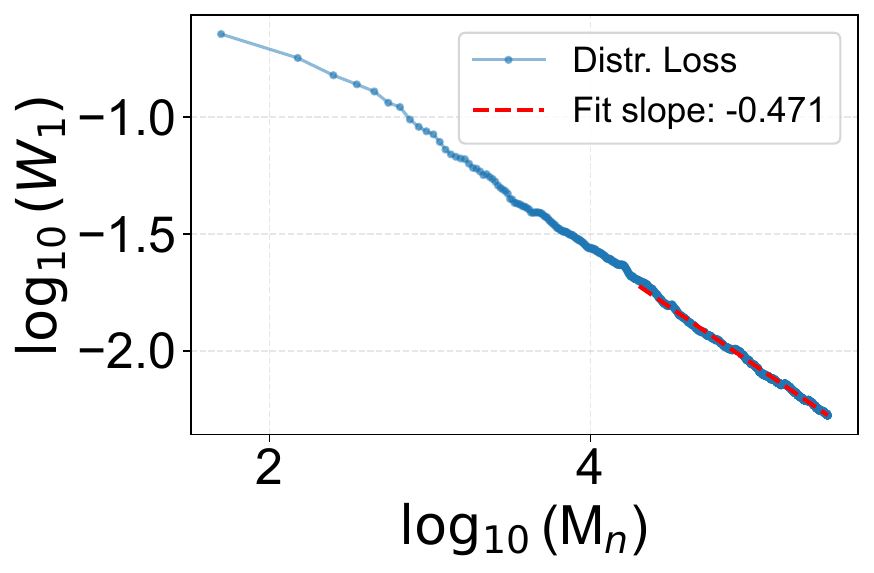}
    \end{minipage}
    \hfill
    \begin{minipage}{0.275\textwidth}
        \centering
        \includegraphics[width=\linewidth]{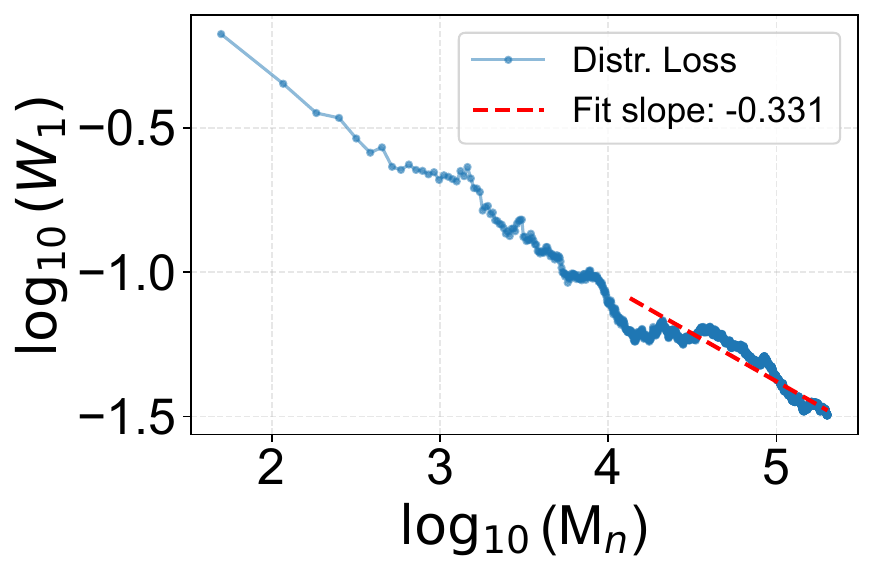}
    \end{minipage}
     \begin{minipage}{0.275\textwidth}
        \centering
        \includegraphics[width=\linewidth]{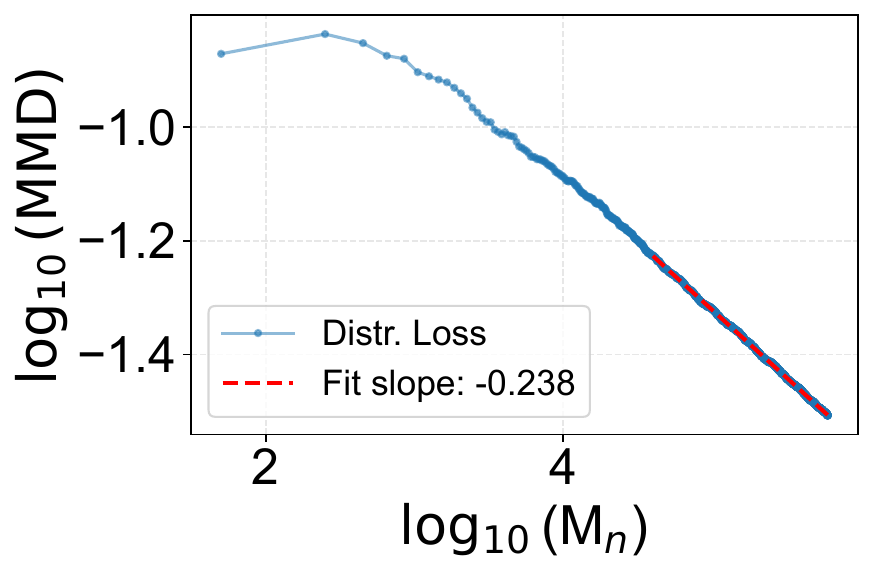}
        {\small $\alpha=0.25,\ q=0.50$}
    \end{minipage}
    \hfill
    \begin{minipage}{0.275\textwidth}
        \centering
        \includegraphics[width=\linewidth]{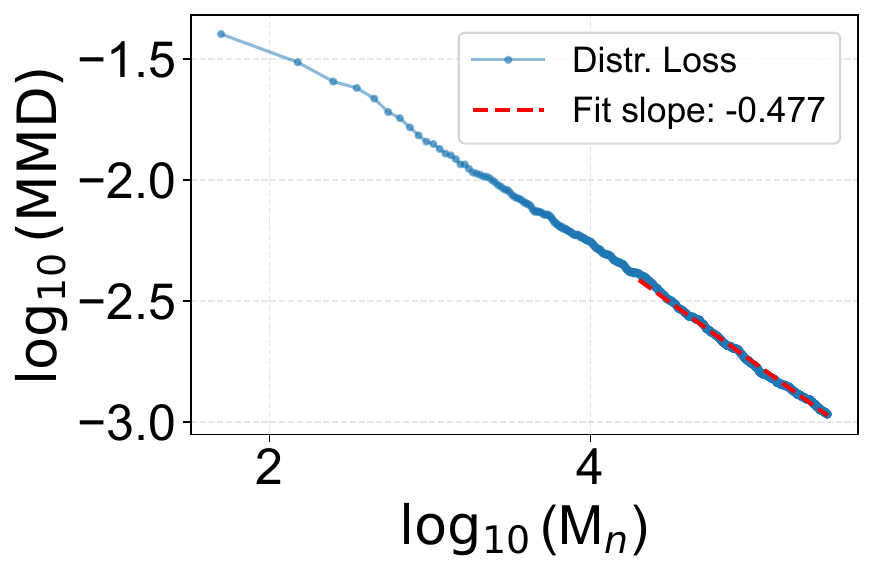}
        {\small $\alpha=0.50,\ q=0.50$}
    \end{minipage}
    \hfill
    \begin{minipage}{0.275\textwidth}
        \centering
        \includegraphics[width=\linewidth]{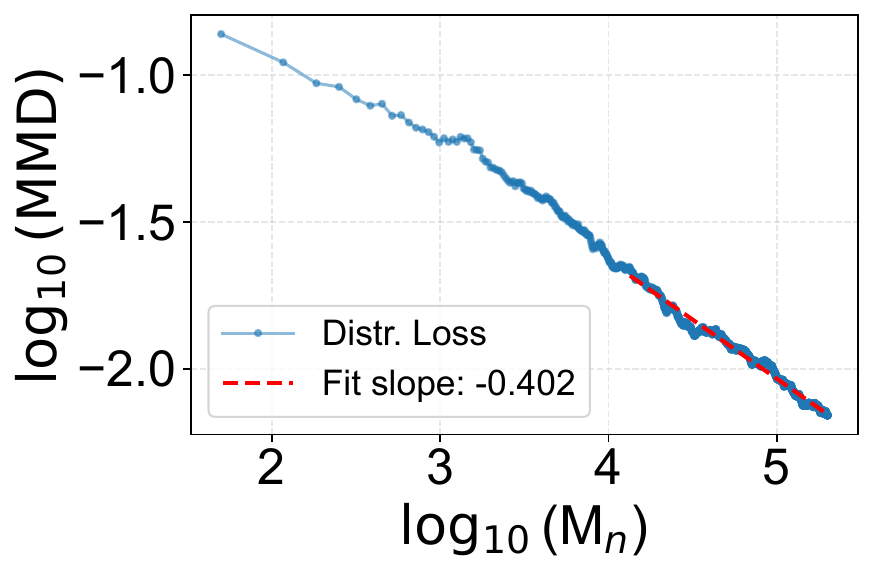}
        {\small $\alpha=0.75,\ q=0.50$}
    \end{minipage}

    \vspace{1em}

    \begin{minipage}{0.275\textwidth}
        \centering
        \includegraphics[width=\linewidth]{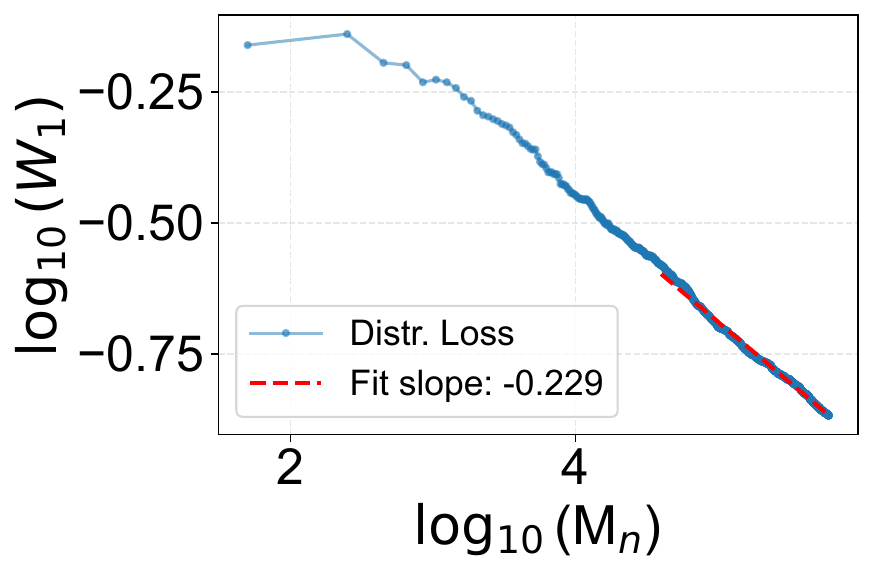}
    \end{minipage}
    \hfill
    \begin{minipage}{0.275\textwidth}
        \centering
        \includegraphics[width=\linewidth]{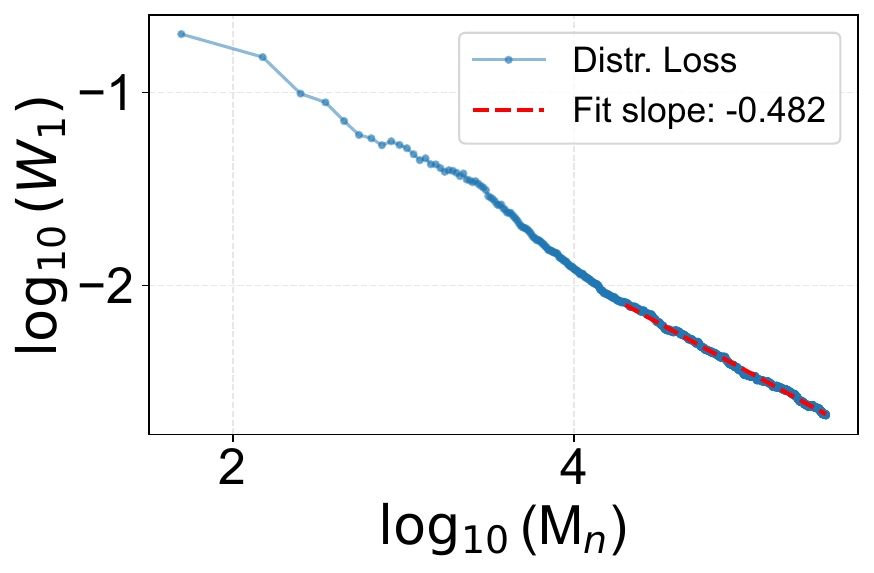}
    \end{minipage}
    \hfill
    \begin{minipage}{0.275\textwidth}
        \centering
        \includegraphics[width=\linewidth]{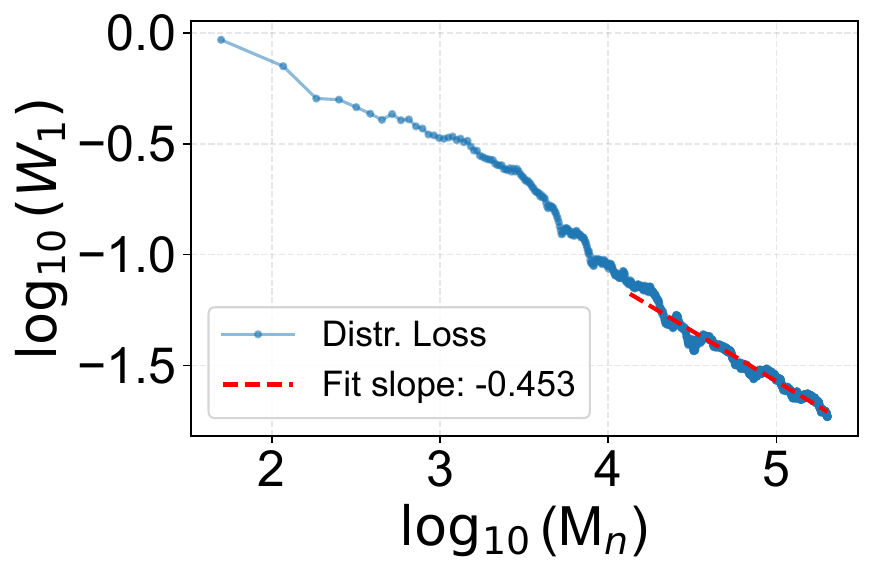}
    \end{minipage}
     \begin{minipage}{0.275\textwidth}
        \centering
        \includegraphics[width=\linewidth]{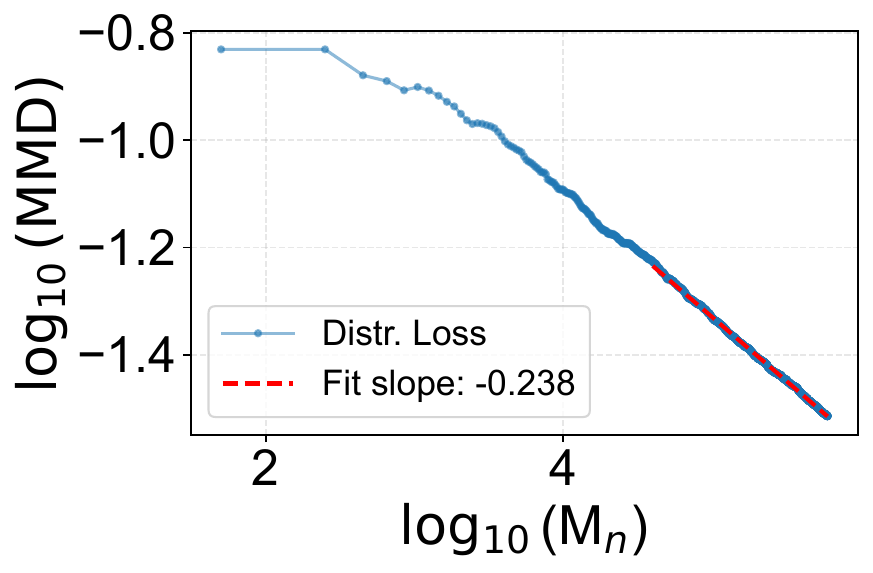}
        {\small $\alpha=0.25,\ q=0.75$}
    \end{minipage}
    \hfill
    \begin{minipage}{0.275\textwidth}
        \centering
        \includegraphics[width=\linewidth]{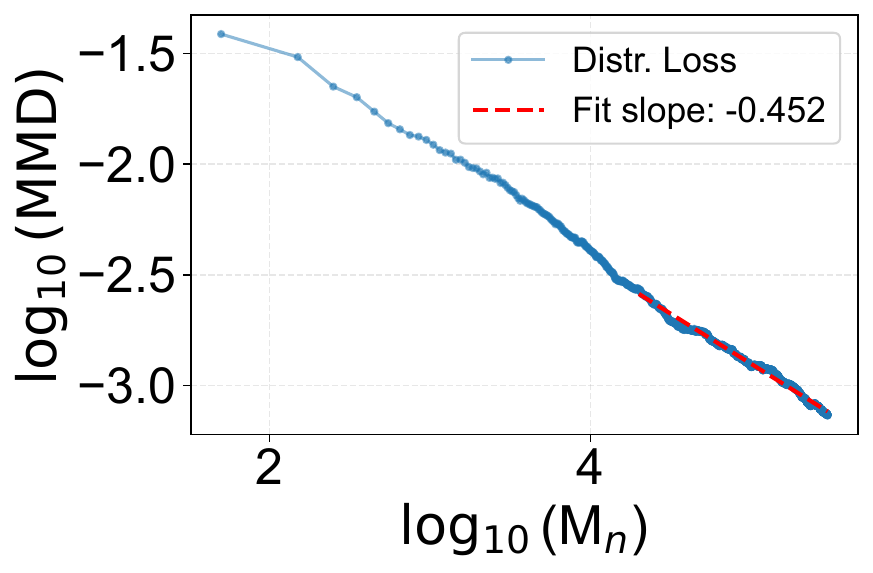}
        {\small $\alpha=0.50,\ q=0.75$}
    \end{minipage}
    \hfill
    \begin{minipage}{0.275\textwidth}
        \centering
        \includegraphics[width=\linewidth]{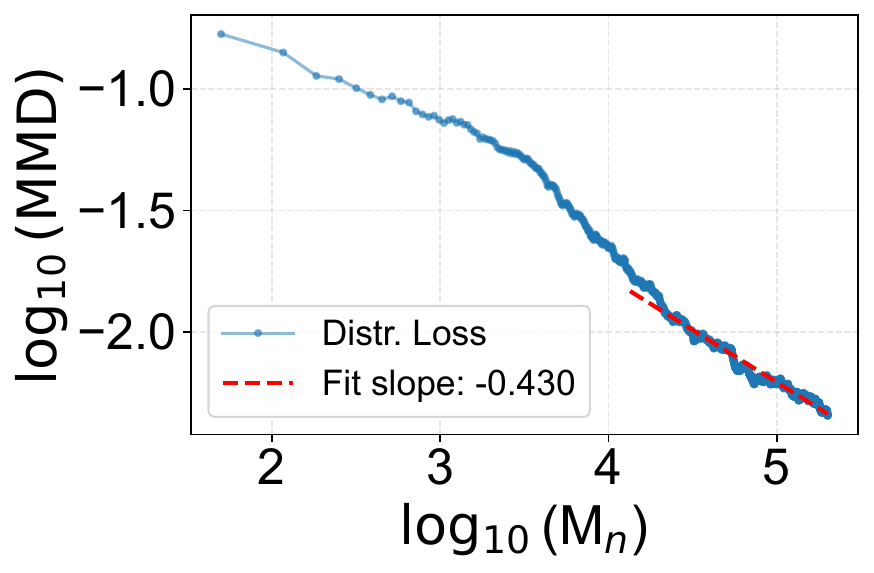}
        {\small $\alpha=0.75,\ q=0.75$}
    \end{minipage}

    \caption{BCRT Simulation (KDE Estimator). $W_1$ and MMD distributional losses and fitted slopes for varying values of $\alpha$ and $q$ for a single run. Slopes are fitted after dropping the first 200 out of 3000 iterations.}
    \label{fig:KDE-three-by-three-alpha-q-loss}
\end{figure}

\clearpage
\section{Additional MNIST Experiment Details}\label{apdx:mnist}

        In this section, we present experimental parameters of the DDPM diffusion model trained on MNIST under CRT. We use data from the 60,000 MNIST training set where each training sample is used by the optimizer an equal number of epochs, as per the setup in \Cref{thm:recursive_convergence}. Each iteration builds on the previous model's generator, such that every new sample is trained for a fixed number of epochs and all data points are used by the optimizer for an equal number of training steps. To ensure uniform representation, the newly introduced real data points at each step are explicitly sampled to maintain a balanced distribution across all ten digit classes. A new sample is produced at each CRT iteration to visualize the output of each iteration's generator.
    
    \begin{table}[H]
        \centering
        \small
        \begin{tabular}{lcl}
        \toprule
        \textbf{Parameter} & \textbf{Value} & \textbf{Description} \\
        \midrule
        $\alpha$ & $\{0.0, 0.25, 0.50, 0.75, 1.00\}$ & real-data fraction (Sweep default) \\
        Samples per step & $300$ & Total combined batch size per iteration \\
        $m_1$ & $\alpha \times 300$ & Real samples per iteration \\
        $T$ & $200$ & Total CRT iterations \\
        $T_{diff}$ & $400$ & Number of diffusion steps \\
        Epochs & $100$ & Epochs of training per data iteration \\
        Batch Size & $100$ & Training batch size \\
        Schedule & Cosine & Optimizer schedule \\
        $\beta_{start}$  & $1\times10^{-4}$ & Noise variance schedule \\
        $\beta_{end}$  & $2\times10^{-2}$ & Noise variance schedule  \\
        Time dimension & $128$ & Embedding dimension for time conditioning \\
        Base channels & $64$ & Base channel multiplier for the UNet \\
        \bottomrule
        \end{tabular}
        \caption{Experimental parameters for MNIST experiment.}
        \label{tab:MNIST-diffusion-params}
        \end{table}        

        We present the details of the network architecture and training objective. We parameterize the diffusion noise predictor using a lightweight UNet-style convolutional network with residual blocks, attention mechanisms, and time conditioning. The timestep $t$ is embedded via a sinusoidal encoding and passed through a small MLP. To improve training stability, the mean squared error loss for epsilon prediction is scaled using Min-SNR weighting ($\gamma = 5.0$). All experimental parameters for model training are summarized in \Cref{tab:MNIST-diffusion-params}.

        Architecturally, the model follows a two-level UNet structure. The encoder consists of an initial convolution followed by residual blocks, downsampling from 28x28 to 14x14, where a scaled dot-product attention block is applied. A second downsampling reduces the spatial resolution to 7x7 while increasing the channel width from 64 to 128. The bottleneck consists of residual blocks interleaved with a secondary attention block at the 7x7 resolution. The decoder mirrors the encoder with two stages of upsampling, skip connections, and a final attention block at the 14x14 resolution. The final convolution maps features back to the single output channel predicting the noise.

\section{Additional LLM Experiment Details}\label{apdx:llm}

    In this section we provide additional details on the contaminated recursive training experiment using LLMs. The dataset we use is the WikiText-103 corpus, consisting of high-quality Wikipedia articles, pre-partitioned into training and test splits. We use the standard release training split for fine-tuning, and the standard release test split for independent evaluation of the text diversity metrics. The training set consists of approximately 103 million tokens, while the test set contains approximately 0.25 million tokens. Both sets of data are further divided into disjoint blocks of 1,024 tokens for training and evaluation. 
    
    The model used for all experiments is the 124M parameter version of GPT-2 from OpenAI with its standard tokenizer, initialized from publicly available pretrained weights and fine-tuned within the recursive training framework described above.

    \subsection{Training} 
    
        Fine-tuning was performed using the standard autoregressive objective, in which the model is optimized to predict each token given its preceding context. For a sequence of tokens \( (x_1, \dots, x_L) \), the objective is
            \[
            \mathcal{L} = - \sum_{i=1}^{L} \log p_{\theta}(x_i \mid x_{<i}),
            \]
        where \( \theta \) denotes the model parameters.

        \begin{table}[H]
            \centering
            \small
            \begin{tabular}{lcl}
            \toprule
            \textbf{Training Parameter} & \textbf{Value} & \textbf{Description} \\
            \midrule
            Block size & 1024 & Tokens per training example \\
            $m_{\text{total}}$ & 1024 & Blocks per iteration dataset $D_t$ \\
            $\alpha$ & $\{0.0, 0.10, 1.0\}$ & real-data fraction \\
            $T$ & 50 & Number of recursive iterations \\
            Epochs & 1 & Epochs per iteration \\
            Learning rate & $2\times10^{-5}$ & AdamW optimizer \\
            Batch size & 1 & Micro-batch size \\
            Gradient accumulation. & 32 & Gradient accumulation steps \\
            Schedule & Linear warmup + decay & Learning rate schedule \\
            \midrule
            \textbf{Sampling Parameter}  \\
            \midrule
            Generation $p$ & 0.95 & Nucleus sampling parameter \\
            Temperature & 1.0 & Sampling temperature \\
            Repetition penalty & 1.05 & Decoding penalty \\
            \bottomrule
            \end{tabular}
            \caption{Experimental parameters for recursive LLM training experiment. Note that the effective batch size for the experiment is the micro-batch size multiplied by the number of gradient accumulations steps. }
            \label{tab:llm-recursive-params}
            \end{table}

             Optimization was performed using the AdamW optimizer with a learning rate of \( 2 \times 10^{-5} \), weight decay, and linear learning rate warmup followed by decay. Gradient accumulation was used to achieve a larger effective batch size under memory constraints. Mixed-precision (FP16) training was employed when supported by the hardware.

             At each iteration, the model was trained for a fixed number of full passes (epochs) over the constructed dataset \( D_t \) at outer training iteration $t$, such that no training token was re-used between iterations. At each $t$, $D_t$ consists of a mixture of real data blocks and synthetic data blocks generated by the generator trained at $t-1$, where the fraction of real data blocks is defined by $\alpha$. Model parameters were updated between iterations without reinitialization unless otherwise specified, as in previous experiments, to ensure all data points are utilized by the optimizer an equal number of times as per \Cref{thm:recursive_convergence}. All experimental parameters for model training and sampling are summarized in \Cref{tab:llm-recursive-params}.

    \subsection{Evaluation} 

        All evaluations are performed on the pre-partitioned test set, which was not directly used for any training or fine-tuning. The primary evaluation metrics used were 1) model perplexity, reflecting the model's predictive accuracy, 2) average per-token predictive entropy, measuring the model output uncertainty, and 3) distinct$-n$ for $n=1,2,3$ to measure output diversity. Perplexity and entropy were evaluated over the independent test set, while distinct$-n$ were evaluated over a new sample of the trained models.

\end{document}